\documentclass[11pt]{article}
\usepackage[left=1in,top=1in,right=1in,bottom=1in,letterpaper]{geometry}

\usepackage{fancyhdr,listings,amsfonts,amsmath,color}
\usepackage[linktocpage,colorlinks,linkcolor=blue,anchorcolor=blue,citecolor=blue,urlcolor=blue,pagebackref]{hyperref}
\usepackage[noend]{algpseudocode}

\usepackage{latexsym,amssymb,amsmath,amsfonts,amsthm,xcolor,graphicx, longtable,tabularx}
\RequirePackage[ruled,vlined]{algorithm}
\RequirePackage{algpseudocode}
\usepackage{amsmath}
\usepackage{mdframed}
\usepackage{float}
\usepackage{color}
\usepackage{pifont}
\usepackage{listings,amsfonts,amsmath,color,amsthm}
\hypersetup{colorlinks=true,linkcolor=blue}
\usepackage[utf8]{inputenc}
\usepackage{enumitem,subcaption,ragged2e}
\usepackage{float}
\usepackage{color}
\usepackage{pifont}
\usepackage{amssymb}
\makeatletter
\newcommand*{\rom}[1]{\expandafter\@slowromancap\romannumeral #1@}
\makeatother
\usepackage{multirow}
\usepackage{diagbox}
\usepackage{bbm}
\usepackage{xcolor}
\usepackage{mathtools}
\usepackage{natbib,todonotes}
\usepackage{booktabs}
\usepackage{amsmath}
\usepackage{amssymb}
\usepackage{mathtools}
\usepackage{amsthm}
\usepackage[makeroom]{cancel}

\newcommand{\mb}{\mathbb}
\newcommand{\dee}{\mathrm{d}}
\newcommand{\mc}{\mathcal}

\newcommand{\TV}{\mathrm{TV}}
\newcommand{\KL}{\mathrm{KL}}

\newcommand{\DMM}{\mc{L}_{\mathrm{DMM}}}
\newcommand{\MM}{\mc{L}_{\mathrm{MM}}}

\newcommand{\Softmax}{\mathrm{Softmax}}

\newcommand{\Tr}{\mathrm{trace}}

\newcommand{\Wx}{W_{\mathrm{x}}}
\newcommand{\Wt}{W_{\mathrm{t}}}

\newcommand{\Err}{\mathrm{Err}}

\newcommand{\tin}{t_{\mathrm{in}}}
\newtheorem{theorem}{Theorem}[section]

\newtheorem{proposition}{Proposition}[section]
\newtheorem{definition}{Definition}[section]
\newtheorem{lemma}{Lemma}[section]
\newtheorem{remark}{Remark}[section]
\newtheorem{corollary}{Corollary}[section]

\usepackage{natbib}
\setcitestyle{authoryear,open={(},close={)}} 
\allowdisplaybreaks

\renewcommand*{\backref}[1]{\ifx#1\relax \else Page #1 \fi}
\renewcommand*{\backrefalt}[4]{%
  \ifcase #1 \footnotesize{(Not cited.)}%
  \or        \footnotesize{(Cited on page~#2.)}%
  \else      \footnotesize{(Cited on pages~#2.)}%
  \fi
}

\newtheorem{assumption}{Assumption}

\title{Diffusion Model's Generalization Can Be Characterized by \\ Inductive Biases toward a Data-Dependent Ridge Manifold}
\usepackage{times}
\date{\vspace{-5ex}}

\usepackage{times}
\author{
 Ye He \thanks{Georgia Institute of Technology. \texttt{yhe367@gatech.edu}}
 \and
 Yitong Qiu \thanks{University of Science and Technology of China. \texttt{qyt0912@mail.ustc.edu.cn}}
  \and
 Molei Tao \thanks{Georgia Institute of Technology. \texttt{mtao@gatech.edu}}
}

\begin{document}
\maketitle

\begin{abstract}

We study a data-dependent notion of diffusion-model generalization: when a model does not memorize the training set, where do its generated samples go relative to the geometry induced by the data? To answer this, we introduce a time-dependent family of log-density ridge manifolds constructed from the smoothed empirical distribution, and use it to characterize reverse-time inference. Our main result shows that generated samples evolve by a \textbf{reach--align--slide} mechanism: they first enter a neighborhood of the ridge, then their distance to the ridge is controlled by the normal component of training error, and finally their motion along the ridge is controlled by the tangential component. We further connect this geometric picture to training dynamics through directional decompositions of the learned error, and make this link explicit for random feature models, where architectural bias and optimization error can be separated quantitatively. Experiments on synthetic multimodal data and MNIST latent diffusion support the predicted geometric behavior in both low and high dimensions.

 \end{abstract}

\section{Introduction}

Diffusion models~\citep{sohl2015deep, ho2020denoising,song2020score} now achieve state-of-the-art sample quality across a wide range of generative tasks, making them a central tool for image, audio, and video generation~\citep{dhariwal2021diffusion,kong2020diffwave,brooks2024video}. At the same time, it is increasingly important to understand how innovative these models actually are, that is, what they generate beyond the training data. A central concern is memorization: in some regimes, diffusion models can behave like stochastic parrots that simply reproduce training data, raising both scientific and practical concerns, including privacy and safety risks~\citep{carlini2023extracting,somepalli2023understanding,duan2023diffusion,liu2024generative}. Recent theory and empirical evidence suggest that non-memorizing behavior arises from various sources of error inside the learned diffusion model \citep[e.g.,][]{ye2025provable}, and in literature the term \emph{generalization} is often used in precisely this sense of non-memorization~\citep[e.g.,][]{kadkhodaie2023generalization, zhang2023emergence}.

This leads to the central question of the paper: 
\begin{center}
    \emph{When a diffusion model does not memorize the training set, where do its generated samples go?} 
\end{center}
Equivalently, once non-memorization has occurred, what geometric structure organizes the new samples produced by reverse-time inference? Our goal is to answer this question explicitly and quantitatively. In particular, we seek not merely to say that generated samples differ from the training data, but to describe where they are located relative to the geometry induced by the data.

This question should be understood in a fully data-dependent sense. Rather than taking an unknown population distribution as the primary reference, we take the finite training dataset itself as the object that defines the relevant geometry, and ask how generation departs from that geometry. In this sense, our focus is different from classical population-level generalization: we are not primarily asking how close the generated distribution is to an unknown population law, but how the model generates new samples relative to the observed data. This viewpoint is especially natural when one wants to understand structured intermediate generations between training samples, since the key issue is not only distributional discrepancy, but also the spatial organization of generated samples. We defer a more detailed comparison with related notions of generalization to Appendix~\ref{append:related work}.

Our answer is geometric. We construct a time-indexed family of log-density ridge manifolds from the smoothed empirical distribution and use it as the reference geometry for reverse-time inference. Relative to these ridges, generated samples follow a \emph{reach-align-slide} mechanism: they first enter a ridge neighborhood, then align toward it in normal directions, and finally move along it in tangent directions. The reach and align stages explain where generation concentrates, while the slide stage gives a directional description of how samples move along the ridge relative to nearby data-induced centers. Rather than fully characterizing the generated distribution within the tangent space, the slide analysis identifies a specific tangential motion that helps explain structured intermediate generations. Crucially, this geometric picture is tied to training: normal training error controls alignment to the ridge, while tangential training error controls sliding along it. Figure~\ref{fig:illustration} illustrates this mechanism on a semi-circular dataset.

\begin{figure}[H]
    \centering
    \begin{subfigure}[b]{0.24\linewidth} 
        \centering
        \includegraphics[width=\linewidth]{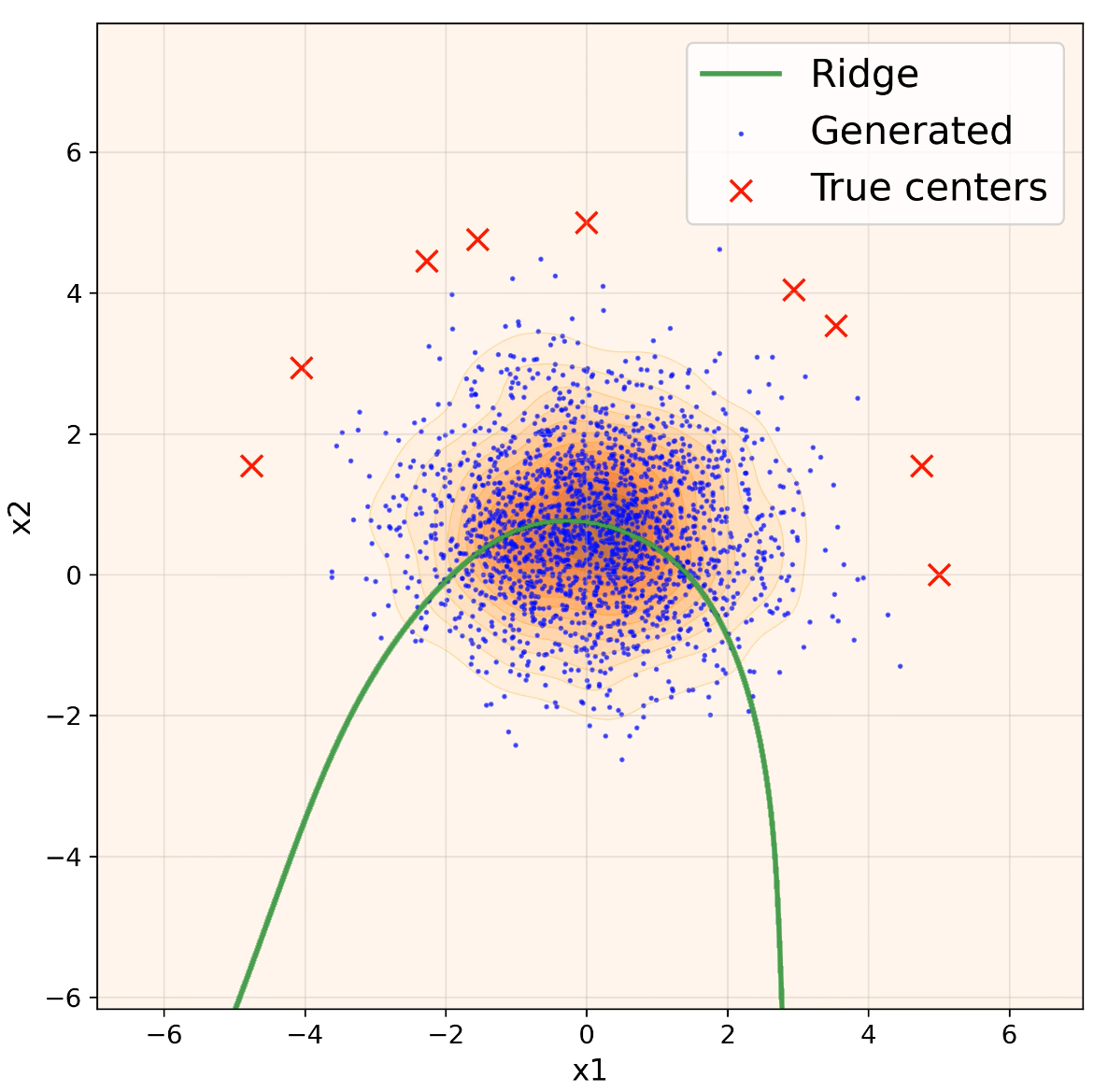}
        \vspace{-1.8em}
        \caption{early phase}
        \label{fig:circle init}
    \end{subfigure}
    \hfill 
    \begin{subfigure}[b]{0.24\linewidth}
        \centering
        \includegraphics[width=\linewidth]{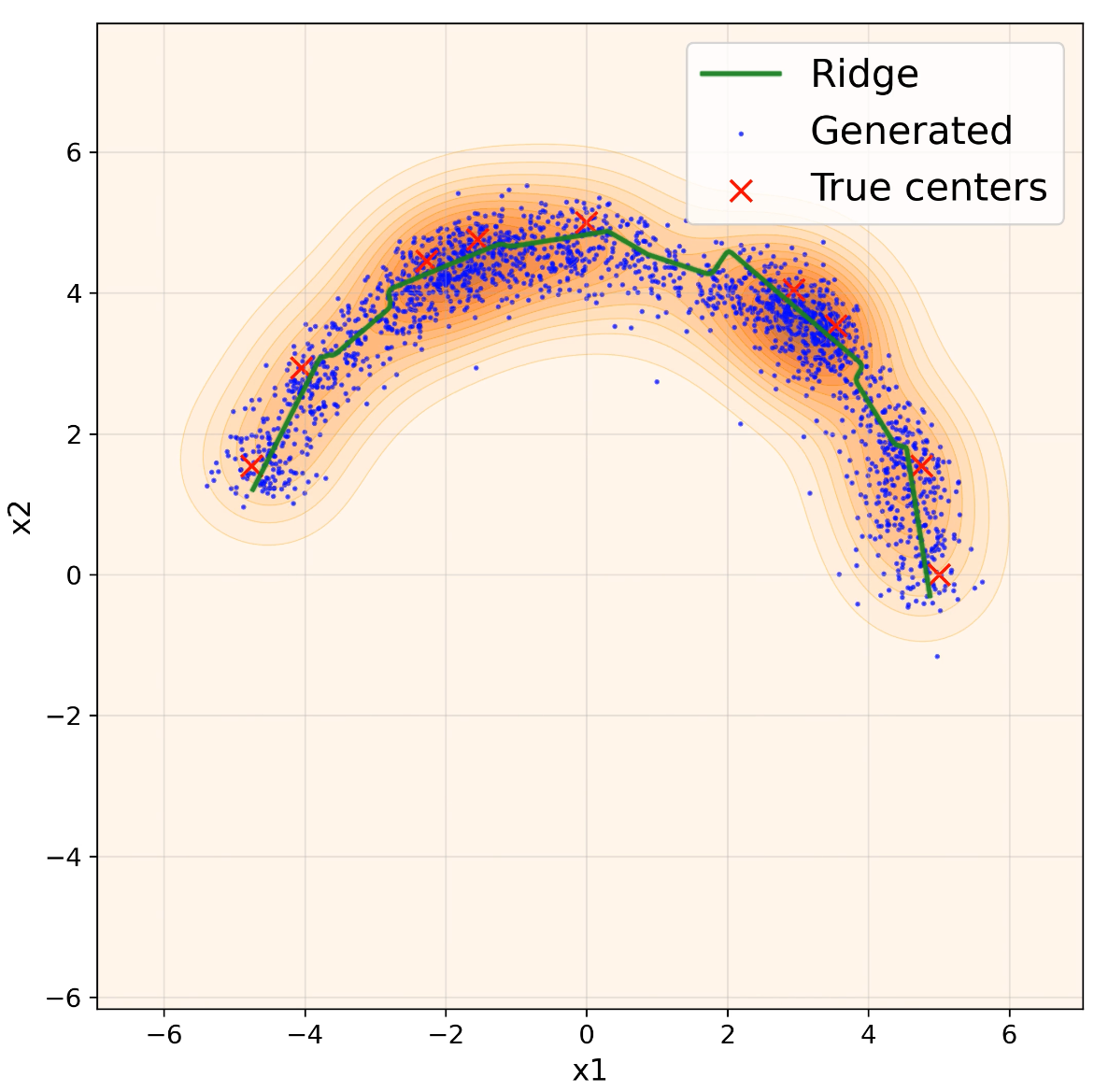}
        \vspace{-1.8em}
        \caption{reach}
        \label{fig:circle stage 1}
    \end{subfigure}
    \hfill
    \begin{subfigure}[b]{0.24\linewidth}
        \centering
        \includegraphics[width=\linewidth]{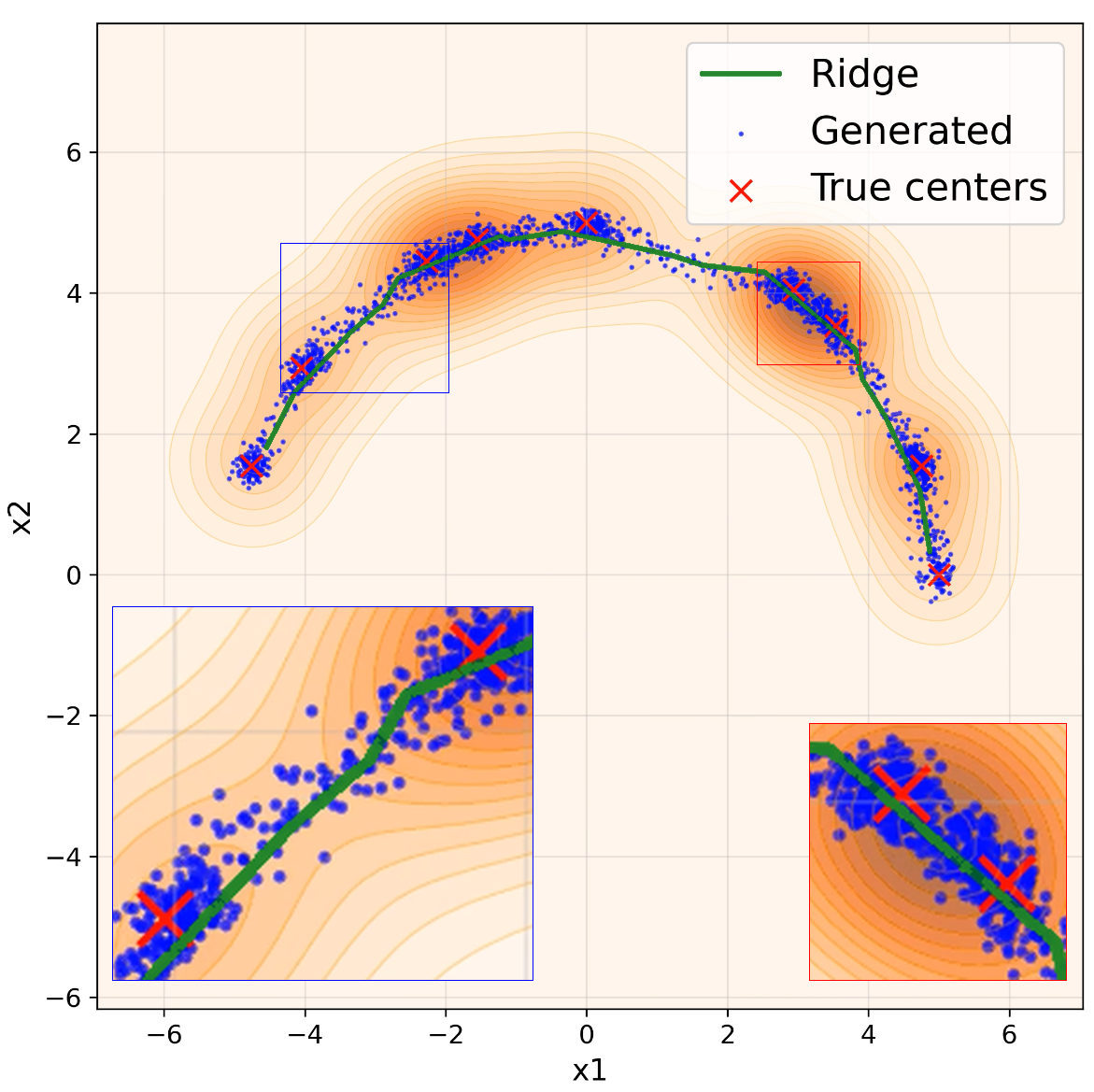}
        \vspace{-1.8em}
        \caption{align}
        \label{fig:circle stage 2}
    \end{subfigure}
    \hfill
        \begin{subfigure}[b]{0.24\linewidth}
        \centering
        \includegraphics[width=\linewidth]{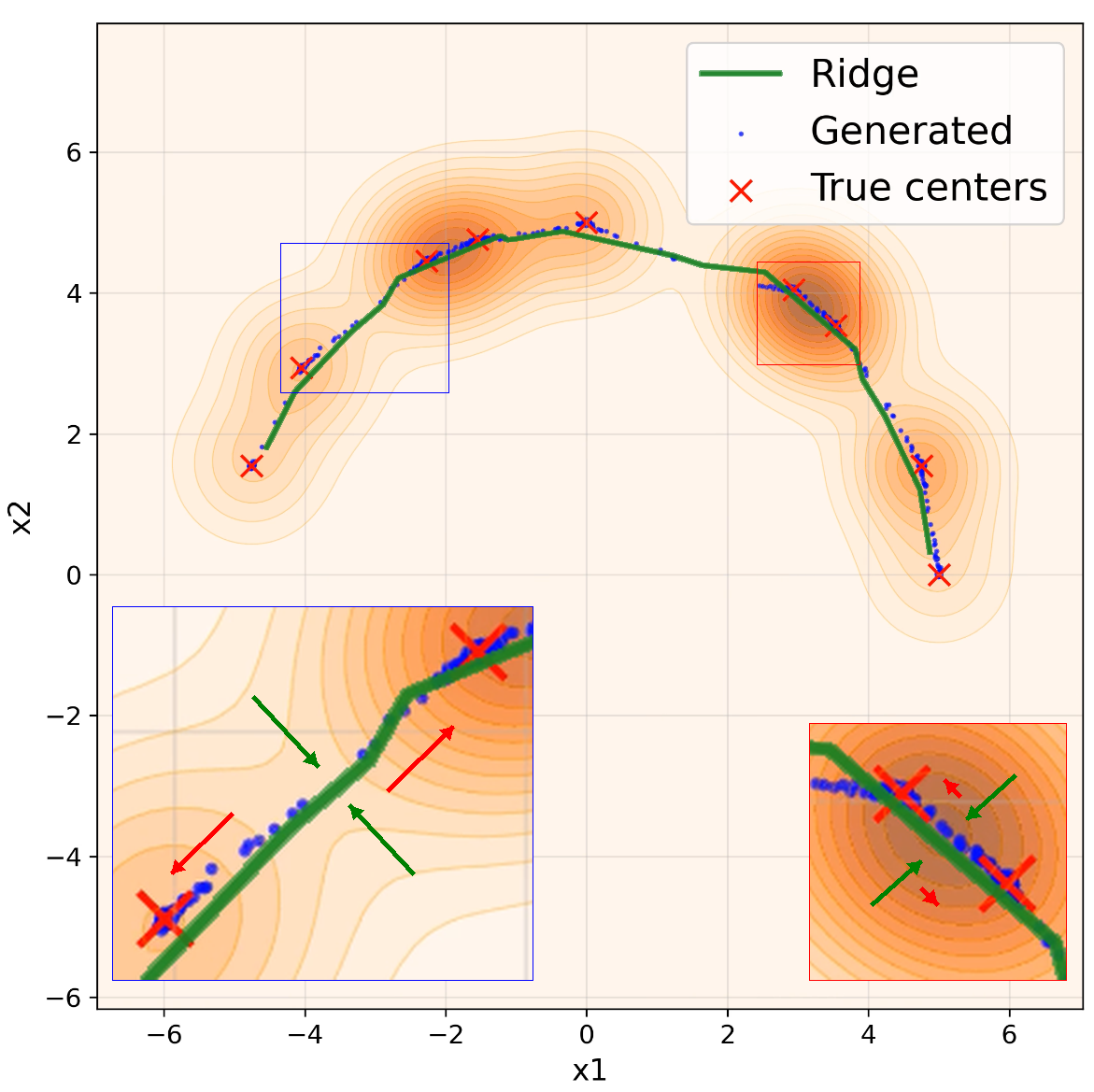}
        \vspace{-1.8em}
        \caption{slide}
        \label{fig:circle stage 3}
    \end{subfigure}
    \caption{\textbf{Reach-align-slide on semi-circular dataset:} $9$ unevenly spaced data points (\textcolor{red}{red} crosses) lie on a semi-circle of radius $4$ centered at the origin. Generated samples (\textcolor{blue}{blue} dots) evolve relative to the log-density ridge (\textcolor{green}{green} curve), exhibiting the reach--align--slide pattern. The zoom-in boxes in (c),(d) show region-dependent sliding: red arrows denote sliding directions and green arrows denote directions of continuation of alignment phase, with arrow lengths indicating intensity.
    }
    \label{fig:illustration}
\end{figure}

A tractable setting where the training-to-geometry link becomes explicit is random feature neural networks (RFNNs) trained by gradient descent. There we decompose training error into directional components and show how finite-width approximation error and incomplete optimization translate into different alignment and sliding behaviors. This RFNN analysis is not intended as a full model of modern architectures, but as a nonasymptotic example showing how architecture and optimization jointly shape diffusion generation through direction-dependent geometric effects.

Our perspective is most closely related to recent theoretical work that seeks to explain why diffusion models generate non-memorizing samples. One line of work studies whether generalization can already arise from the stochasticity or structure of the finite training target itself~\citep{vastola2025generalization,bertrand2025closed}; in contrast, we take the empirical training set as given and ask how the learned model generates relative to the geometry induced by that set. Another closely related direction analyzes training-induced bias, focusing on how model class, feature learning, or optimization shape the learned scores~\citep{kamb2024analytic,shah2025does,wu2025taking,bonnaire2025diffusion}. Our contribution is complementary: rather than only asking what bias training creates, we quantify how that bias appears during inference through distinct normal and tangential effects relative to a data-dependent ridge family. Finally, several recent works study inference-time behavior under structured settings or geometry-adaptive smoothing~\citep{baptista2025memorization,farghly2025diffusion,li2025scores}. Our work is closest in spirit to this direction, but differs in three ways. First, we make the relevant geometric object explicit from the empirical data through a time-dependent ridge family. Second, we characterize reverse-time inference relative to this geometry in a way that goes beyond concentration near a low-dimensional set: the reach–align–slide analysis also captures part of the tangential organization of generation relative to nearby data. Third, we connect this geometry back to directional components of training error.

Taken together, our framework answers the data-dependent generalization question by linking data geometry, inference dynamics, and training error in three steps: we
\begin{itemize}
\item[(1)] construct a time-dependent log-density ridge geometry from the smoothed empirical distribution, which serves as the reference object for where generation occurs (Section~\ref{sec:data assumption}).
\item[(2)] show that inference follows a reach--align--slide mechanism relative to this geometry: samples reach the ridge neighborhood, align in normal directions, and slide in tangent directions, thereby describing both concentration and partial tangential organization of non-memorizing generation (Sections~\ref{sec:stage1}--\ref{sec:stage3}).
\item[(3)] connect this geometry back to training through directional error components controlling alignment and sliding, with a nonasymptotic RFNN+GD instantiation separating architecture and optimization effects (Sections~\ref{subsec:general transfer}--\ref{subsec:RF+GD}).
\end{itemize}
Experiments in Section~\ref{sec:experiments} support this training-to-geometry picture.

\section{Preliminaries}\label{sec:prelim}

\noindent\textbf{SDE-based Diffusion Models.} To generate samples in $\mb{R}^d$ from data samples $x_0^{(1)},\cdots, x_0^{(n)}\in \mb{R}^d$, we consider the variance-preserving (VP) forward process $
    \dee X_t = -X_t \dee t +\sqrt{2}\dee B_t
$
for all $t\in [0,T]$, with marginals $p_t=\text{Law}(X_t)$ and $p_0=p$. The corresponding reverse process is
\begin{align}\label{eq:BD}
    \dee Y_t = \big( Y_t + 2\nabla\log p_{T-t} (Y_t) \big) \dee t +\sqrt{2}\dee \Bar{B}_t.
\end{align}
 Write $a_t=e^{-t}$, $h_t=1-e^{-2t}$ so that $X_t=a_t X_0+\sqrt{h_t}z$ in distribution with $z\sim \mc{N}(0,I_d)$ independent to $X_0$. We introduce the early stopping time $0<\delta\ll 1$ so that the score $\nabla\log p_t$ is used only on $[\delta,T]$.

\noindent\textbf{Denoising Mean Matching Loss.} We consider learning the posterior mean instead of the score. By Tweedie’s formula,
\begin{align}
    \nabla \log p_t(x) = - \frac{x}{h_t} + \frac{1}{h_t}\mb{E}_{x_0\sim p_{0|t}(\cdot|x)}[a_t x_0],
    \label{eq:score}
\end{align}
where $p_{0|t}(\cdot|x)$ denotes the law of $X_0$ given $X_t=x$. Define
\begin{align}
    m(t,x)\coloneqq \mb{E}_{x_0\sim p_{0|t}(\cdot|x)}[a_t x_0],
    \label{eq:posterior mean}
\end{align}
Thus learning the score is equivalent to learning the posterior mean, which we use throughout the paper. We measure approximation error through the mean matching loss
\begin{align}\label{eq:mean matching loss}
    \MM=\int_\delta^{T} \frac{w(t)}{h_t^2} \mb{E} \big[  \|  m_A(t,X_t) - m(t,X_t) \|^2 \big] \dee t,
\end{align}
where $m_A(t,x)$ is the learned posterior mean. In practice we train using its denoising version
\begin{align}\label{eq:DMM}
    \DMM & = \int_\delta^{T} \frac{w(t)}{h_t^2} \mb{E} \big[  \| -a_t X_0+ m_A(t,X_t) \|^2 \big] \dee t.
\end{align}
More details on $\DMM$ are discussed in Appendix \ref{append:loss}. The simulated reverse process, initialized at Gaussian, is 

\begin{align}\label{eq:BD approximate}
    \dee \Tilde{Y}_t = \big(  \Tilde{Y}_t + \frac{2(m_A(T-t,\Tilde{Y}_t)-\Tilde{Y}_t)}{h_{T-t}} \big) \dee t +\sqrt{2}\dee \Tilde{B}_t.
\end{align}

\noindent\textbf{Random Feature Neural Network and Gradient Descent.} The inference analysis in Section~\ref{sec:inference} is architecture-agnostic. To make the training-to-geometry link explicit in Section~\ref{subsec:RF+GD}, we use a random feature neural network (RFNN) parametrization of the posterior mean:

\begin{align*}
   m_A(t,x) \coloneqq \frac{A}{\sqrt{p}} \sigma (\frac{\Wx x}{\sqrt{d}}   + \frac{\Wt \varphi_t}{\sqrt{2K_t+1}} + b )\coloneqq \frac{A}{\sqrt{p}}\sigma_t(x) 
\end{align*}
where $\Wx,\Wt$ are fixed Gaussian random matrices, $\varphi_t$ is a Fourier time embedding, $b\sim \mc N(0,I_p)$, and only $A\in\mb R^{d\times p}$ is trained. With this parametrization, $\DMM(A)$ is a quadratic function of $A$, and full-batch gradient descent with step size $\eta$ satisfies
\vspace{-.15cm}
\begin{align}\label{eq:gd}
     A_{k+1} - A_k = -{2\eta} A_k \Tilde{U} & + {2\eta}\Tilde{V},  
\end{align}
with $\Tilde{U}  = \int_\delta^T \frac{w(t)}{h_t^2} \frac{\mb{E}[ \sigma_t(X_t)\sigma_t(X_t)^\intercal ]}{p}\dee t \in \mb{R}^{p\times p}$ and $\Tilde{V}  = \int_\delta^T \frac{w(t)}{h_t^2}\frac{\mb{E}_z [ a_t X_0 \sigma_t(X_t)^\intercal ]}{\sqrt{p}}\dee t \in \mb{R}^{d\times p}$. Further RFNN details are deferred to Appendix~\ref{append:gradient descent}.
\section{Geometric Properties of the Inference Process}\label{sec:inference}
In this section, we introduce the data-dependent ridge geometry that characterizes non-memorizing generation and study reverse-time inference relative to it. This yields the three-stage picture: generated samples first reach a neighborhood of the ridge, then align toward it in normal directions and slide along it in tangent directions.
Throughout, we work under the following empirical-data setting.
\begin{assumption}\label{assump:data} Data points $\{x_0^{(i)}\}_{i=1}^n$ are well-separated and bounded, i.e., $\Delta\coloneqq \min_{i\neq j} \| x_0^{(i)}-x_0^{(j)} \|>0$ and $R \coloneqq \max_i \| x_0^{(i)} \|<\infty$.
The data distribution $p$ is the \emph{empirical} distribution of the data, i.e., $p= \tfrac{1}{n}\sum_{i=1}^n \delta_{x_0^{(i)}}$. 
\end{assumption}

\subsection{Data-dependent Manifolds - Log-density Ridge Sets}\label{sec:data assumption} 
\vspace{-.1cm}
For the smoothed empirical distribution $p_t$, the log-density ridge is a geometric object that reflects the structure induced by the training data and serves as the reference object for describing generation; see Figure~\ref{fig:illustration} for a visual example. Intuitively, ridges are the ``directional local maximums'' of the log-density: a 0-dimensional ridge is an isolated mode, a 1-dimensional ridge traces a curve connecting high-density regions, and higher-dimensional ridges capture broader structures in the data.

\begin{definition}[Log-density Ridge Sets]\label{def:ridge} For any smooth probability density $p\in \mc{P}(\mb{R}^d)$ and any positive integer $d^*<d$, the $d^*$-dimensional log-density ridge set of $p$ with threshold $\beta>0$, denoted as $\mc{R}_{d^*}(p;\beta)$, is defined by
\vspace{-.1cm}
\begin{align*}
    \mc{R}_{d^*}(p;\beta)\coloneqq \left\{ x\in \mb{R}^d | E(x)E(x)^\intercal \nabla\log p(x)=0 , \lambda_{d^*+1}(x)\le -\beta \right\}
    \vspace{-.1cm}
\end{align*}
where $E(x)=(v_{d^*+1}(x),\cdots, v_d(x))\in \mb{R}^{d\times(d-d^*)}$ with $\{(\lambda_i(x),v_i(x))\}_{i=1}^d$ being the eigenvalues/eigenvectors of $\nabla^2 \log p(x)$ in descending order, i.e., $\lambda_1(x)\ge \lambda_2(x)\cdots\ge \lambda_{d^*}(x)>\lambda_{d^*+1}(x)\ge \cdots\ge \lambda_d(x)$, for all $x\in \mb{R}^d$.   
\end{definition}  
The first condition imposes stationarity in the normal directions given by the bottom eigenspace of $\nabla^2\log p(x)$, while the threshold condition enforces sufficient normal concavity, making the ridge identifiable at scale $\beta$. Unlike classical density ridges~\citep{genovese2014nonparametric,chen2015asymptotic}, this definition works in log-density space, which is natural for diffusion models because reverse-time dynamics is governed by $\nabla\log p_t$ and its local geometry by $\nabla^2\log p_t$.

For each $t\in [\delta,T]$, we apply Definition~\ref{def:ridge} to forward marginal $p_t$ and denote the ridge by $\mc{R}_t$. This yields a family of ridges varying with the noise level, which serves as the evolving reference geometry for reverse-time inference. The threshold is chosen at scale $\beta_t=\Theta(1/h_t)$, the natural curvature scale near the data, and also the scale needed for the later alignment estimates. See Remark~\ref{rem:choice of beta} for details.

\noindent\textbf{Tube neighborhood and projection map.} 
To analyze inference relative to $\mc{R}_t$, we work in a tube neighborhood where the nearest-point projection onto the ridge is well-defined. This requires the following smoothness-and-reach condition.

\begin{assumption}[Smoothness and positive reach]\label{assump:ridge set} For any $t\in [\delta, T]$, there exists $r_t>0$ such that $\mc{R}_t$ is a (piecewise) $C^2$-embedded submanifold in $\mb{R}^d$ with a reach no smaller than $r_t>0$, i.e., for all $x\in \mb{R}^d$ with $\mathrm{dist}(x,\mc{R}_t)\le r_t$, there exists a unique nearest point on $\mc{R}_t$.
\end{assumption}

The following proposition provides the projection estimates needed for the later dynamical analysis.
\begin{proposition}\label{prop:projection-regularity} Under Assumption \ref{assump:data}, the log-density ridge family $\{\mc{R}_t\}_{\delta\le t\le T}$ satisfies Assumption \ref{assump:ridge set}. More precisely, as $t\to \delta^+\ll 1$, the reach satisfies $r_t = \Omega(h_t^{2} \theta_t^{-1} R^{-3})$ for arbitrary $\theta_t=\exp(- o(h_{t}^{-1}) )$. For any radius $\rho_t\in (0,r_t)$, define the tube neighborhood
   \begin{align}\label{eq:tube ridge}
    \mc{T}_t(\rho_t)\coloneqq \{ x\in \mb{R}^d | \mathrm{dist}(x, \mc{R}_t)\le \rho_t \}.
\end{align} 
Then the nearest-point projection $\Pi_t: \mc{T}_t(\rho_t)\to \mc{R}_t$ is well-defined, and
\begin{itemize}
    \item [(1)] for all $ x\in \mc{T}_t(\rho_t)$, the displacement $ x-\Pi_t(x)$ lies in the normal space of $\mc{R}_t$ at $\Pi_t(x)$;
    \item [(2)]  $\Pi_t$ is $C^1$ on $\mc{T}_t(\rho_t)$ and $\sup_{x\in \mc{T}_t(\rho_t)} \| \nabla \Pi_t(x) \|\le \tfrac{1}{1-\rho_t/r_t}$;
    \item [(3)] if $\rho_t=\Theta(r_t)$, the ridge motion is uniformly bounded: $\sup_{x\in \mc{T}_t(\rho_t)}\| \partial_t\Pi_t(x)\| = \mc{O}\left( R\right)$.
\end{itemize}
\end{proposition}
In particular, this proposition gives a well-defined projection onto the ridge inside a tube neighborhood, identifies projection residuals as normal directions, and controls both the spatial stability of the projection map and the time variation of the ridge family.
\subsection{Stage 1 - Reaching the Tube Neighborhood}
\label{sec:stage1}
The later alignment and sliding analysis begin once the inference trajectory enters the projection tube. Define $ \Tilde{t}_{\mathrm{in}}  \coloneqq \inf \{ 0\le t\le T-\delta \mid \Tilde{Y}_t\in \mc{T}_{T-t}(\rho_{T-t}) \}$, which satisfies the following property:
\begin{theorem}[Informal, formal one in Theorem \ref{thm:formal stage 1}]\label{thm:informal stage 1} Under Assumption \ref{assump:data}, we have
\begin{align*}
    \mb{P}(\Tilde{t}_{\mathrm{in}}\le T-\delta)\ge 1-e_\delta -\varepsilon(T)-\sqrt{\varepsilon_A(T,\delta)/8},
\end{align*}
where $\lim_{\delta\to 0^+} e_\delta = 0$, $\lim_{T\to\infty}\varepsilon(T)= 0$ and $\varepsilon_A(T,\delta)\coloneqq \int_\delta^T h_t^{-2}\mb{E}[ \| m(t,X_t)-m_A(t,X_t) \|^2 ]\dee t$. 
\end{theorem}
Thus, with high probability, the learned process reaches the ridge neighborhood; the failure probability is controlled by early-stopping, large-time approximation, and global posterior-mean error.

\subsection{Stage 2 - Aligning along Normal Directions}
\label{sec:stage2}
After entry into the tube, (squared) off-ridge displacement is measured by $D_{T-t}(x) \coloneqq \| x - \Pi_{T-t}(x) \|^2 \coloneqq \| n_{T-t}(x) \|^2$. 
The key mechanism is contraction of this quantity along inference. In the main text, we state only the resulting bound at the final inference time $T-\delta$; the full time-resolved contraction estimate is given in Appendix~\ref{append:stage 2}.

\begin{theorem}\label{them:normal contraciton} Under Assumption \ref{assump:data}, let $e_A^\perp(t,x)\coloneqq P^\perp(\Pi_t(x))e_A(t,x)$ with $e_A=m_A-m$. Choose $\beta_t=c/h_t$ for $c\in [\frac{1}{2},1)$\footnote{$c$ can be chosen arbitrarily between $[\tfrac{1}{2},1)$ due to the property of $\nabla^2\log p_t(x)$ as explained in Remark~\ref{rem:choice of beta}. }. Then for $\delta\ll 1$, 
\begin{align*}
  \mb{E}[D_{\delta}(\Tilde{Y}_{T-\delta})] = \mc{O}\bigg(d\delta^c+ \delta^c\int_{\Tilde{t}_{\mathrm{in}}}^{T-\delta}  h_{T-u}^{-1-c}\,{\mb{E}[\| e_A^\perp(T-u,\Tilde{Y}_u)\|^2]} \dee u  \bigg). \vspace{-.1cm}
\end{align*}
\end{theorem}
Hence the final squared off-ridge displacement consists of a training-independent geometric term $d\delta^c$ and a cumulative contribution from the normal component of training residual. Thus, good normal alignment follows when the learned model has small error in directions transverse to the ridge. 

\subsection{Stage 3 - Sliding along Tangent Directions}\label{sec:stage3} 
The final stage concerns motion along the ridge. Near the end of inference, the smoothed empirical density $p_{s}$ is a Gaussian mixture centered at the transported training points $\{m_{s}^{(i)}\coloneqq a_{T-t}x_0^{(i)}\}_{i=1}^n$. In this regime, a trajectory typically enters a region where one mixture component is dominant. Inside such a region, the local tangent space of the ridge can be approximated using the top eigendirections of $\nabla^2\log p_{s}$, which allows us to define tangent coordinates relative to the nearby center. Define the $i^{th}$ center-dominant region $\mc{B}_{s}^{(i)}(\theta_{s}) \coloneqq \{ x\in \mb{R}^d \mid  \Softmax(-\tfrac{\| x-m_{s} \|^2}{2h_{s}})_i \ge 1-\theta_{s} \}$ where $\theta_{s}=\exp(- o(h_{s}^{-1}) )$ as $s\to 0^+$. For $\Tilde{Y}_t\in  \mc{B}_{T-t}^{(i)}(\theta_{T-t})$, define the tangent coordinate 
\begin{align}\label{eq:tangent error equation}
    \Tilde{u}_t^{(i)} &\coloneqq (U_{T-t}^{(i)})^\intercal  (\Tilde{Y}_t-m_{T-t}^{(i)}) \in \mb{R}^{d^*} 
\end{align}
where $U_{T-t}^{(i)}\in \mb{R}^{d\times d^*}$ contains the top-$d^*$ eigenvectors of $\nabla^2\log p_{T-t}(m_{T-t}^{(i)})$. The vector $\Tilde{u}_t^{(i)}$ measures the sample displacement along the local tangent directions. As in the normal-direction analysis, the key mechanism is a contraction along inference; in the main text we state only its consequence at the terminal time $T-\delta$. The full time-dependent estimate is deferred to Appendix~\ref{append:stage 3}.

\begin{theorem}\label{them:tangent contraciton} Under Assumption \ref{assump:data}, let $e_A^{\parallel,i}(t,x)=(U_{t}^{(i)})^\intercal e_A(t,x)$ with $e_A=m_A-m$. If $\Tilde{Y}_t\in \mc{B}_{T-t}^{(i)}(\theta_{T-t})$,  then for $\delta\ll 1$, 
\begin{align*}
     \mb{E}[\| \Tilde{u}_{T-\delta}^{(i)} \|^2] = \mc{O}\bigg( d\sqrt{\delta} + \sqrt{\delta} \int_{\Tilde{t}_{\mathrm{in}}}^{T-\delta}  h_{T-u}^{-\frac{3}{2}}{\mb{E}[\|e_A^{\parallel,i}(T-u,\Tilde{Y}_{u}) \|^2]} \dee u  \bigg). 
\end{align*}
\end{theorem}
The final amount of sliding along the ridge is controlled by a training-independent term $d\sqrt{\delta}$ and a cumulative contribution from the tangential training residual. Hence samples may align closely with the ridge without collapsing onto the training points, leaving structured intermediate generations.
\begin{remark}[Effect of training weight on generation]\label{rem:effect_of_weight}
Suppose the per-time contribution satisfies 
$w(t)h_t^{-2} \mathbb{E}[\|e_A(T-t,\Tilde{Y}_t)\|^2]=\mathcal{O}(1),
$
then the cumulative mean-error terms in Theorems \ref{them:normal contraciton} and \ref{them:tangent contraciton} scale as
$
\delta^{c}\int_{\delta}^{T-\tilde t_{\mathrm{in}}}{h_t^{1-c}}/{w(t)}dt $ and $ \delta^{1/2}\int_{\delta}^{T-\tilde t_{\mathrm{in}}}{h_t^{1/2}}/{w(t)}dt
$
respectively. Thus larger weight near small $t$ suppresses end-stage errors and pushes generation toward memorization. For $w(t)=h_t^2,h_t,1$, these terms scale as
$\mathcal{O}(1)$, $\mathcal{O}(\delta)$, and $\mathcal{O}(\delta^2)$, respectively.
\end{remark}

\section{From Training Error to Generation Geometry}\label{sec:training to inference} 
Section~\ref{sec:inference} reduced non-memorizing generation to directional errors: normal error controls alignment to the ridge, while tangential error controls spread along it. We now connect these errors to training, first through projected posterior-mean matching losses and then through an RFNN+GD example where these losses split into architecture- and optimization-driven terms.
\subsection{Directional Decomposition of Training Loss}\label{subsec:general transfer} 
For $\dagger\in\{\perp,\parallel\}$, define
\begin{align*}
    \MM^\dagger(A)
\coloneqq 
\int_\delta^T
\frac{w(t)}{h_t^2}
\mb E\big[\|P_t^\dagger(X_t)e_A(t,X_t)\|^2\big]\dee t,
\qquad e_A=m_A-m.
\end{align*}
Then $\MM=\MM^\perp+\MM^\parallel$, with $\MM^\perp$ and $\MM^\parallel$ the normal and tangent training errors respectively.
\begin{theorem}[Directional training losses control generation geometry]\label{thm:informal inference to training} Under mild assumptions, the normal and tangential errors in Theorems \ref{them:normal contraciton} and \ref{them:tangent contraciton} can be estimated by projected posterior mean matching loss in corresponding directions:
\begin{align*}
    \textnormal{normal-error bound} &\lesssim C_\delta^\perp\MM^\perp + d\delta^c + C_\delta^\perp(\sqrt{d}+R) e^{-T}, \\
    \textnormal{tangent-error bound} &\lesssim C_\delta^\parallel\MM^\parallel + d\delta + C_\delta^\parallel(\sqrt{d}+R)e^{-T},
\end{align*}   
where $C_\delta^\perp\coloneqq \delta^c  \big(1 \vee \frac{\delta^{1-c}}{w(\delta)}\big)$, $C_\delta^\parallel\coloneqq \delta^\frac{1}{2}  \big(1 \vee \frac{\delta^{\frac{1}{2}}}{w(\delta)}\big)$ and $c=\lim_{t\to\delta} h_t\beta_t$ is arbitrary in $[\tfrac{1}{2},1)$.
\end{theorem}
Theorem~\ref{thm:informal inference to training}, proved in Appendix~\ref{append:from inference to training}, turns the pathwise error terms from Section~\ref{sec:inference} into training-side quantities. Up to remainders controlled by $\delta$ and $T$, normal loss predicts alignment and tangential loss predicts sliding.

\subsection{RFNN: Architecture and Optimization Effects on Generation}
\label{subsec:RF+GD}
We now specialize the training-to-geometry connection to RFNN trained by gradient descent. In this setting, each directional training loss splits into an \emph{architecture} term, reflecting the finite-width approximation floor, and an \emph{optimization} term, reflecting incomplete training from initialization. This separates how model class and optimization shape alignment to the ridge and spreading along it.
\begin{theorem}[Architecture and optimization control geometry]\label{thm:directional decomposition}
Under conditions in Theorem \ref{thm:informal inference to training}, let $\{A_k\}_{k\ge 0}$ follow GD in \eqref{eq:gd} with learning rate $\eta<\tfrac{2}{\lambda_1}$, then up to remainders controlled by $\delta,T$,
\begin{itemize}
    \item the normal error at training step $k$ is bounded by $C_\delta^\perp(\Err_{arc}^\perp+\Err_{train}^\perp(k))$;
    \item the tangential error at training step $k$ is bounded by $C_\delta^\parallel(\Err_{arc}^\parallel+\Err_{train}^\parallel(k))$.
\end{itemize}
For each $\dagger \in \{\perp,\parallel \}$, $\Err_{arc}^\dagger$ is the architecture-driven term and $\Err_{train}^\dagger(k)$ is the optimization-driven term; their explicit formulas are given in Appendix~\ref{append:gradient descent}.
\end{theorem}
Theorem~\ref{thm:directional decomposition} refines the training-to-geometry link in the RFNN setting: each directional generation error is controlled by a finite-width approximation floor plus a finite-time optimization error. Thus the ridge-based analysis separates how model class and training procedure affect generation, and it does so separately in normal and tangential directions. The two-point example below makes the dependence on initialization and the RFNN kernel spectrum explicit.

\noindent\textbf{Fully Explicit Results of Two-point Data.} WLOG assume that $x_0^{(1)}=(- \mu ,0)$ and $x_0^{(2)}=(\mu,0)$. Then the ridge $\mc{R}_t \equiv \{x_2=0\}$ and the posterior mean $m(t,x)=( a_t\mu \tanh ( \tfrac{a_t\mu}{h_t}x_1 ), 0 )$. 

Writing the GD initialization $A_0 = (A_{0,1},A_{0,2})^\intercal$, the optimization-driven errors becomes,
\begin{align*}
    \Err_{train}^\parallel(k) = \sum_{i=1}^r \lambda_i (1-2\eta \lambda_i)^{2k} ( A_{0,1}^\intercal u_i - \tfrac{\Tilde{v}^\intercal u_i}{\lambda_i} )^2 , \  \Err_{train}^\perp(k) = \sum_{i=1}^r \lambda_i (1-2\eta \lambda_i)^{2k} ( A_{0,2}^\intercal u_i )^2.
\end{align*} 
where $\{(\lambda_i,u_i)\}_{i=1}^r$ is spectrum of $\Tilde{U}$ and $\Tilde{v}$ the first row of $\Tilde{V}$ for $\Tilde{U},\Tilde{V}$ in \eqref{eq:gd}. In particular, the normal optimization error depends only on the second row of $A_0$: if $A_{0,2}=0$, then $\Err_{train}^\perp(k)\equiv 0$, so alignment to the ridge is immediate. By contrast, if $A_{0,2}$ is aligned with the slowest spectral mode, then $\Err_{train}^\perp(k)$ decays only at small rate, which delays normal alignment to the ridge.

In contrast, the architecture-driven error is purely tangential: $\Err_{arc}^\perp=0$, while 

\begin{align*}
    \Err_{arc}^\parallel
    = \int_\delta^T \tfrac{w(t)}{h_t^2}\mb{E}_{x\sim p_t}[  a_t^2\mu^2 \tanh ( \tfrac{a_t\mu}{h_t}x_1 )^2 ] \dee t  - \Tilde{v}^\intercal \Tilde{U}^+ \Tilde{v} \vspace{-.1cm}
\end{align*}
is strictly positive at finite $p$.
Hence samples can align strongly to the ridge while still spreading along it, producing edge-like interpolation between the two data points. As width increases, this tangential floor shrinks, and the behavior becomes increasingly memorization-like.

This example isolates the two core mechanisms of the paper: optimization error can delay normal alignment, especially through slow spectral modes, while finite-width architecture error can sustain tangential spreading along the ridge. Their combination yields non-memorizing generation with strong ridge alignment but persistent along-ridge spread. Numerical illustrations are given in Appendix~\ref{append:initialization}.
\section{Experiments}\label{sec:experiments}
We evaluate the proposed geometric framework from three complementary perspectives. First, 2D examples illustrate that log-density ridges identify the geometric support of non-memorizing generation, including curved and edge-like structures. Second, the two-point problem provides a fully explicit setting where the predicted normal/tangential quantities can be compared directly with generated samples and training losses. Third, MNIST latent diffusion tests whether the same reach--align--slide picture remains informative in higher dimension. Experimental details are deferred to Appendix~\ref{append:experimental detail}, and additional experiments are presented in Appendix~\ref{append:more experiments}.
\subsection{2D Illustrations of the Role of Ridge}
We begin with qualitative 2D examples whose purpose is to show what the ridge geometry captures. Unlike the two-point setting later, these examples are not used for quantitative verification; instead, they demonstrate that the log-density ridge can identify nontrivial data-induced structures along which generated samples concentrate.
\begin{figure}[htbp]
\captionsetup{font=footnotesize}
    \centering
    \begin{minipage}{0.48\columnwidth}
        \centering
        \includegraphics[width=0.95\linewidth]{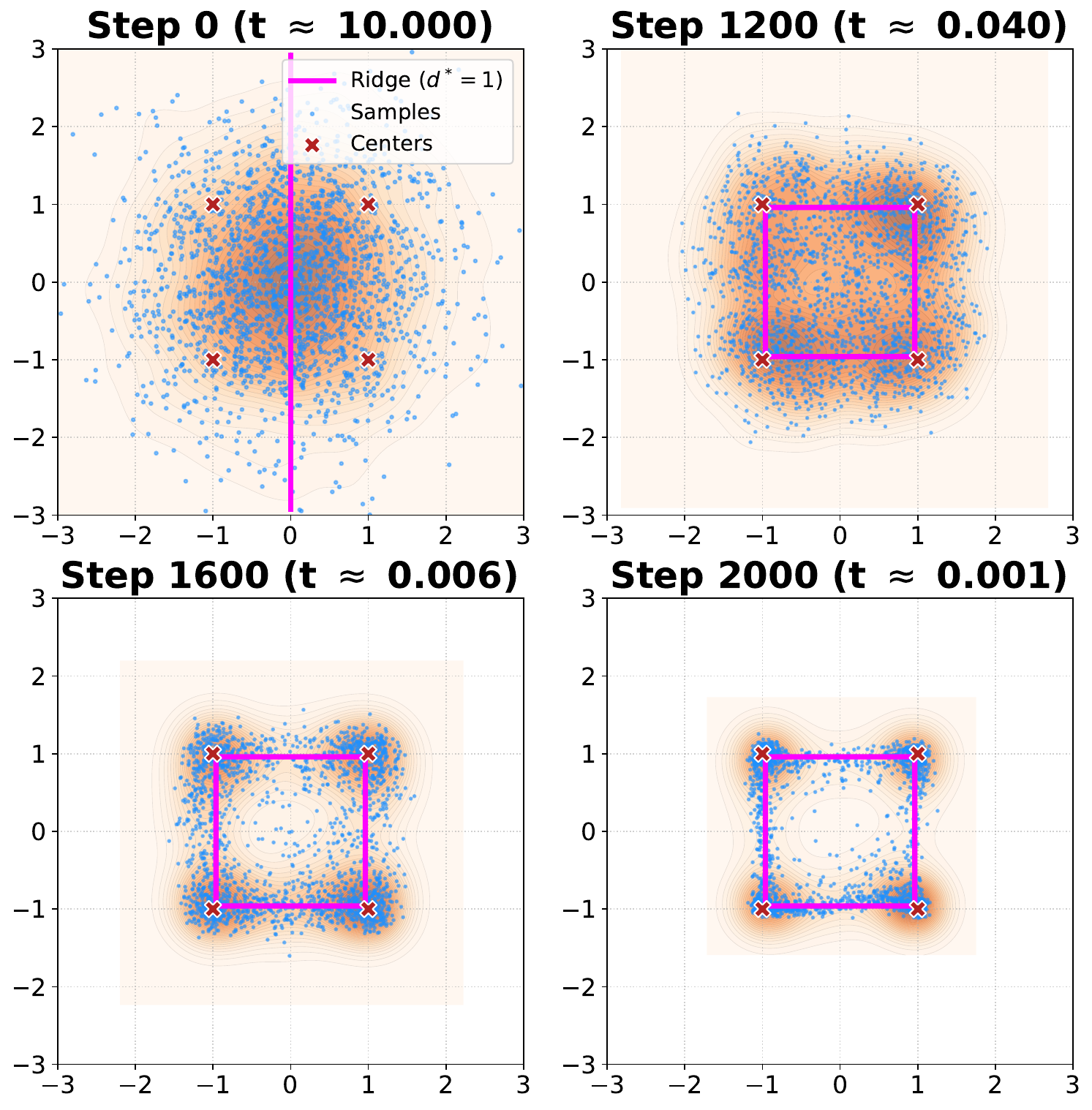}   
        \caption{Generalization from 4 training points.}
        \label{fig:4pointsquare}
    \end{minipage}
    \hfill 
    \begin{minipage}{0.48\columnwidth}
        \centering
        \includegraphics[width=0.95\linewidth]{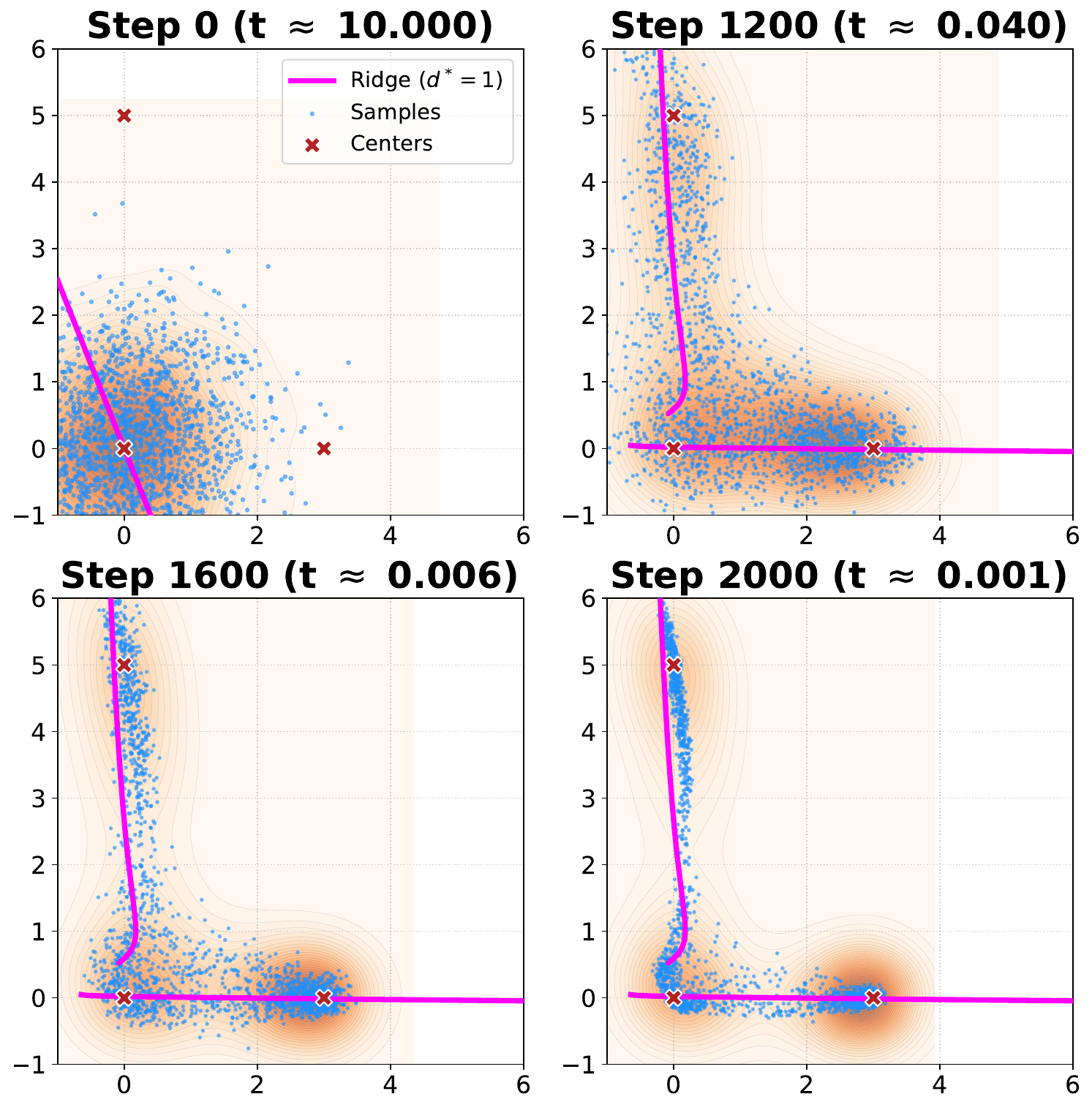}     
        \caption{Generalization from 3 training points.}
        \label{fig:3point3005}
    \end{minipage}
\end{figure}

For four training points at $(\pm1,\pm1)$, Figure~\ref{fig:4pointsquare} shows that generated samples concentrate along edge-like structures that are not themselves training samples. The moving ridge closely tracks these structures, indicating that it predicts where non-memorizing generation occurs. The three-point example in Figure~\ref{fig:3point3005} further shows that this behavior is not merely straight-line mode interpolation: for training points $(0,0)$, $(3,0)$, and $(0,5)$, the ridge is visibly bent and the generated samples follow this curved geometry. Thus, even in simple 2D settings, the ridge captures data-induced generation geometry beyond isolated modes or straight interpolation paths.
\subsection{Two Points in 2D Plane}\label{subsec:synthetic data}
We next use the two-point dataset $\{(\pm3,0)\}$ as a quantitative test case. Here the ridge is exactly the horizontal axis, with tangent and normal directions $e_1=(1,0)$ and $e_2=(0,1)$. This makes it possible to compare three quantities directly: the generated sample geometry, the normal/tangential inference errors from Section~\ref{sec:inference}, and the normal/tangential training losses from Section~\ref{sec:training to inference}. We report RFNN results here; MLP experiments and initialization effects are deferred to Appendix~\ref{append:more experiments}.

\noindent\textbf{Directional geometry and its training origin.}
Figure~\ref{fig:RFNN_comprehensive_training} summarizes the full comparison. Panels (a)--(c) show generated samples under $w(t)\in\{1,h_t,h_t^2\}$; panel (d) reports the corresponding normal and tangential inference errors; and panel (e) reports the corresponding directional training losses.
\begin{figure}[htbp]
    \centering
    \captionsetup{font=footnotesize}

    \begin{minipage}[t]{0.53\textwidth}
        \centering

        \begin{minipage}[t]{0.325\linewidth}
            \centering
            \includegraphics[width=\linewidth]{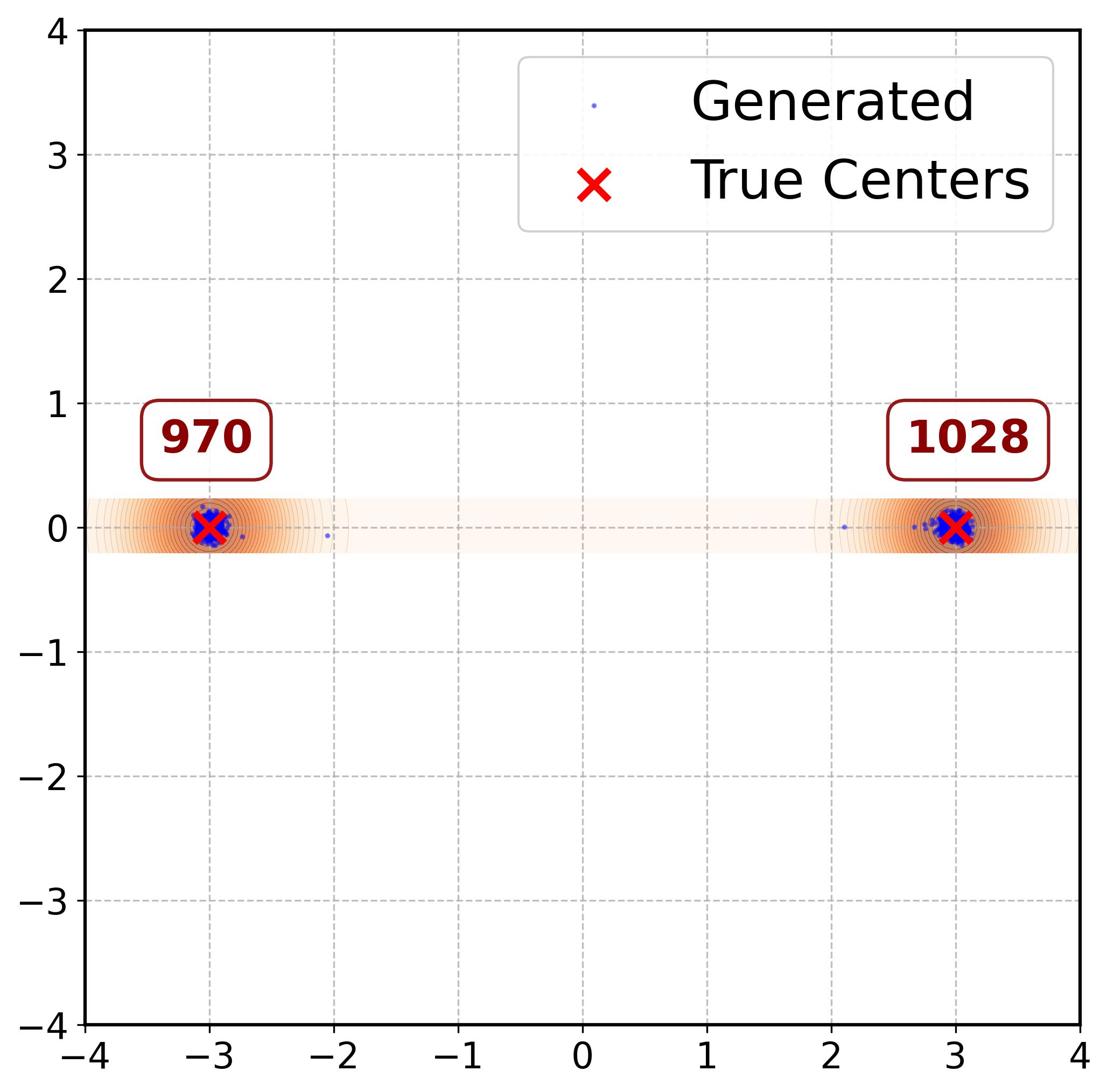}\\[-0.3em]
            {\scriptsize (a) $w(t)=1$}
        \end{minipage}
        \hfill
        \begin{minipage}[t]{0.325\linewidth}
            \centering
            \includegraphics[width=\linewidth]{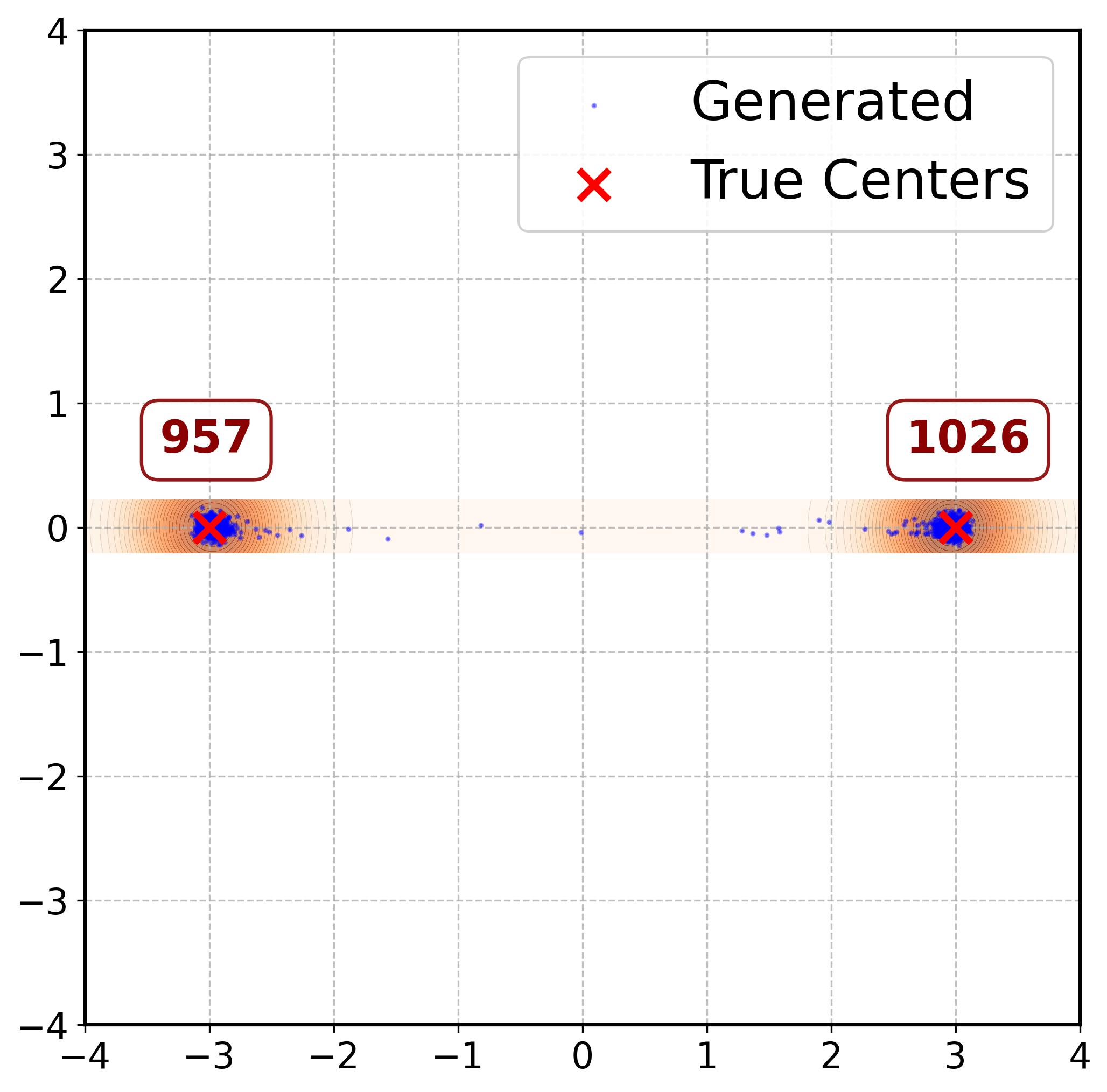}\\[-0.3em]
            {\scriptsize (b) $w(t)=h_t$}
        \end{minipage}
        \hfill
        \begin{minipage}[t]{0.325\linewidth}
            \centering
            \includegraphics[width=\linewidth]{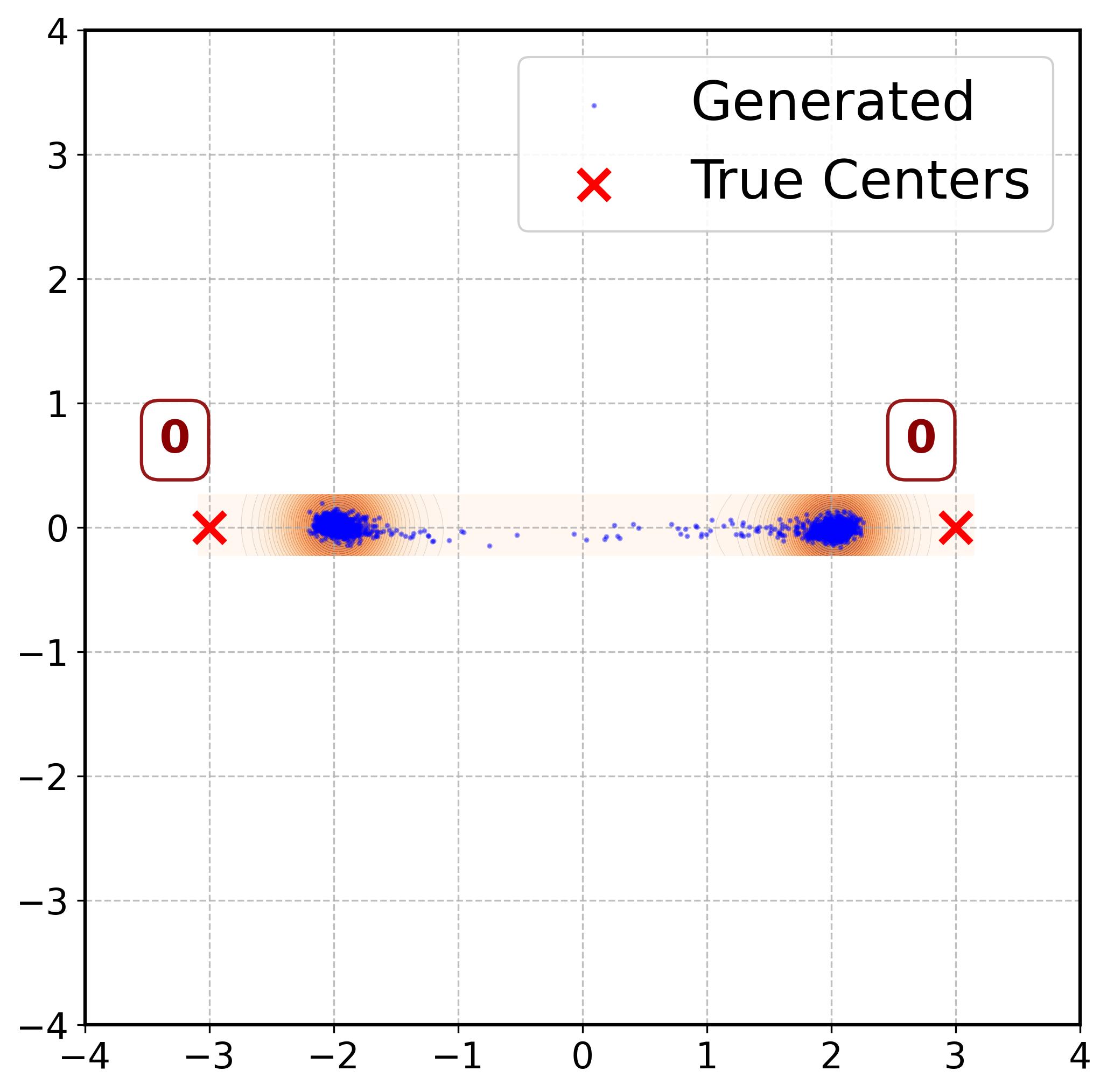}\\[-0.3em]
            {\scriptsize (c) $w(t)=h_t^2$}
        \end{minipage}
    \end{minipage}
    \hfill
    \begin{minipage}[t]{0.45\textwidth}
        \centering

        \begin{minipage}[t]{0.48\linewidth}
            \centering
            \includegraphics[width=\linewidth]{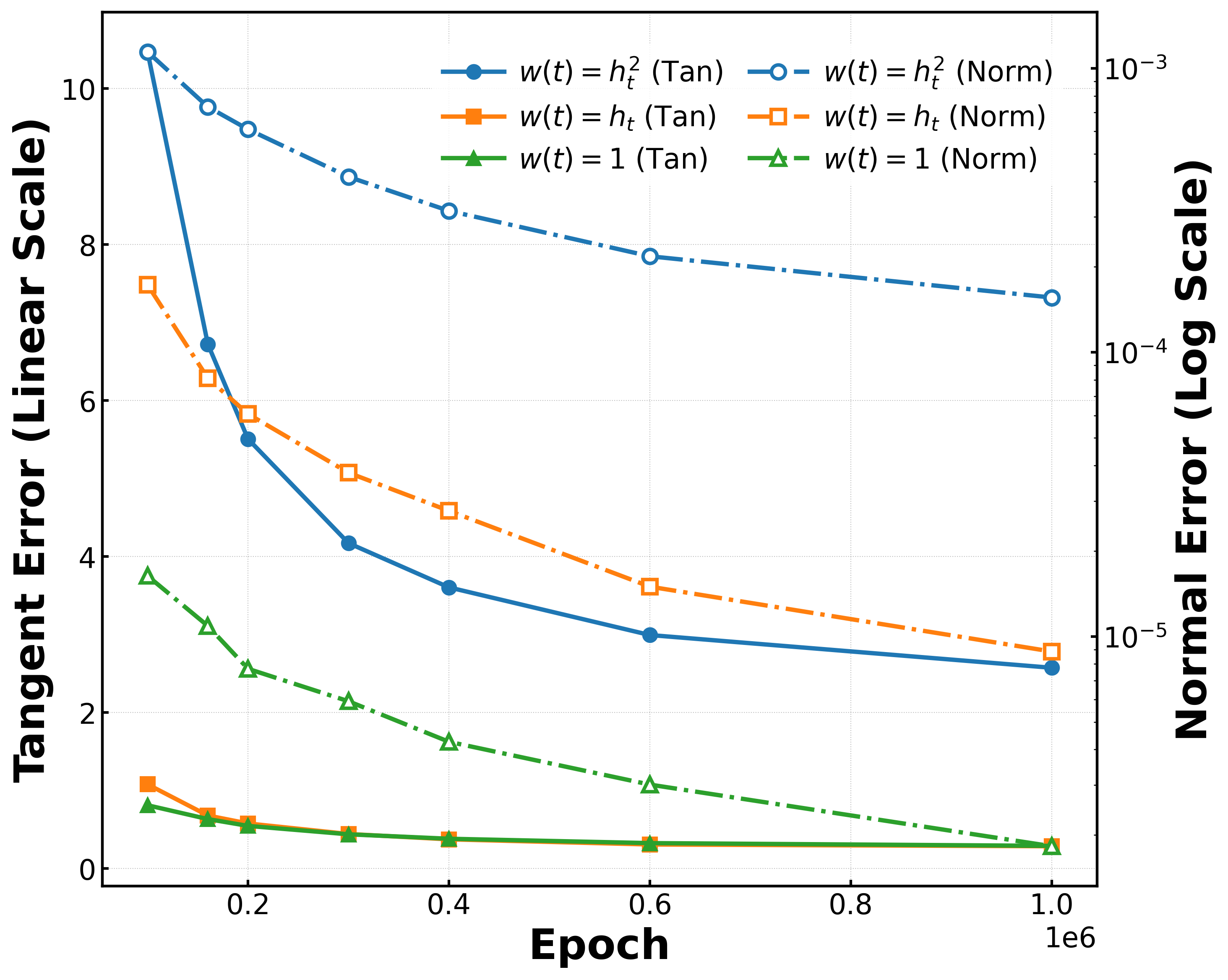}\\[-0.3em]
            {\scriptsize (d) Error Dynamics}
        \end{minipage}
        \hfill
        \begin{minipage}[t]{0.48\linewidth}
            \centering
            \includegraphics[width=\linewidth]{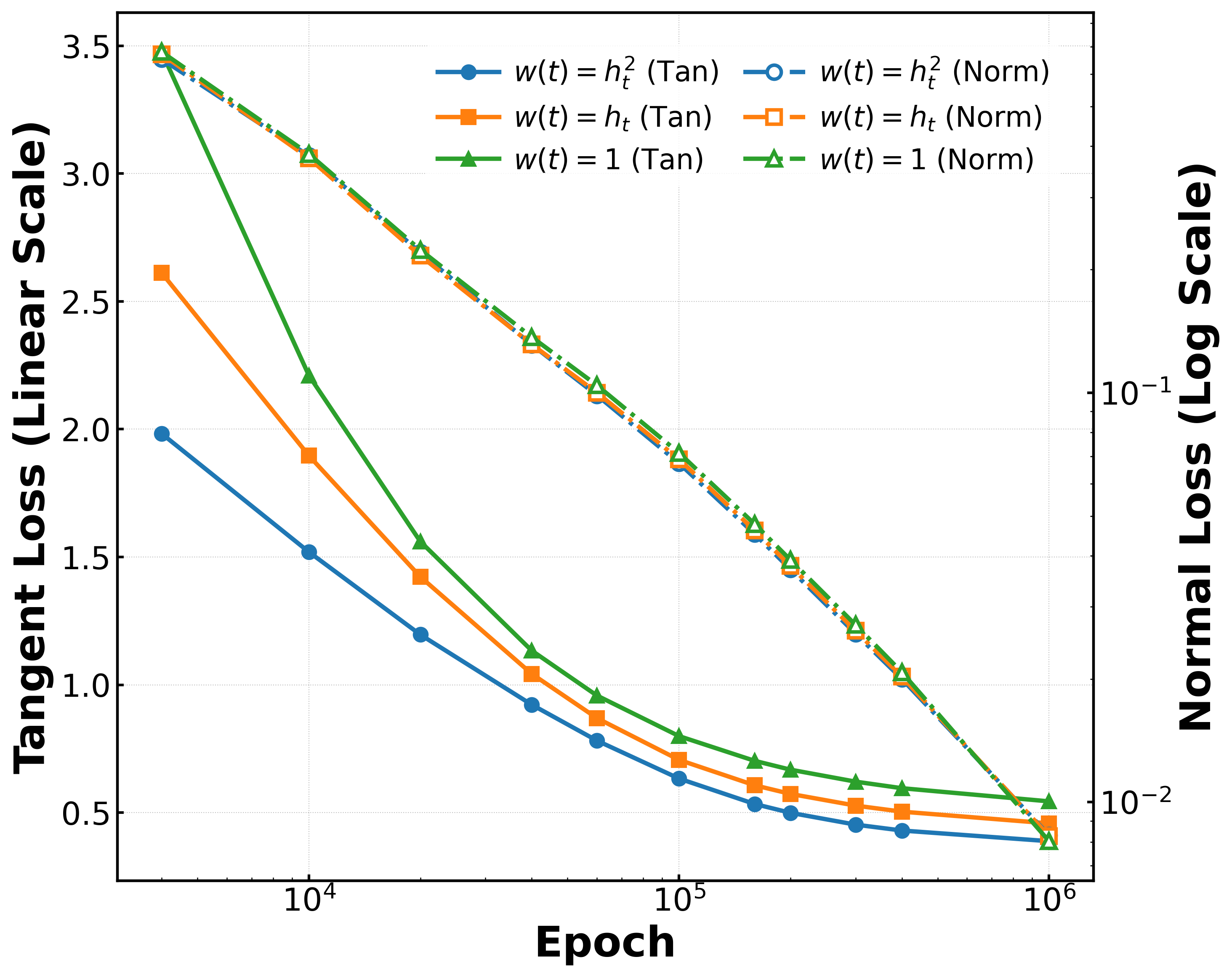}\\[-0.3em]
            {\scriptsize (e) Loss Decomposition}
        \end{minipage}
    \end{minipage}
    \caption{\textbf{RFNN generated samples, error dynamics, and training loss decomposition.}
    Left: generated samples under $w(t)=1,h_t,h_t^2$; boxed numbers count samples within radius $0.5$ of the target modes, and background color shows the KDE. Right: evolution of tangent and normal inference errors and training losses under the same weights. Solid lines show tangent-direction quantities on a linear scale; dash-dot lines show normal-direction quantities on a log scale.}
    \label{fig:RFNN_comprehensive_training}
\end{figure}

Panels (a)--(d) validate the geometric prediction from Section~\ref{sec:inference}. Across all three schedules, the normal error remains extremely small, and generated samples stay tightly concentrated near the ridge $y=0$. The tangential behavior, however, changes substantially with the weight schedule: $w(t)=1$ gives the smallest tangential error and samples concentrate near the two data points; $w(t)=h_t^2$ gives the largest tangential error and produces the most pronounced edge-like spread; $w(t)=h_t$ lies between these extremes. Thus, once normal alignment is achieved, the remaining tangential error determines how much generation spreads along the ridge.

Panel (e) connects this geometry back to training. The normal training loss is small across schedules, explaining the uniformly strong normal alignment. The tangential loss alone, however, does not fully determine the final geometry: $w(t)=h_t^2$ can have relatively small tangential training loss but still produce large tangential spread, because the coefficient $C_\delta^\parallel$ in Theorem~\ref{thm:informal inference to training} amplifies end-stage error when $w(\delta)$ is small. Conversely, $w(t)=1$ has weaker amplification and therefore produces smaller tangential spread despite larger raw tangential loss. This confirms the training-to-geometry mechanism: directional training losses, together with their sensitivity coefficients, control directional inference errors and hence the observed sample geometry.
\subsection{Higher Dimension Data - MNIST}
Finally, we test whether the same geometric picture remains meaningful in higher dimension. We study a binary MNIST problem in latent space and ask whether generated samples exhibit the two behaviors predicted by the theory: normal alignment toward the ridge geometry and limited tangential sliding toward the training data.
\begin{figure}[h]
    \centering
    \captionsetup{font=footnotesize}
    \includegraphics[width=0.95\linewidth,height=6cm]{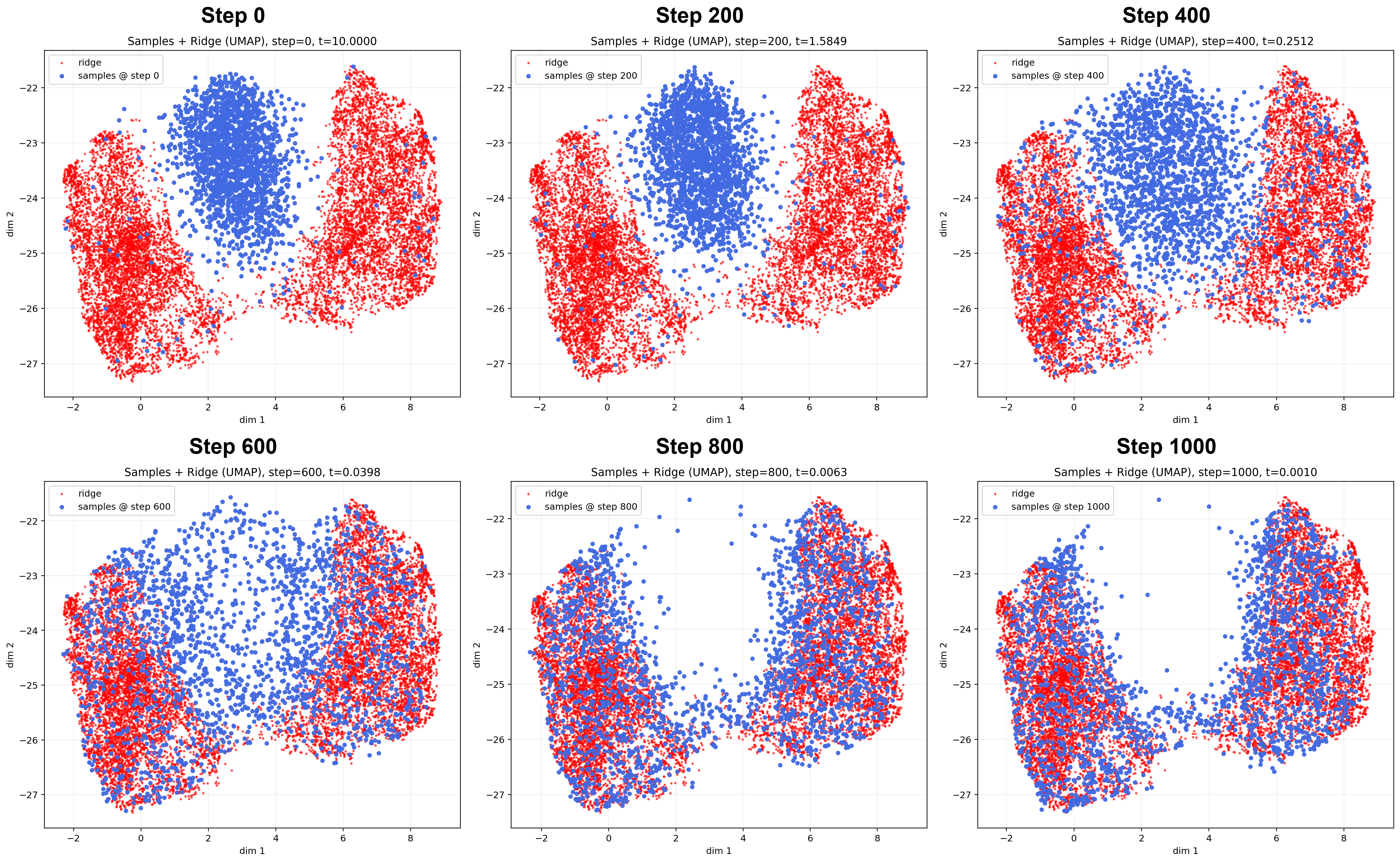}
    \caption{\textbf{UMAP visualization of generated samples and ridge.} Red dots represent underlying ridge structure $\mathcal{R}_t$ with $t=0.001$; blue dots represent generated samples along inference steps. The ridge accurately captures the sample distribution.}
    \label{fig:umap_mnist}
\end{figure}

\noindent\textbf{Experimental setup.}
We consider the digits 4 and 8 from MNIST. To simplify the problem, we first train a VAE to embed the images into a 32-dimensional latent space, and then train a time-conditioned MLP score model in that latent space with weight $w(t)=1$. Full architectural and training details are deferred to Appendix~\ref{sec:MNIST training detail}.

\noindent\textbf{Qualitative Visualization.} Figure~\ref{fig:umap_mnist} gives a qualitative view of the generation dynamics through UMAP. Early samples are far from the ridge structure, while later samples move toward the global geometry captured by the ridge. This supports the ridge-based description at a visual level, although UMAP does not preserve the normal/tangent decomposition needed for quantitative verification.

\noindent\textbf{Quantitative normal and tangential behavior.} We next quantify the two directional effects. Since the true latent distribution is unknown, we estimate distance to the ridge by solving the corresponding constrained optimization problem; details on solving this optimization problem and the associated stability analysis are given in Appendix \ref{app:newton_distance} and \ref{sec:stability}, respectively. Figure~\ref{fig:combined_ridge_metrics}(a) shows that the mean ridge distance over 200 trajectories decreases through most of inference and then stabilizes at a small floor after around 900 steps, matching the predicted normal-alignment stage.

Tangential motion follows a different pattern. Figures~\ref{fig:combined_ridge_metrics}(b)--(c) show that tangential motion decreases more slowly and becomes negligible only near the end of inference. This temporal imbalance is the main observation: inference spends most of its time aligning samples toward the ridge, leaving limited time for substantial sliding toward exact training examples. Consequently, generated samples organize around the ridge without fully collapsing onto the training set, matching the proposed non-memorizing mechanism.

\begin{figure}[h]
    \centering
    \captionsetup{font=footnotesize}
    \begin{minipage}[b]{0.33\linewidth}
        \centering
        \includegraphics[width=\linewidth,height=3cm]{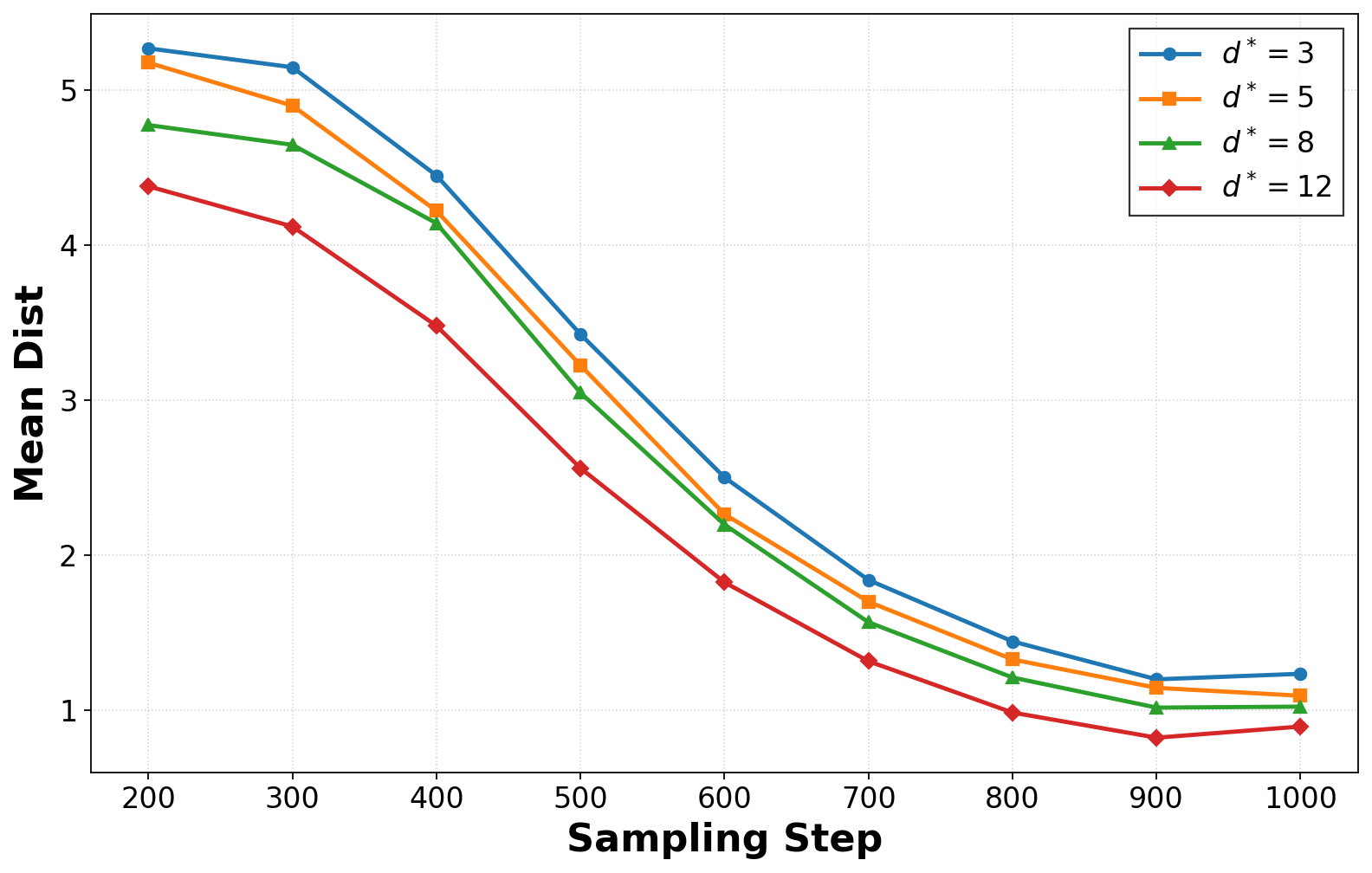}\\[-0.3em]
        {(a) distance to ridge}
    \end{minipage}\hfill
    \begin{minipage}[b]{0.33\linewidth}
        \centering
        \includegraphics[width=\linewidth,height=3cm]{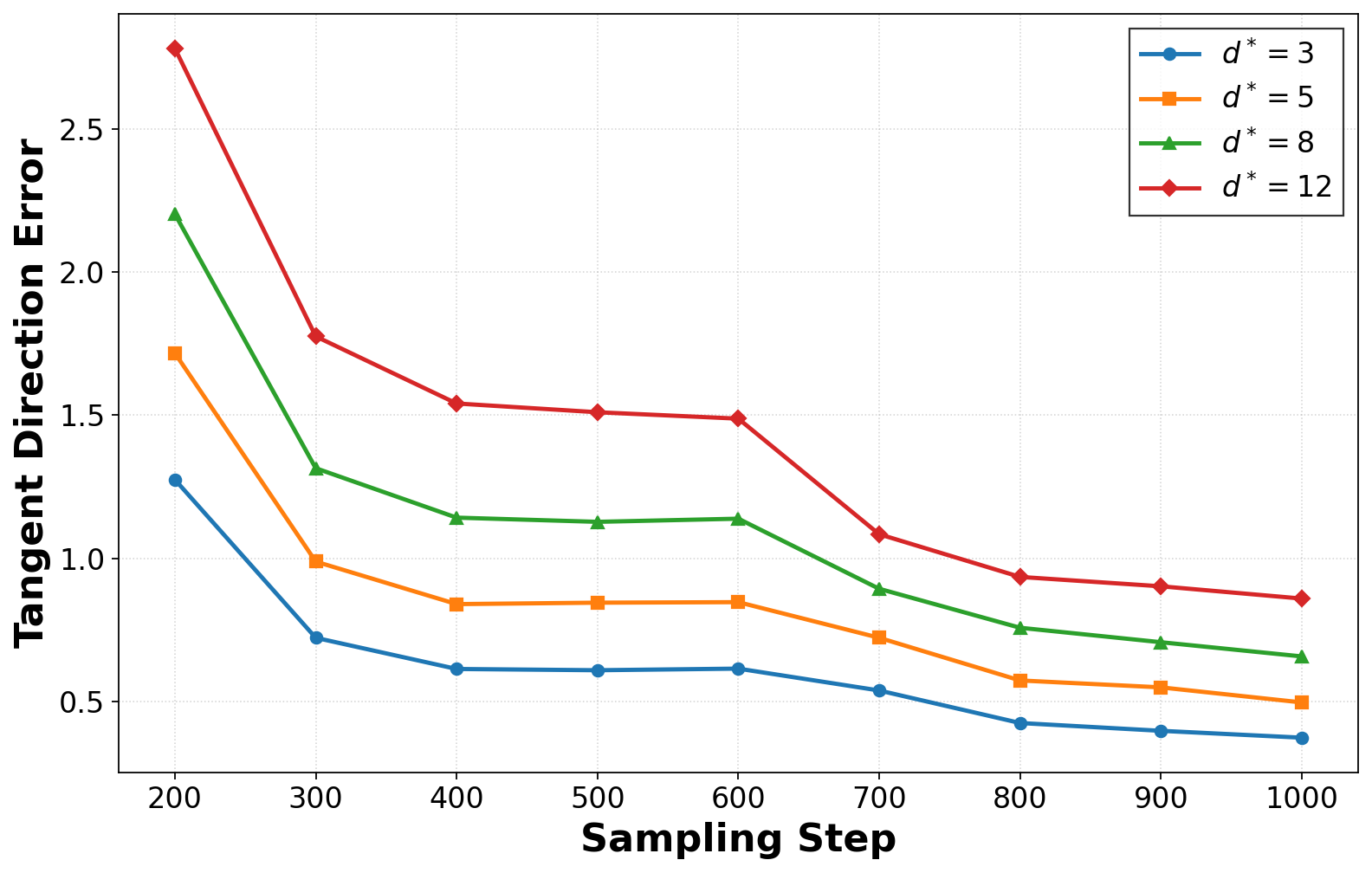}\\[-0.3em]
        {(b) tangent error (full)}
    \end{minipage}\hfill
    \begin{minipage}[b]{0.33\linewidth}
        \centering
        \includegraphics[width=\linewidth,height=3cm]{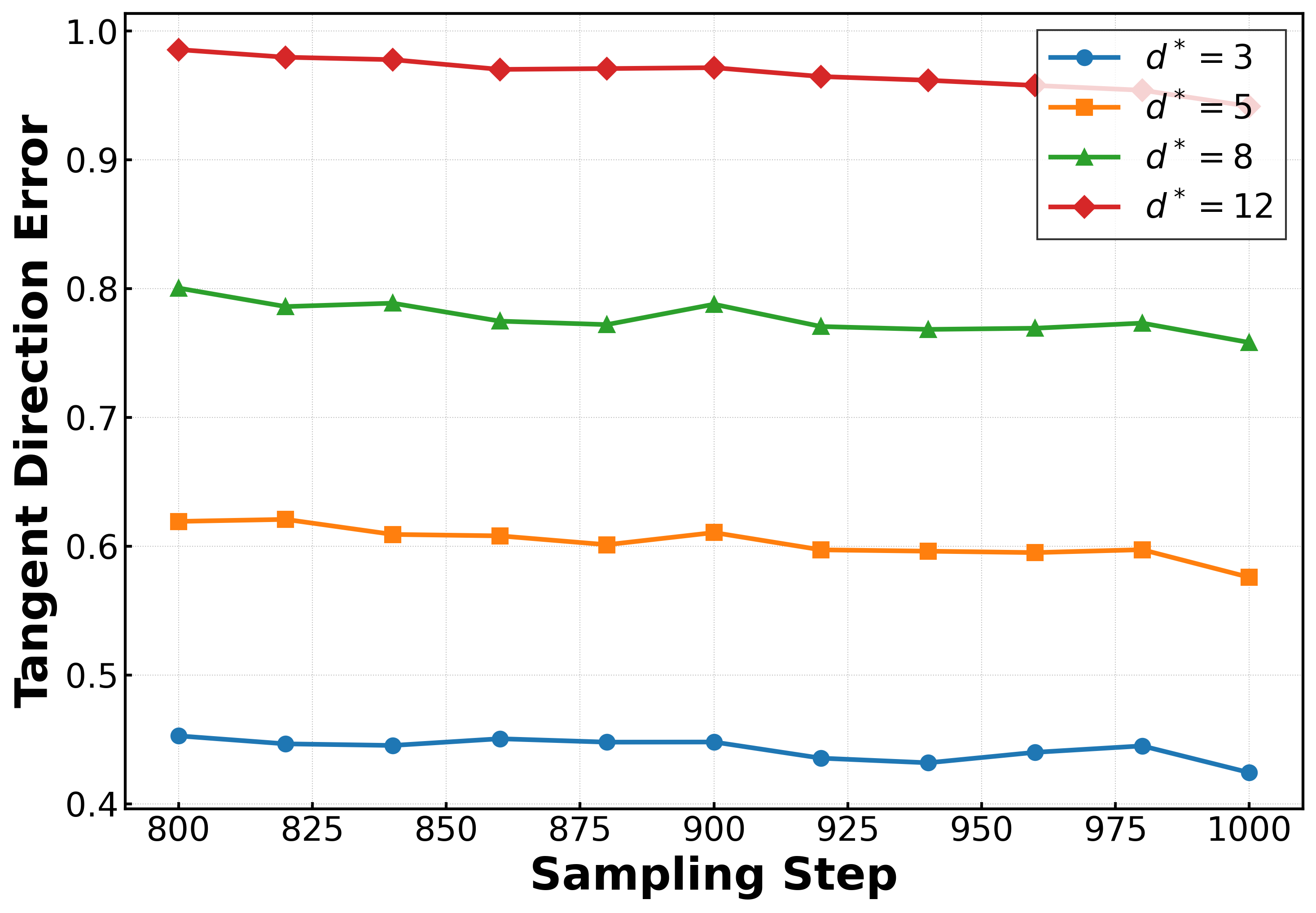}\\[-0.3em]
        {(c) tangent error (zoom)}
    \end{minipage}
    
    \caption{\textbf{Evolution of Ridge Manifold metrics during inference.} (a) Mean distance versus sampling steps for different $d^*$ from Step 200. (b) Tangential error trajectories from step 200. (c) Zoom-in of the tangential error in the final 200 steps.}
    \label{fig:combined_ridge_metrics}
\end{figure}

\section{Conclusions and Limitations}
This work gives a data-dependent geometric perspective on diffusion models’ generation. We show that when memorization does not occur, generated samples are organized by a time-dependent log-density ridge geometry induced by the training data, and that reverse-time inference follows a reach–align–slide mechanism relative to this geometry. The analysis also clarifies how training affects generation: normal training error controls alignment to the ridge, while tangential training error controls spread along it.

Our analysis has several limitations. We do not study errors induced by time discretization, which could alter both normal and tangential geometry, although prior discretization results suggest these effects should remain limited unless the step sizes are very large \citep{lee2022convergence,de2022convergence,chen2022sampling,chen2023improved,benton2024nearly,conforti2023score,wang2024evaluating}. In addition, our explicit training-level decomposition is developed in the RFNN setting, which serves as a tractable nonasymptotic example rather than a full model of modern diffusion architectures. Extending the present framework to discretized samplers and richer learning models would be natural next steps.

\section*{Acknowledgment} MT thanks Hongkai Zhao, Mikhail Belkin and Peter L. Bartlett for inspiring discussions.

\newpage
\bibliography{cite}

@article{george2025denoising,
  title={Denoising Score Matching with Random Features: Insights on Diffusion Models from Precise Learning Curves},
  author={George, Anand Jerry and Veiga, Rodrigo and Macris, Nicolas},
  journal={arXiv preprint arXiv:2502.00336},
  year={2025}
}

@article{bonnaire2025diffusion,
  title={Why Diffusion Models Don't Memorize: The Role of Implicit Dynamical Regularization in Training},
  author={Bonnaire, Tony and Urfin, Rapha{\"e}l and Biroli, Giulio and M{\'e}zard, Marc},
  journal={arXiv preprint arXiv:2505.17638},
  year={2025}
}

@article{song2020score,
  title={Score-based generative modeling through stochastic differential equations},
  author={Song, Yang and Sohl-Dickstein, Jascha and Kingma, Diederik P and Kumar, Abhishek and Ermon, Stefano and Poole, Ben},
  journal={ICLR},
  year={2021}
}

@inproceedings{sohl2015deep,
  title={Deep unsupervised learning using nonequilibrium thermodynamics},
  author={Sohl-Dickstein, Jascha and Weiss, Eric and Maheswaranathan, Niru and Ganguli, Surya},
  booktitle={International conference on machine learning},
  pages={2256--2265},
  year={2015},
  organization={pmlr}
}

@article{niyogi2008finding,
  title={Finding the homology of submanifolds with high confidence from random samples},
  author={Niyogi, Partha and Smale, Stephen and Weinberger, Shmuel},
  journal={Discrete \& Computational Geometry},
  volume={39},
  number={1},
  pages={419--441},
  year={2008},
  publisher={Springer}
}

@Article{genovese2014nonparametric,
  title={Nonparametric ridge estimation},
  author={Genovese, Christopher R and Perone-Pacifico, Marco and Verdinelli, Isabella and Wasserman, Larry},
  year={2014}
}

@article{chen2015asymptotic,
  title={ASYMPTOTIC THEORY FOR DENSITY RIDGES},
  author={Chen, Yen-Chi and Genovese, Christopher R and Wasserman, Larry},
  journal={The Annals of Statistics},
  pages={1896--1928},
  year={2015},
  publisher={JSTOR}
}

@article{leobacher2021existence,
  title={Existence, uniqueness and regularity of the projection onto differentiable manifolds},
  author={Leobacher, Gunther and Steinicke, Alexander},
  journal={Annals of global analysis and geometry},
  volume={60},
  number={3},
  pages={559--587},
  year={2021},
  publisher={Springer}
}

@article{farghly2025diffusion,
  title={Diffusion models and the manifold hypothesis: Log-domain smoothing is geometry adaptive},
  author={Farghly, Tyler and Potaptchik, Peter and Howard, Samuel and Deligiannidis, George and Pidstrigach, Jakiw},
  journal={arXiv preprint arXiv:2510.02305},
  year={2025}
}

@article{vastola2025generalization,
  title={Generalization through variance: how noise shapes inductive biases in diffusion models},
  author={Vastola, John J},
  journal={arXiv preprint arXiv:2504.12532},
  year={2025}
}

@article{bertrand2025closed,
  title={On the Closed-Form of Flow Matching: Generalization Does Not Arise from Target Stochasticity},
  author={Bertrand, Quentin and Gagneux, Anne and Massias, Mathurin and Emonet, R{\'e}mi},
  journal={arXiv preprint arXiv:2506.03719},
  year={2025}
}

@article{kamb2024analytic,
  title={An analytic theory of creativity in convolutional diffusion models},
  author={Kamb, Mason and Ganguli, Surya},
  journal={arXiv preprint arXiv:2412.20292},
  year={2024}
}

@article{shah2025does,
  title={Does generation require memorization? creative diffusion models using ambient diffusion},
  author={Shah, Kulin and Kalavasis, Alkis and Klivans, Adam R and Daras, Giannis},
  journal={arXiv preprint arXiv:2502.21278},
  year={2025}
}

@article{wu2025taking,
  title={Taking a big step: Large learning rates in denoising score matching prevent memorization},
  author={Wu, Yu-Han and Marion, Pierre and Biau, G{\~A}{\v{S}}rard and Boyer, Claire},
  journal={arXiv preprint arXiv:2502.03435},
  year={2025}
}

@article{dhariwal2021diffusion,
  title={Diffusion models beat gans on image synthesis},
  author={Dhariwal, Prafulla and Nichol, Alexander},
  journal={Advances in neural information processing systems},
  volume={34},
  pages={8780--8794},
  year={2021}
}

@article{brooks2024video,
  title={Video generation models as world simulators},
  author={Brooks, Tim and Peebles, Bill and Holmes, Connor and DePue, Will and Guo, Yufei and Jing, Li and Schnurr, David and Taylor, Joe and Luhman, Troy and Luhman, Eric and others},
  journal={OpenAI Blog},
  volume={1},
  number={8},
  pages={1},
  year={2024}
}

@article{kong2020diffwave,
  title={Diffwave: A versatile diffusion model for audio synthesis},
  author={Kong, Zhifeng and Ping, Wei and Huang, Jiaji and Zhao, Kexin and Catanzaro, Bryan},
  journal={ICLR},
  year={2021}
}

@article{ho2020denoising,
  title={Denoising diffusion probabilistic models},
  author={Ho, Jonathan and Jain, Ajay and Abbeel, Pieter},
  journal={Advances in neural information processing systems},
  volume={33},
  pages={6840--6851},
  year={2020}
}

@article{aamari2019estimating,
  title={Estimating the reach of a manifold},
  author={Aamari, Eddie and Kim, Jisu and Chazal, Fr{\'e}d{\'e}ric and Michel, Bertrand and Rinaldo, Alessandro and Wasserman, Larry},
  year={2019}
}

@article{moshksar2024refining,
  title={Refining Concentration for Gaussian Quadratic Chaos},
  author={Moshksar, Kamyar},
  journal={arXiv preprint arXiv:2412.03774},
  year={2024}
}

@article{he2024zeroth,
  title={Zeroth-order sampling methods for non-log-concave distributions: Alleviating metastability by denoising diffusion},
  author={He, Ye and Rojas, Kevin and Tao, Molei},
  journal={Advances in Neural Information Processing Systems},
  volume={37},
  pages={71122--71161},
  year={2024}
}

@inproceedings{carlini2023extracting,
  title={Extracting training data from diffusion models},
  author={Carlini, Nicolas and Hayes, Jamie and Nasr, Milad and Jagielski, Matthew and Sehwag, Vikash and Tramer, Florian and Balle, Borja and Ippolito, Daphne and Wallace, Eric},
  booktitle={32nd USENIX security symposium (USENIX Security 23)},
  pages={5253--5270},
  year={2023}
}

@inproceedings{duan2023diffusion,
  title={Are diffusion models vulnerable to membership inference attacks?},
  author={Duan, Jinhao and Kong, Fei and Wang, Shiqi and Shi, Xiaoshuang and Xu, Kaidi},
  booktitle={International Conference on Machine Learning},
  pages={8717--8730},
  year={2023},
  organization={PMLR}
}

@article{liu2024generative,
  title={Generative AI model privacy: a survey},
  author={Liu, Yihao and Huang, Jinhe and Li, Yanjie and Wang, Dong and Xiao, Bin},
  journal={Artificial Intelligence Review},
  volume={58},
  number={1},
  pages={33},
  year={2024},
  publisher={Springer}
}

@article{li2025scores,
  title={When Scores Learn Geometry: Rate Separations under the Manifold Hypothesis},
  author={Li, Xiang and Shen, Zebang and Hsieh, Ya-Ping and He, Niao},
  journal={arXiv preprint arXiv:2509.24912},
  year={2025}
}

@article{baptista2025memorization,
  title={Memorization and regularization in generative diffusion models},
  author={Baptista, Ricardo and Dasgupta, Agnimitra and Kovachki, Nikola B and Oberai, Assad and Stuart, Andrew M},
  journal={arXiv preprint arXiv:2501.15785},
  year={2025}
}

@article{wang2024evaluating,
  title={Evaluating the design space of diffusion-based generative models},
  author={Wang, Yuqing and He, Ye and Tao, Molei},
  journal={Advances in Neural Information Processing Systems},
  volume={37},
  pages={19307--19352},
  year={2024}
}

@article{somepalli2023understanding,
  title={Understanding and mitigating copying in diffusion models},
  author={Somepalli, Gowthami and Singla, Vasu and Goldblum, Micah and Geiping, Jonas and Goldstein, Tom},
  journal={Advances in Neural Information Processing Systems},
  volume={36},
  pages={47783--47803},
  year={2023}
}

@article{ye2025provable,
  title={Provable separations between memorization and generalization in diffusion models},
  author={Ye, Zeqi and Zhu, Qijie and Tao, Molei and Chen, Minshuo},
  journal={arXiv preprint arXiv:2511.03202},
  year={2025}
}

@article{kadkhodaie2023generalization,
  title={Generalization in diffusion models arises from geometry-adaptive harmonic representations},
  author={Kadkhodaie, Zahra and Guth, Florentin and Simoncelli, Eero P and Mallat, St{\'e}phane},
  journal={arXiv preprint arXiv:2310.02557},
  year={2023}
}

@article{zhang2023emergence,
  title={The emergence of reproducibility and consistency in diffusion models},
  author={Zhang, Huijie and Zhou, Jinfan and Lu, Yifu and Guo, Minzhe and Wang, Peng and Shen, Liyue and Qu, Qing},
  journal={arXiv preprint arXiv:2310.05264},
  year={2023}
}

@article{lee2022convergence,
  title={Convergence for score-based generative modeling with polynomial complexity},
  author={Lee, Holden and Lu, Jianfeng and Tan, Yixin},
  journal={Advances in Neural Information Processing Systems},
  volume={35},
  pages={22870--22882},
  year={2022}
}

@article{de2022convergence,
  title={Convergence of denoising diffusion models under the manifold hypothesis},
  author={De Bortoli, Valentin},
  journal={TMLR},
  year={2022}
}

@article{chen2022sampling,
  title={Sampling is as easy as learning the score: theory for diffusion models with minimal data assumptions},
  author={Chen, Sitan and Chewi, Sinho and Li, Jerry and Li, Yuanzhi and Salim, Adil and Zhang, Anru R},
  journal={ICLR},
  year={2023}
}

@inproceedings{chen2023improved,
  title={Improved analysis of score-based generative modeling: User-friendly bounds under minimal smoothness assumptions},
  author={Chen, Hongrui and Lee, Holden and Lu, Jianfeng},
  booktitle={International Conference on Machine Learning},
  pages={4735--4763},
  year={2023},
  organization={PMLR}
}

@inproceedings{benton2024nearly,
  title={Nearly d-linear convergence bounds for diffusion models via stochastic localization},
  author={Benton, Joe and De Bortoli, Valentin and Doucet, Arnaud and Deligiannidis, George},
  booktitle={The Twelfth International Conference on Learning Representations},
  year={2024}
}

@article{conforti2023score,
  title={Score diffusion models without early stopping: finite Fisher information is all you need},
  author={Conforti, Giovanni and Durmus, Alain and Silveri, Marta Gentiloni},
  journal={arXiv preprint arXiv:2308.12240},
  year={2023}
}

@article{chen2025interpolation,
  title={On the interpolation effect of score smoothing},
  author={Chen, Zhengdao},
  journal={arXiv preprint arXiv:2502.19499},
  year={2025}
}

\newpage
\appendix

\tableofcontents

\section{Related Work}\label{append:related work}

This appendix expands the brief discussion in the introduction and clarifies how our notion of generalization and our ridge-based analysis relate to several nearby directions in the diffusion-model literature. Our goal is not to survey the entire area, but to explain more precisely which question our paper addresses and how it connects to existing viewpoints.

\noindent\textbf{Population-level generalization versus our data-dependent question.}
A substantial line of work \citep[e.g.][]{wang2024evaluating,bertrand2025closed,ye2025provable,bonnaire2025diffusion} studies diffusion-model generalization by comparing the generated distribution to an unknown population distribution and deriving bounds in global discrepancy metrics. This perspective is natural when the goal is population-level recovery or distributional approximation. Our paper addresses a different question. We take the finite training dataset itself as the primary reference object and ask where generated samples go relative to the geometry induced by that dataset. In this sense, our focus is not primarily on global closeness to an unknown population law, but on the geometric organization of non-memorizing generations relative to the observed data. This viewpoint is especially useful when the phenomenon of interest is structured intermediate generation between training samples, since the relevant issue is not only how different two distributions are, but also how generated samples are spatially arranged.

\noindent\textbf{Target-side stochasticity and finite-data target structure.}
One line of work \citep[e.g.,][]{vastola2025generalization,bertrand2025closed} asks whether generalization can already arise from the stochasticity or structure of the finite-data training target itself. From this viewpoint, the learned diffusion model may generate non-memorizing samples not only because of imperfections in training or inference, but also because the empirical target differs from a population-level object in a structured way. Our framework is related to this direction in that it also adopts a fully finite-data viewpoint. However, our emphasis is different: rather than analyzing the stochastic gap between empirical and population targets, we take the empirical dataset as given and study how reverse-time inference organizes samples relative to the geometry induced by that dataset.

\noindent\textbf{Training-induced bias.}
A closely related direction studies the inductive bias created during training, for example through model class, feature learning, optimization dynamics, or finite training time \citep[e.g.,][]{kamb2024analytic,shah2025does,wu2025taking,bonnaire2025diffusion}. This literature explains how the learned score or posterior mean differs from the ideal one and how such differences depend on architecture and optimization. Our contribution is complementary in two ways. First, rather than stopping at aggregate training or test error, we identify \emph{directional} components of training error relative to a data-dependent geometric object. Second, in the RFNN setting, building on random-feature analyses of diffusion training~\citep{george2025denoising,bonnaire2025diffusion}, we show how limited expressivity due to finite width and incomplete optimization translate into different geometric effects during inference: normal components of error control alignment to the ridge, while tangential components control spreading along it. In this sense, our framework explains not only what bias training creates, but also how that bias appears geometrically during sampling.

\noindent\textbf{Inference-time bias accumulation: metric and geometric viewpoints.}
Another broad perspective studies how errors accumulate through inference. One version of this literature \citep[e.g.,][]{lee2022convergence,chen2023improved,benton2024nearly,wang2024evaluating} characterizes the gap between exact and learned reverse processes through divergence-type quantities such as KL or TV. Such results are useful for quantifying distributional discrepancy, but by themselves they say relatively little about the geometry of generated samples or how those samples are organized relative to the training data beyond a global metric.

A second version~\citep{chen2025interpolation,baptista2025memorization,farghly2025diffusion,li2025scores} takes a more geometric viewpoint, often under manifold-type assumptions on the data distribution. Our work is closest in spirit to this line, but differs in several important respects. First, we make the relevant data structure explicit by constructing a time-indexed family of log-density ridge sets directly from the smoothed empirical distribution, rather than assuming an underlying manifold a priori. This differs from \cite{farghly2025diffusion}, which motivates log-density smoothing as a useful analytical lens but does not study simulated inference trajectories or provide a data-dependent ridge-manifold description. It is also complementary to \cite{baptista2025memorization}, which analyzes memorization through the reverse dynamics induced by the empirical-loss minimizer using Voronoi geometry, whereas our focus is the complementary non-memorizing regime, where we identify a time-dependent ridge geometry from the smoothed empirical distribution and analyze how inference evolves relative to it.   Second, we analyze how inference trajectories evolve relative to this geometry through a reach–align–slide mechanism. This goes beyond concentration near a low-dimensional set by capturing part of the tangential organization of generation relative to nearby data, while remaining intentionally partial: the analysis predicts tangential motion toward nearby data-induced centers rather than fully characterizing the generated configuration inside the tangent space. Third, our setting complements \cite{li2025scores}, which studies the transition between sampling uniformly from a low-dimensional manifold and sampling a target distribution supported on that manifold. In contrast, we study full denoising diffusion, where the noise level is time-dependent and vanishes as sampling approaches the data, and we explicitly quantify how the learned score or posterior mean estimator drives inference around the geometric object governing generation. Finally, compared with analytical explanations of interpolation bias in stylized settings or under imposed error ansatz~\citep[e.g.,][]{chen2025interpolation}, our conclusions are derived under verifiable regularity conditions and connect the resulting geometric bias explicitly to both the dataset and the training parameters.

\noindent\textbf{How our framework combines these perspectives.}
Viewed together, our framework combines several of the above viewpoints in a single data-driven analysis. We work directly with the empirical data distribution, avoiding any need to assume an underlying smooth population distribution or a fixed manifold known in advance. We then construct a time-dependent ridge family adapted to the diffusion noise level and prove that reverse-time inference evolves relative to this family through a reach--align--slide mechanism. Finally, in the RFNN+GD setting, we decompose the relevant directional training errors into architecture-driven and optimization-driven terms, making explicit how model class, training procedure, and inference geometry interact. Throughout, the analysis is nonasymptotic: the dataset is finite, and in the RFNN example the data dimension, sample size, network width, and training time are all kept finite rather than taken to infinity.

\noindent\textbf{Summary.}
Relative to nearby work, the main distinction of our paper is therefore not a single isolated technical improvement, but a shift in viewpoint. We study non-memorizing generation in a fully data-dependent setting, identify an explicit time-dependent geometric object from the empirical data, analyze reverse-time inference relative to that object through reach-align-slide, and connect the resulting geometry back to directional components of training error.

\section{Denoising Mean Matching Loss}\label{append:loss} In this section, we introduce the detailed derivation of the posterior mean-matching loss $\MM$ and the denoising posterior mean-matching loss $\DMM$. According to Tweedie's formula that 
\begin{align*}
    \nabla \log p_t(x) =-\frac{x}{h_t}+\frac{\mb{E}[a_tX_0|X_t=x]}{h_t} = -\frac{x}{h_t}+\frac{m(t,x)}{h_t},
\end{align*}
to parametrize the score, it suffices to parametrize the posterior mean $m$. We denote the parametrization by $m_A$. Then the posterior mean matching loss $\MM$ defined in \eqref{eq:mean matching loss} is equivalent to the score matching loss:
\begin{align*}
    \MM & =  \int_\delta^{T} \frac{w(t)}{h_t^2} \mb{E} \big[  \|  m_A(t,X_t) - m(t,X_t) \|^2 \big] \dee t \\
    &= \int_\delta^{T} w(t) \mb{E} \big[  \|  \frac{-X_t+m_A(t,X_t)}{h_t} - \frac{-X_t+m(t,X_t)}{h_t} \|^2 \big] \dee t \\
    & = \int_\delta^{T} w(t) \mb{E} \big[  \|  s_A(t,X_t) - \nabla\log p_t(X_t) \|^2 \big] \dee t.
\end{align*}
However, $\MM$ can't be evaluated directly using data from $X_0$. We can apply the same denoising trick as what's done for score matching loss.
\begin{align*}
    \MM & =  \int_\delta^{T} \frac{w(t)}{h_t^2} \mb{E} \big[  \|  m_A(t,X_t) - m(t,X_t) \|^2 \big] \dee t \\
    & = \int_\delta^{T} \frac{w(t)}{h_t^2} \mb{E} \big[  \|  m_A(t,X_t) -a_t X_0 + a_tX_0 - m(t,X_t) \|^2 \big] \dee t \\
    & = \int_\delta^{T} \frac{w(t)}{h_t^2} \mb{E} \big[  \|  m_A(t,X_t) -a_t X_0  \|^2 \big] \dee t +  \int_\delta^{T} \frac{w(t)}{h_t^2} \mb{E} \big[  \|  m(t,X_t) -a_t X_0  \|^2 \big] \dee t\\
    & \quad -  \bcancel{2 \int_\delta^{T} \frac{w(t)}{h_t^2} \mb{E} \big[  \langle m_A(t,X_t) -a_t X_0, m(t,X_t) -a_t X_0 \rangle  \big] \dee t} \\
    & = \DMM + C,
\end{align*}
where the last term in the third identity is canceled due to the definition of $m$ and tower property. The second term in the third identity is a constant independent to the trained parameter $A$. Therefore, we can train to optimize $\DMM$ directly. The DMM loss evolves two expectations and one integral:
\begin{align*}
    \DMM & =  \int_\delta^{T}  \frac{w(t)}{h_t^{2}}\mb{E}_{X_0}\mb{E}_z\big[\|  -a_t X_0+ m_A(t,a_t X_0 + \sqrt{h_t}z) \|^2 \big] \dee t.
\end{align*}
Under data assumption \ref{assump:data} that $p=\tfrac{1}{n}\sum_{i=1}^n \delta_{x_0^{(i)}}$, the expectation $\mb{E}_{X_0}$ can be exactly evaluated through empirical average over all training data, i.e., 
\begin{equation}\label{eq:DMM_data}
    \DMM  =  \frac{1}{n}\sum_{i=1}^n \int_\delta^{T}  \frac{w(t)}{h_t^{2}}\mb{E}_z\big[ \|  -a_t x_0^{(i)}+ m_A(t,a_t x_0^{(i)} + \sqrt{h_t}z) \|^2 \big] \dee t.
\end{equation}
For the convenience of analysis, we focus on analyzing the loss defined in \eqref{eq:DMM_data}, which corresponds to exact evaluations for $\mb{E}_z$ and integral in $t$.

In practice, the loss in \eqref{eq:DMM_data} is used after further numerical approximations for $\mb{E}_z$ and integral in $t$. The practical DMM loss is given by
\begin{align*}
    \DMM^{m,N} = \frac{1}{nm}\sum_{i,j=1}^{n,m}\sum_{k=1}^N \frac{t_{k}-t_{k-1}}{h_{t_k}^2} \|  -a_{t_k} x_0^{(i)}+ m_A(t_k,a_{t_k} x_0^{(i)} + \sqrt{h_{t_k}}z^{(i,j)}) \|^2
\end{align*}
where $\{z^{(i,j)}\}_{1\le i\le n, 1\le j\le m}$ is a sequence of i.i.d. standard Gaussian vectors in $\mb{R}^d$ and $\delta=t_0<t_1<\cdots<t_N=T$ are the time grids for numerical integration on $[\delta,T]$. 

\section{Data-Independent Properties of the Log-density Ridge Sets} 

In this section, we introduce properties of the log-density ridges that are independent to our data assumptions. We summarize them in the following Proposition.

\begin{proposition}\label{prop:tube well-defined} Under Assumption \ref{assump:ridge set}, for any $\rho_t\in (0,r_t]$ and the tube neighborhood $\mc{T}_t(\rho_t)$ given below
   \begin{align}\label{eq:tube ridge 2}
    \mc{T}_t(\rho_t)\coloneqq \{ x\in \mb{R}^d | \mathrm{dist}(x, \mc{R}_t)\le \rho_t \},
\end{align} 
the nearest-point projection $\Pi_t: \mc{T}_t(\rho_t)\to \mc{R}_t$ is well-defined and we have
\begin{itemize}
    \item [(1)] $\Pi_t$ is $C^1$ on $\mc{T}_t(\rho_t)$;
    \item [(2)] $\forall x\in \mc{T}_t(\rho_t)$, $n_t(x)\coloneqq x-\Pi_t(x)$ is in the normal space at $\Pi_t(x)$, i.e, $n_t(x) \in N_{\Pi_t(x)}(\mc{R}_t)$;
    \item [(3)]  $\sup_{x\in \mc{T}_t(\rho_t)} \| \nabla \Pi_t(x) \|\le \tfrac{1}{1-\rho_t/r_t}$;
    \item [(4)] the motion of $\Pi_t$ is bounded: for any $ z\in \mc{R}_t$, there exists a velocity field $v_t\in N_z(\mc{R}_t)$ s.t. $\sup_{x\in \mc{T}_t(\rho_t)} \| \partial_t \Pi_t\| \le V_t + \frac{\rho_t}{1-\rho_t/r_t} W_t$ where $$V_t=\sup_{z\in \mc{R}_t} \| v_t(z) \|,\quad W_t= \sup_{z\in \mc{R}_t, \| u \|=1} \| P^\parallel(z)(\nabla v_t(z)u) \|\footnote{For any $z\in \mc{R}_t$, we use $P^\parallel(z)$ (or $P^\perp(z)$) to represent the orthogonal projection from $\mb{R}^d$ to $T_z(\mc{R}_t)$ (or $N_z(\mc{R}_t)$).}. $$
\end{itemize}
\end{proposition}

\begin{proof}[Proof of Proposition \ref{prop:tube well-defined}] That $\Pi_t$ is well-defined on $\mc{T}_t(\rho_t)$ directly follows from Assumption \ref{assump:ridge set}. The $C^1$ smoothness of $\Pi_t$ in space follows from \cite[Theorem 2]{leobacher2021existence}. Property $(2)$ follows from the optimality of the nearest-point: $z=\Pi_t(x)$ minimizes $\|x-z'\|^2$ over $z'\in \mc{R}_t$. Therefore, differentiating $z'\mapsto \|x-z'\|^2$ along any tangent direction $u\in T_z(\mc{R}_t)$ yields
\begin{align*}
    D\big( \|x-z'\|^2 \big)[u]|_{z'=z} = -2\langle x-z, u \rangle =-2\langle n_t(x), u \rangle= 0 .
\end{align*}
Therefore, $n_t(x)\in N_z(\mc{R}_t)$.   

\noindent To prove $(3)$, we first apply the explicit expression of $\nabla\Pi_t(x)$ in \cite[Theorem C]{leobacher2021existence}:
\begin{align*}
    \nabla \Pi_t(x) = \big( id_{T_{\Pi_t(x)}(\mc{R}_t)} - \| x- \Pi_t(x) \| S_{t,\Pi_t(x), v} \big)^{-1} P^\parallel(\Pi_t(x)),\quad \theta = \frac{x-\Pi_t(x)}{\|x-\Pi_t(x)\|},
\end{align*}
where $S_{t,\Pi_t(x), \theta}$ is the shape operator in the normal direction $\theta$. According to Lemma \ref{lem:shape operator} and the linearity of $S_{t,\Pi_t(x),\theta}$ in $\theta$, we have 
\begin{align*}
    \|\nabla \Pi_t(x)\| & \le \| \big( id_{T_{\Pi_t(x)}(\mc{R}_t)} - \| x- \Pi_t(x) \| S_{t,\Pi_t(x), v} \big)^{-1} P^\parallel(\Pi_t(x))\| \\
    &\le  \| \big( id_{T_{\Pi_t(x)}(\mc{R}_t)} - S_{t,\Pi_t(x), \| x- \Pi_t(x) \| v} \big)^{-1} \| \le  \frac{1}{1- \| x-\Pi_t(x) \|/r_t } \le  \frac{1}{1- \rho_t/r_t }.
\end{align*}
\noindent Last, to prove $(4)$, we first apply Lemma \ref{lem:normal gauge existence} to show existence of the normal velocity field $v_t$. Then the estimation of $\|\partial_t\Pi_t\|$ follows from the definition of $V_t$ and Lemma \ref{lem:motion speed tangent component}.
\end{proof}

\begin{definition}[Shape Operator]\label{def:shape operator} Let $\text{\rom{2}}_{t,z}$ be the second fundamental form of $\mc{R}_t$ at $z\in \mc{R}_t$. For $\theta\in N_z(\mc{R}_t)$, the shape operator in the direction of $\theta$ is defined as $S_{t,z,\theta}:T_z(\mc{R}_t)\to T_z(\mc{R}_t)$ is then defined as  
\begin{align*}
    \langle S_{t,z,\theta}u,v \rangle = \langle \mathrm{\rom{2}}_{t,z}(u,v),\theta \rangle,\quad \forall u,v\in T_z(\mc{R}_t).
\end{align*}
\end{definition}
\begin{lemma}\label{lem:shape operator} Under Assumption \ref{assump:ridge set}, for any $z\in \mc{R}_t$ and $\theta\in N_z(\mc{R}_t)$, we have
\begin{itemize}
    \item [(1)] the shape operator is bounded: $\| S_{t,z,\theta} \|\le \| \theta\|/r_t$;
    \item [(2)] if $\| \theta \|\le \rho_t<r_t$, the operator $L_{t,z,\theta}\coloneqq id_{T_z(\mc{R}_t)}-S_{t,z,\theta}$ is invertible and
    \begin{align*}
        \| L_{t,z,\theta}^{-1} \| \le \frac{1}{1-\| \theta \|/r_t}.
    \end{align*}
\end{itemize}
\end{lemma}
\begin{proof}[Proof of Lemma \ref{lem:shape operator}] According to \cite[Proposition 6.1]{niyogi2008finding}, under Assumption \ref{assump:ridge set}, $\| \mathrm{\rom{2}}_{t,z} \|\le 1/r_t$ for all $t$. Therefore, $\| S_{t,z,\theta} \|\le \| \theta\|/r_t$ and hence $\|  L_{t,z,\theta}^{-1} \| = \| (id_{T_z(\mc{R}_t)}-S_{t,z,\theta})^{-1} \|\le \tfrac{1}{1-\| \theta \|/r_t}$.    
\end{proof}

\begin{lemma}\label{lem:normal gauge existence} Fix $t\in [\delta,T]$. Let $U\subset\mb{R}^{d^*}$ be open and $\Tilde{\Phi}_t:U\to\mb{R}^d$ be a $C^1$-family of $C^2$ embeddings such that $\Tilde{\Phi}_t(U)\subset \mc{R}_t$ and $\Tilde{\Phi}_t$ is a local parametrization of $\mc{R}_t$ along a given $C^1$-curve $z(t)\in \mc{R}_t$. Assume that $z(t)=\Tilde{\Phi}_t(\Tilde{u}(t))$ for some $C^1$-curve $\Tilde{u}(t)\in U$. Then there exists a $C^1$-family of local diffeomorphisms $\psi_t: U'\to U$ Then there exists a family of diffeomorphism $\psi_t:U'\to U$ such that
\begin{align*}
    \Phi_t \coloneqq \Tilde{\Phi}_t \circ \psi_t : U'\to \mc{R}_t,\quad \partial_t \Phi_t(u) \perp T_{\Phi_t(u)}(\mc{R}_t),\ \forall (t,u)\in [\delta,T]\times U'.
\end{align*}
Consequently, the velocity field $v_t(z)\coloneqq \partial_t \Phi_t(u)$ for $z= \Phi_t(u)$ is a well-defined $C^0$ normal velocity field on $\Phi(U')$, i.e., $v_t(z)\in N_z(\mc{R}_t)$. Moreover, $v_t$ is intrinsic in the sense that it equals to the normal component of $\partial_t\Tilde{\Phi}_t$: $v_t(\Phi_t(u))=P_{N_{\Phi_t(u)}}(\partial_t\Tilde{\Phi}_t(\psi_t(u)))$, and is therefore independent of the tangent reparametrization of the chart. 
\end{lemma}
\begin{proof}[Proof of Lemma \ref{lem:normal gauge existence}] Define $\Tilde{v}_t(u)\coloneqq \partial_t \Tilde{\Phi}_t(u)$. Since $\Tilde{\Phi}_t$ is an embedding, $\nabla\Tilde{\Phi}_t(u): \mb{R}^{d^*}\to T_{\Tilde{\Phi}_t(u)}(\mc{R}_t)$ is a linear isomorphism for each $(t,u)$. Let $P^\parallel(\Tilde{\Phi}_t(u))$ denote orthogonal projection onto $T_{\Tilde{\Phi}_t(u)}(\mc{R}_t)$. Define a time-dependent vector field $a_t$ on $U$ by
\begin{align}\label{eq:def-a}
\nabla\Tilde{\Phi}_t(u) a_t(u) = -P^\parallel(\Tilde{\Phi}_t(u))\widetilde v_t(u)\in T_{\Tilde{\Phi}_t(u)}(\mc R_t).
\end{align}
Since $\nabla\Tilde{\Phi}_t(u)$ is invertible on $T_{\Tilde{\Phi}_t(u)}(\mc{R}_t)$, $a_t(u)$ in \eqref{eq:def-a} is uniquely defined. Furthermore, under our assumptions, $(t,u)\mapsto \nabla\Tilde{\Phi}_t(u)$ and $(t,u)\mapsto P^\parallel(\Tilde{\Phi}_t(u))\Tilde{v}_t(u)$ are continuous in $t$ and smooth in $u$, hence $(t,u)\mapsto a_t(u)$ is also continuous in $t$ and smooth in $u$.   

\noindent Pick an open set $U'\subset U$ that contains $\Tilde{u}(t)$ for all $t\in [\delta,T]$. Consider the ODE
\begin{align}\label{eq:time-dependent ODE}
    \partial_t \psi_t(u) = a_t(\psi_t(u)),\quad \psi_\delta(u)=u\in U'.
\end{align}
Due to the regularity of $a_t(\cdot)$, there exists a unique solution $\psi_t$ for all $t\in [\delta,T]$ and $\psi_t$ is a diffeomorphism for each $t$.

\noindent Now apply chain rule and we get
\begin{align*}
    \partial_t \Phi_t(u) & = \partial_t ( \Tilde{\Phi}_t \circ \psi_t )(u) = \Tilde{v}_t(\psi_t(u)) + \nabla\Tilde{\Phi}_t(\psi_t(u))\partial_t\psi_t(u) \\
    &= \Tilde{v}_t(\psi_t(u)) + \nabla\Tilde{\Phi}_t(\psi_t(u))a_t(\psi_t(u))  = \Tilde{v}_t(\psi_t(u)) - P^\parallel( \Tilde{\Phi}_t(\psi_t(u)))\widetilde v_t(\psi_t(u))\\
    & = P^\perp( \Tilde{\Phi}_t(\psi_t(u)))\widetilde v_t(\psi_t(u)),
\end{align*}
where the second last identity follows from \eqref{eq:def-a} and $P^\perp( \Tilde{\Phi}_t(\psi_t(u)))= I - P^\parallel( \Tilde{\Phi}_t(\psi_t(u)))$ is the normal projection. Therefore, $\partial\Phi_t(u)\perp T_{\Tilde{\Phi}_t(\psi_t(u))}(\mc{R}_t)$. Last, $v_t(z)\coloneqq \partial\Phi_t(u)\in N_{\Tilde{\Phi}_t(\psi_t(u))}(\mc{R}_t)$ with $z=\Phi_t(u)$ is well-defined on $\Phi_t(U')$ and we can check that $v_t(\Phi_t(u))$ is exactly the normal component of $\partial_t\Tilde{\Phi}_t(\psi_t(u))$, hence independent to the tangent reparametrization.
\end{proof}

\begin{lemma}\label{lem:motion speed tangent component} For $t\in [\delta,T]$, let $x\in \mc{T}_t(\rho_t)$ and $z(t)\coloneqq \Pi_t(x)$. Under conditions in Lemma \ref{lem:normal gauge existence} and Assumption \ref{assump:ridge set}, there exists a reparametrization chart $\Phi_t$ in normal gauge and a $C^1$ curve $u(t)\in U\subset \mb{R}^{d^*}$ such that $z(t)=\Phi_t(u(t))$ and
\begin{align*}
    \partial_t \Pi_t(x) = v_t(z(t)) + \tau_t,\quad v_t(z(t))\in N_{z(t)}(\mc{R}_t)\ \text{and }\tau_t\coloneqq \nabla_u\Phi_t(u(t))\partial_t u(t)\in T_{z(t)}(\mc{R}_t).
\end{align*}
Furthermore, $\| \tau_t \|\le \tfrac{\rho_t}{1-\rho_t/r_t} W_t$.   
\end{lemma}
\begin{proof}[Proof of Lemma \ref{lem:motion speed tangent component}] The existence of $\Phi_t$ follows from Lemma \ref{lem:normal gauge existence}. Next, differentiate $z(t)=\Phi_t(u(t))$ and we get
\begin{align*}
    \partial_t z(t) = \partial_t \Phi_t(u(t)) + \nabla_u\Phi_t(u(t))\partial_t u(t) = v_t(z) + \tau_t.
\end{align*}
To prove the bound for $\| \tau_t\|$, we consider the local tangent frame $\{E_i(t)\}_{i=1}^{d^*}$ along $z(t)$ induced by the chart $u\mapsto \Phi_t(u)$. Since $n_t(x)=x-z(t)\perp T_{z(t)}(\mc{R}_t)$, we have $\langle n_t(x), E_i(t) \rangle=0$ for all $1\le i\le d^*$. Differentiate wrt $t$ on both sides and we get
\begin{align*}
    0 = -\langle \partial_t z(t), E_i(t) \rangle + \langle n_t(x), \partial_t E_i(t)\rangle.
\end{align*}
Decompose $\partial_t z(t)=v_t(z(t))+\tau_t$ and use the fact that $v_t(z(t))\perp T_{z(t)}(\mc{R}_t)$, and we get
\begin{align}\label{eq:inter}
    \langle \tau_t, E_i(t) \rangle = \langle n_t(x), \partial_t E_i(t) \rangle .
\end{align}
Using the local frame we can compute $\partial_t E_i(t)$ as follows. Since $E_i(t)=\partial_{u_i}\Phi_t(u(t))$, we have
\begin{align*}
    \partial_t E_i(t) &= \partial_{u_i}\partial_t \Phi_t(u(t))+\sum_{j=1}^{d^*} \partial^2_{u_i u_j} \Phi_t(u(t))\partial_j u(t) \\
    & = \partial_{u_i} (v_t\circ \Phi_t)(u(t)) + \sum_{j=1}^{d^*} P^\perp(z)(\partial^2_{u_i u_j}\Phi_t(u(t)))\partial_t u_j(t) + \sum_{j=1}^{d^*} P^\parallel(z)(\partial^2_{u_i u_j}\Phi_t(u(t)))\partial_t u_j(t)\\
    & = \nabla v_t(z) E_i(t) + \sum_{j=1}^{d^*} \text{\rom{2}}_{t,z}(E_i(t),E_j(t))\partial_t u_j(t) + \sum_{j=1}^{d^*} P^\parallel(z)(\partial^2_{u_i u_j}\Phi_t(u(t)))\partial_t u_j(t).
\end{align*}
Since $n_t(x)\in N_{z(t)}(\mc{R}_t)$, we have
\begin{align*}
    \langle n_t(x), \partial_t E_i(t) \rangle & = \langle n_t(x), \nabla v_t(z(t)) E_i(t) \rangle + \langle n_t(x),  \text{\rom{2}}_{t,z}(E_i(t),\sum_j \partial_t u_j(t)E_j(t)) \rangle \\
    & = \langle n_t(x), \nabla v_t(z(t)) E_i(t) \rangle +\langle S_{t,z,n_t(x)} \tau_t, E_i(t) \rangle,
\end{align*}
where the last identity follows from the definition of shape operator and $\tau_t=\sum_j \partial_t u_j(t)E_j(t)$. Plug the above equation into \eqref{eq:inter} and we get that restricted to the tangent space $\mc{T}_{z(t)}(\mc{R}_t)$,
\begin{align*}
    (I-S_{t,z,n_t(x)} )\tau_t = P^\parallel(z(t))(\nabla v_t(z(t)))^\intercal n_t(x).
\end{align*}
Therefore, according to Lemma \ref{lem:shape operator}, the definition of $W_t$ and the fact that $x\in \mc{T}_t(\rho_t)$,
\begin{align*}
    \| \tau_t \| \le \| L_{t,z,n_t(x)}^{-1} \| \| \nabla v_t(z(t)) \| \| n_t(x) \| \le \frac{\rho_t}{1-\rho_t/r_t} W_t.
\end{align*}
\end{proof}

\section{Data dependent Ridge Motion Estimations}\label{append:data dependent motion estimation}

In this section, we provided some data-dependent estimations for quantities related to dynamical properties of the log-density ridges.

\begin{proposition}\label{prop:constant estimation} Under Assumption \ref{assump:data}, the log-density ridge sets satisfy Assumption \ref{assump:ridge set}. Furthermore, when $t\to \delta^+\ll 1$, we have the following order estimations:
\begin{align*}
    r_t = \Omega(\beta_t h_t^{3} \theta_t^{-1} R^{-3}),\quad V_t = \mc{O}( {\beta_t^{-1} h_t^{-1}} R  ),\quad W_t =\mc{O} ( \beta_t^{-1}h_t^{-1} R r_t^{-1} ),
\end{align*}
where $\theta_t$ is arbitrarily with order $\theta_t=\exp(- o(h_{t}^{-1}) )$.
\end{proposition}
\begin{remark}[Order estimation of ridge motion]\label{rem:ridge motion bound} Combining Propositions \ref{prop:tube well-defined} and \ref{prop:constant estimation}, picking $\rho_t=\Theta(r_t)$ and $\beta_t=\Theta({1}/{h_t})$, we proved Proposition \ref{prop:projection-regularity}-(3): $\sup_{x\in \mc{T}_t(\rho_t)}\| \partial_t\Pi_t(x)\| = \mc{O}\left( R\right)$. 
\end{remark}
To study the relation between the data Assumption \ref{assump:data} and properties of the data-dependent manifold $\mc{R}_t$. The key is to estimate the derivatives of the score $\nabla\log p_{T-t}(\cdot)$ in different regions dominated by centers $\{m_{T-t}^{(i)}\coloneqq a_{T-t}x_0^{(i)}\}_{i=1}^n$. For each $i\in [n]$ and $|\theta_s|<1$ for any $\zeta>0$ when $s\to 0^+$, define the center-$i$ dominate region
\begin{align}\label{eq:data dominate region}
    \mc{B}_{T-t}^{(i)}(\theta_{T-t}) \coloneqq \{ x\in \mb{R}^d \mid \Softmax(-\frac{\| x-a_{T-t}x_0 \|^2}{2h_{T-t}})_i \ge 1-\theta_{T-t} \}.
\end{align}
Next, we introduce a sufficient condition for $x\in \mc{B}_{T-t}^{(i)}(\theta_{T-t})$.
\begin{lemma}\label{lem:sufficient closeness} For any $x\in \mb{R}^d$ and $\theta_{T-t}\in (0,\tfrac{1}{2})$, if  
\begin{align}\label{eq:sufficient closeness}
    \| x- m_{T-t}^{(j)} \|^2 - \| x- m_{T-t}^{(i)} \|^2 \ge 2h_{T-t}\log\big( \frac{(1-\theta_{T-t})(n-1)}{\theta_{T-t}} \big),\quad \forall j\neq i,
\end{align}
then $x\in \mc{B}_{T-t}^{(i)}(\theta_{T-t})$.
\end{lemma}
\begin{remark}\label{rem:sufficient closeness} As $t\to T^-$, if $h_{T-t}=o(1)$ and $\theta_{T-t}=\exp(- o(h_{T-t}^{-1}) )$, then RHS of \eqref{eq:sufficient closeness} is of order $o(1)$. Therefore, $Y_t$ (or $\Tilde{Y}_t$) satisfies the \eqref{eq:sufficient closeness} with probability $1$ as $t\to T^-$. As a consequence of Lemma \ref{lem:sufficient closeness}, $Y_t\in \mc{B}_{T-t}^{(i)}(\theta_{T-t})$ (or $\Tilde{Y}_t\in \mc{B}_{T-t}^{(i)}(\theta_{T-t})$) with probability $1$ as $t\to T^-$.
\end{remark}
\begin{proof}[Proof of Lemma \ref{lem:sufficient closeness}] Under \eqref{eq:sufficient closeness}, we have that for all $j\neq i$,
\begin{align*}
    \frac{\Softmax(-\frac{\| x-m_{T-t} \|^2}{2h_{T-t}})_j}{\Softmax(-\frac{\| x-m_{T-t} \|^2}{2h_{T-t}})_i} & =\exp\big( -\frac{\| x-m_{T-t}^{(j)}\|^2-\| x-m_{T-t}^{(i)} \|^2}{2h_{T-t}} \big) \le \frac{\theta_{T-t}}{(1-\theta_{T-t})(n-1)}.
\end{align*}
Therefore,
\begin{align*}
    &\quad 1-\Softmax(-\frac{\| x-m_{T-t} \|^2}{2h_{T-t}})_i\\
    & = \sum_{j\neq i}\Softmax(-\frac{\| x-m_{T-t} \|^2}{2h_{T-t}})_j  \le \frac{\theta_{T-t}}{(1-\theta_{T-t})} \Softmax(-\frac{\| x-m_{T-t} \|^2}{2h_{T-t}})_i.
\end{align*}
The statement follows from the definition of $\mc{B}_{T-t}^{(i)}(\eta_{T-t})$.
\end{proof}
Now we provide estimates of derivatives of the score on the center-dominate regions.
\begin{lemma}\label{lem:important estimate} For any $t\in [\delta,T]$, we have
\begin{align*}
    \sup_{x\in \cup_{i=1}^n \mc{B}_{t}^{(i)}(\theta_{t})} \| \nabla m(t,x) \|\le \frac{20a_t^2 \theta_t R^2}{h_t},\quad  \sup_{x\in \cup_{i=1}^n \mc{B}_{t}^{(i)}(\theta_{t})}\| \nabla^3 \log p_t(x) \|\le \frac{80 a_t^3\theta_t R^3}{h_t^{3}},
\end{align*} 
where $\mc{B}_{t}^{(i)}(\theta_{t})$ is defined in \eqref{eq:data dominate region} with any $\theta_t=\exp(- o(h_{t}^{-1}) )$.
\end{lemma}
\begin{remark}[Choice of $\beta_t$]\label{rem:choice of beta} Lemma \ref{lem:important estimate} also validates the choice of ridge threshold $\beta_t=\Theta(1/h_t)$ when $t\to 0^+$. According to Lemma \ref{lem:explicit formula}, $\nabla^2\log p_t(x) = -\tfrac{1}{h_t}I_d + \tfrac{1}{h_t^2}\Sigma(t,x)$. According to estimations in Lemma \ref{lem:important estimate}, each eigenvalue $\lambda_j(t,x)$ of $\nabla^2\log p_t(x)$ satisfies
\begin{align*}
    \lambda_j(t,x) \le -\frac{1}{h_t} + \frac{20a_t^2 \theta_t R^2 }{h_t^2} \le -\frac{c}{h_t}, \quad \frac{1}{2}<c<1,
\end{align*}
under some choice of $\theta_t=\exp(- o(h_{t}^{-1}) )$. Therefore, as $t\to 0^+$, the choice of $\beta_t=c/h_t$ makes the second condition in the log-density ridge definition automatically satisfied.    
\end{remark}
\begin{proof}[Proof of Lemma \ref{lem:important estimate}] According to Lemma \ref{lem:explicit formula}, for all $x\in \mc{B}_{t}^{(i)}(\theta_{t})$
\begin{align*}
    \| \nabla m(t,x) \| = \frac{1}{h_t}\| \Sigma(t,x) \| \le \frac{1}{h_t} \sum_{j=1}^n \Softmax(-\frac{\| x-a_{t}x_0 \|^2}{2h_{t}})_j \| a_t x_0^{(j)} - m(t,x) \|^2.
\end{align*}
Notice that for $j= i$,
\begin{align*}
    \| a_t x_0^{(i)} - m(t,x) \| & = \| a_t x_0^{(i)} - \sum_{j'=1}^n  \Softmax(-\frac{\| x-a_{t}x_0 \|^2}{2h_{t}})_{j'} a_t x_0^{(j')}  \| \\
     & \le \sum_{j'\neq i} \Softmax(-\frac{\| x-a_{t}x_0 \|^2}{2h_{t}})_{j'}  \| a_t x_0^{(j')} - a_t x_0^{(i)}\| \le 2a_t \theta_t R.
\end{align*}
For $j\neq i$, we have
\begin{align*}
    \| a_t x_0^{(j)} - m(t,x) \| & \le  \| a_t x_0^{(i)} - m(t,x) \| + \| a_t x_0^{(i)} - a_t x_0^{(j)} \| \le 2a_t (\theta_t+1) R\le 4a_tR.
\end{align*}
Combining the above estimations and we get
\begin{align*}
    &\quad \| \nabla m(t,x)\| \\
    & \le \frac{1}{h_t} \Softmax(-\frac{\| x-a_{t}x_0 \|^2}{2h_{t}})_{i}(2a_t\theta_t R)^2 + \frac{1}{h_t} (\sum_{j\neq i} \Softmax(-\frac{\| x-a_{t}x_0 \|^2}{2h_{t}})_{j}  ) (4a_t R)^2 \\&\le \frac{20a_t^2 \theta_t R^2}{h_t}.
\end{align*}
Following the same approach, we can bound $\|\nabla^3\log p_t (x)\|$ for $x\in \mc{B}_{t}^{(i)}(\theta_{t})$:
\begin{align*}
    &\quad\| \nabla^3\log p_t(x) \|\\
    & \le \frac{1}{h_t^3} \sum_{j=1}^n \Softmax(-\frac{\| x-a_{t}x_0 \|^2}{2h_{t}})_j \| a_t x_0^{(j)} - m(t,x) \|^3\\
    & \le \frac{1}{h_t^3} (\Softmax(-\frac{\| x-a_{t}x_0 \|^2}{2h_{t}})_{i}) (2a_t\theta_t R)^3 + \frac{1}{h_t^3} (\sum_{j\neq i} \Softmax(-\frac{\| x-a_{t}x_0 \|^2}{2h_{t}})_{j}  ) (4a_t R)^3 \\
    &\le \frac{80a_t^3 \theta_t R^3}{h_t^3}.
\end{align*}
\end{proof}
Now we introduce the proof of Proposition \ref{prop:constant estimation}.
\begin{proof}[Proof of Proposition \ref{prop:constant estimation}] 
According to Lemma \ref{lem:second fundamental form bound}, we have
\begin{align*}
   \sup_{z\in \mc{R}_t} \| \mathrm{\rom{2}}_{t,z} \| \le \frac{1}{\beta_t} \sup_{z\in \mc{R}_t} \| \nabla^3 \log p_t(z) \|.
\end{align*}
Since reach is at least the reciprocal of maximal curvature, we derive that
\begin{align*}
    r_t \gtrsim \beta_t (\sup_{z\in \mc{R}_t} \| \nabla^3 \log p_t(z) \|)^{-1}.
\end{align*}
As $t\to 0^+$, according to Lemma \ref{lem:sufficient closeness} and definition of $\mc{B}_{t}^{(i)}(\theta_{t})$, $x\in \cup_{i=1}^n \mc{B}_{t}^{(i)}(\theta_{t})$ with probability 1 and $ \overline{\cup_{i=1}^n \mc{B}_{t}^{(i)}(\theta_{t}))}\to \mb{R}^d$. Therefore, as $t\to 0^+$, we have
\begin{align*}
    r_t \gtrsim \beta_t (\sup_{z\in \mc{R}_t} \| \nabla^3 \log p_t(z) \|)^{-1} \gtrsim \beta_t (\sup_{x\in \overline{\cup_{i=1}^n \mc{B}_{t}^{(i)}(\theta_{t})}} \| \nabla^3 \log p_t(x) \|)^{-1} = \Omega( \frac{\beta_t h_t^{3}}{\theta_t R^3} ),
\end{align*}
for any $\theta_t=\exp(- o(h_{t}^{-1}) )$ and the last estimation follows from Lemma \ref{lem:important estimate}.

\noindent To bound $V_t$, recall that $V_t=\sup_{z\in \mc{R}_t} \| v_t(z) \|$ with $v_t(z)$ being the velocity field induced by the normal gauge in Lemma \ref{lem:normal gauge existence}. 

\noindent According to Lemma \ref{lem:normal velocity field},
\begin{align*}
    v_t(z) &= -(E_t(z)\nabla^2\log p_t(z) E_t(z))^{-1} E_t(z)^\intercal \partial_t\nabla\log p_t(z) \\
    &= -(E_t(z)\nabla^2\log p_t(z) E_t(z))^{-1} \frac{1}{h_t}E_t(z)^\intercal \partial_t m(t,z),
\end{align*}
where the last identity follows from Lemma \ref{lem:ridge motion expression}. Therefore,
\begin{align*}
    V_t & \le \sup_{z\in \mc{R}_t } \| v_t(z) \|  \\
    &\le  \frac{1}{h_t} \sup_{z\in \mc{R}_t} \| (E_t(z)\nabla^2\log p_t(z) E_t(z))^{-1} \| \| E_t(z)^\intercal \partial_t m(t,z) \| \\
    &\le \frac{1}{\beta_t h_t} \sup_{z\in \mc{R}_t} \| E_t(z)^\intercal \partial_t m(t,z) \|,
\end{align*}
where the last inequality follows form the estimate in the proof of Lemma \ref{lem:normal velocity field}. Hence according to Lemma \ref{lem:key estimation}, $V_t = \mc{O}( \tfrac{R}{h_t\beta_t} )$.

\noindent Last, to bound $W_t$, recall that $W_t=\sup_{z\in \mc{R}_t, \| u \|=1} \| P^\parallel(z)(\nabla v_t(z)u) \|$. Notice that $v_t(z)\in N_z(\mc{R}_t)$. Hence $P^\perp(z)v_t(z)=v_t(z)$. Differentiate both sides of the equation at $z$ along direction $u$, we have
\begin{align*}
     &(\nabla_u P^\perp(z)) v_t(z) + P^\perp(z) \nabla_u v_t(z)  = \nabla_u v_t(z) \\
     \quad \implies  & (\nabla_u P^\perp(z)) v_t(z) = (I- P^\perp(z)) \nabla_u v_t(z) = P^\parallel(z) \nabla_u v_t(z).
\end{align*}
Therefore, we immediately obtain
\begin{align*}
    W_t \le \big( \sup_{z\in\mc{R}_t, \| u \|=1} \| \nabla_u P^\perp(z) \|\big) V_t
\end{align*}
Next, we bound $\|\nabla_u P^\parallel(z)\|$. Notice that $P^\parallel(z)+P^\perp(z)=I_d$, hence $\|\nabla_u P^\parallel(z)\|=\|\nabla_u P^\perp(z)\|$. Now we start from $P^\perp(z)^2=P^\perp(z)$, taking the directional derivative on both side, left multiplying $P^\parallel(z)$ and right multiplying $P^\parallel(z)$,
\begin{align*}
    P^\parallel(z)(\nabla_u P^\parallel(z))P^\parallel(z) + P^\parallel(z)(\nabla_u P^\parallel(z))P^\parallel(z) = P^\parallel(z)\nabla_u P^\parallel(z) P^\parallel(z),
\end{align*}
Hence $P^\parallel(z)\nabla_u P^\parallel(z) P^\parallel(z) =0$. Similarly, we start from $P^\perp(z)^2=P^\perp(z)$, taking the directional derivative on both side, left multiplying $P^\perp(z)$ and right multiplying $P^\perp(z)$, and get
\begin{align*}
    P^\perp(z) (\nabla_u P^\parallel(z))P^\parallel(z)P^\perp(z) + P^\perp(z)P^\parallel(z) (\nabla_u P^\parallel(z))P^\perp(z) = P^\perp(z)(\nabla_u P^\parallel(z))P^\perp(z).
\end{align*}
Since $P^\parallel(z)P^\perp(z)=P^\perp(z)P^\parallel(z)$, we can express $\nabla_u P^\parallel(z)$ on $T_z(\mc{R}_t)\oplus N_z(\mc{R}_t)$ as
\begin{align*}
    \nabla_u P^\parallel(z) = \begin{pmatrix}
        0 & B^\intercal \\
        B & 0
    \end{pmatrix},
\end{align*}
with $B=P^\perp(z)(\nabla_u P^\parallel(z))P^\parallel(z): T_z(\mc{R}_t)\to N_z(\mc{R}_t)$. Therefore,
\begin{align*}
    \|\nabla_u P^\perp(z)\| = \| \nabla_u P^\parallel(z) \| = \| B \| = \| \mathrm{\rom{2}}_{t,z}(u,\cdot) \|_{T_z(\mc{R}_t)\to N_z(\mc{R}_t)}\lesssim \frac{1}{r_t}.
\end{align*}
where the last inequality follows from \cite[Theorem 3.4]{aamari2019estimating}.
Therefore, we have
\begin{align*}
    W_t\lesssim V_t/r_t.
\end{align*}
Hence we proved Proposition \ref{prop:constant estimation}-(3).
\end{proof}

\begin{lemma}[Explicit formulas]\label{lem:explicit formula} Under Assumption \ref{assump:data}, we have that for all $t\in [\delta,T]$,
    \begin{align*}
        \nabla \log p_t(x) &= -\frac{x}{h_t}+\frac{m(t,x)}{h_t}, \quad \nabla^2 \log p_t(x) = -\frac{I_d}{h_t}+\frac{\Sigma(t,x)}{h_t^2},\\
        \nabla^3\log p_t(x) &= \frac{\mb{E}[(U(t,x)-m(t,x))^{\otimes 3}]}{h_t^3},
    \end{align*}
    where $U(t,x)\in \mb{R}^d$ is a random vector taking values $\{a_tx_0^{(i)}\}_{i=1}^n$ with probabilities $\{\Softmax( -\frac{\| x-a_t x_0 \|^2}{2h_t} )_i\}_{i=1}^n$ and 
    \begin{align*}
        m(t,x) &= \mb{E}[U(t,x)], \quad
        \Sigma(t,x)  =\mathrm{Cov}(U(t,x)),
    \end{align*}
    with $x_0=(x_0^{(1)}, x_0^{(2)},\cdots, x_0^{(n)})$. Furthermore, $\nabla m(t,x)= \Sigma(t,x)/h_t$.
\end{lemma}
\begin{lemma}[Second Fundamental Form Bound]\label{lem:second fundamental form bound} For any $t\in [\delta,T]$ and $z\in \mc{R}_t$, we have
\begin{align*}
    \| \mathrm{\rom{2}}_{t,z} \| \le \frac{1}{\beta_t} \| \nabla^3 \log p_t(z) \|.
\end{align*}   
\end{lemma}
\begin{proof}[Proof of Lemma \ref{lem:second fundamental form bound}] For any $z\in \mc{R}_t$, we consider $T\in \mb{R}^{d\times d^*}$ and $N\in \mb{R}^{d\times (d-d^*)}$ to be the orthonormal basis spanning the spaces $T_z(\mc{R}_t)$ and $N_z(\mc{R}_t)$ respectively. We use coordinates $(u,v)\in \mb{R}^{d^*}\times \mb{R}^{d-d^*}$ via $x(u,v)=z+Tu+Nv$. We can define the function 
\begin{align*}
    F(u,v) \coloneqq N^\intercal \nabla \log p_t( x(u,v) ) \in \mb{R}^{d-d^*}.
\end{align*}
Since $z\in \mb{R}_t$, the definition of ridge set $\mc{R}_t$ includes the normal component of the score $\nabla\log p_t(x)$ is zero, i.e., $F(0,0)=0$. The partial derivatives along the normal direction is
\begin{align*}
    \partial_v F(u,v) = N^\intercal \nabla^2\log p_t(x(u,v)) N,\quad \partial_v F(0,0) = N^\intercal \nabla^2\log p_t(z) N.
\end{align*}
Since $\nabla^2\log p_t$ is invertible, we have
\begin{align*}
    \| (\partial_v F(0,0) )^{-1}\| = \| (N^\intercal \nabla^2\log p_t(z) N)^{-1} \| \le \frac{1}{\beta_t},
\end{align*}
where the inequality follows from the second condition in the definition of $\mc{R}_t$. Therefore, apply the implicit function theorem and we get: there exists a neighborhood $0\in U\subset \mb{R}^{d^*}$ and a $C^2$ map $\phi: U\to \mb{R}^{d-d^*}$ with $\phi(0)=0$ such that the local solution set to $F(u,v)=0$ is the set $\{(u,v) \mid v=\phi(u)\}$. Therefore, a local parametrization of the manifold is
\begin{align*}
    \gamma(u) = z + Tu + N\phi(u).
\end{align*}
And differentiating $F(u,\phi(u))=0$ at $u=0$ implies $\nabla \phi(0)=0$. Now differentiate $F(u,\phi(u))=0$ twice and evaluate at $u=0$ and use the fact that $\nabla\phi(0)=0$, we get
\begin{align*}
    \partial_{uu}F(0,0) + \partial_v F(0,0) \nabla^2\phi(0) =0 .
\end{align*}
Since $F(u,v)=N^\intercal \nabla \log p_t( x(u,v) )$, we have that for all $\xi\in \mb{R}^{d^*}$,
\begin{align*}
    \partial_{uu}F(0,0)[\xi,\xi] = N^\intercal (\nabla^3\log p_t(z))[T\xi, T\xi] \implies \|\partial_{uu} F(0,0) \|\le \| \nabla^3\log p_t(z) \|.
\end{align*}
Therefore, we have
\begin{align*}
    \| \nabla^2\phi(0) \| &= \| -(\partial_v F(0,0))^{-1} \partial_{uu}F(0,0) \| \le \| -(\partial_v F(0,0))^{-1} \| \|\partial_{uu}F(0,0) \|  \le \frac{1}{\beta_t}  \| \nabla^3\log p_t(z) \|.
\end{align*}
Last, along the local parametrization $\gamma(u)=z+Tu+N\phi(u)$, we have
\begin{align*}
    \partial_{ij}^2 \gamma(0) = N \partial_{ij}^2 \phi(0) .
\end{align*}
Therefore, apply the definition of $\mathrm{\rom{2}}$ and we get
\begin{align*}
    \mathrm{\rom{2}}_{t,z}(\nabla\gamma(0)\xi, \nabla\gamma(0)\theta) = P^\perp(z) \nabla^2\gamma(0)[\xi,\theta] = N \nabla^2\phi (0)[\xi,\theta].
\end{align*}
Since $N$ is isotropic (orthonormal), we prove the Lemma.
\end{proof}
\begin{lemma}[Normal velocity field expression]\label{lem:normal velocity field} For all fixed $t\in [\delta,T]$ and $z\in \mc{R}_t$, let $E_t\in \mb{R}^{d\times(d-d^*)}$ be the normal eigen-matrix in the definition of $\mc{R}_t$. The normal velocity field in Lemma \ref{lem:normal gauge existence} can be expressed as \begin{align}
    v_t(z) = -(E_t(z)\nabla^2\log p_t(z) E_t(z))^{-1} E_t(z) \partial_t\nabla\log p_t(z) \in N_z(\mc{R}_t).
\end{align} 
\end{lemma}
\begin{proof}[Proof of Lemma \ref{lem:normal velocity field}] 

\noindent Fix $t_0\in (\delta,T)$ and $z_0\in \mc{R}_{t_0}$, consider the normal gauge parametrization near $(t_0,z_0)$. we define $F(t,x)\coloneqq E_{t_0}(z_0)^\intercal \nabla \log p_t(x)\in \mb{R}^{d-d^*}$ with $E_{t_0}$ being the orthonormal basis of $N_{z_0}(\mc{R}_{t_0})$. Since $z_0\in \mc{R}_{t_0}$, we have $F(t_0,z_0)=0$. Taking partial derivatives of $F$ and evaluating at $(t_0,z_0)$, we get
\begin{align*}
    \partial_t F(t_0,z_0)  & = E_{t_0}(z_0)^\intercal \partial_t \nabla\log p_{t}(z_0)|_{t=t_0}, \\
    \nabla F(t_0,z_0)  & = E_{t_0}(z_0)^\intercal \nabla^2\log p_{t_0}(z_0).
\end{align*}
Pick a normal gauge curve $t\mapsto z(t)\in \mc{R}_t$ with $z(t_0)=z_0$ and $\partial_t z(t)|_{t=t_0}\in N_{z_0}(\mc{R}_{t_0})$. We have $F(z, z(t))=0$. Taking derivative wrt. $t$ to both sides of the equation, we get
\begin{align*}
    0 & = \partial_t F(t_0,z_0) + \nabla F(t_0,z_0) \partial_t z \\
    & = E_{t_0}(z_0)^\intercal \partial_t \nabla\log p_{t}(z_0)|_{t=t_0} + E_{t_0}(z_0)^\intercal \nabla^2\log p_{t_0}(z_0) \partial_t z(t)|_{t=t_0} \\
    & = E_{t_0}(z_0)^\intercal \partial_t \nabla\log p_{t}(z_0)|_{t=t_0} + E_{t_0}(z_0)^\intercal \nabla^2\log p_{t_0}(z_0) E_{t_0}(z_0) \theta_0,
\end{align*}
In the last identity, due to the fact that $\partial_t z(t)|_{t=t_0}\in N_{z_0}(\mc{R}_{t_0})$, we write $\partial_t z(t)|_{t=t_0}=E_{t_0}(z_0) \theta_0$ for some $\theta_0\in \mb{R}^{d-d^*}$. According to the definition of $\mc{R}_{t_0}$, all eigenvalues of $E_{t_0}(z_0)^\intercal \nabla^2\log p_{t_0}(z_0) E_{t_0}(z_0)\in \mb{R}^{(d-d^*)\times (d-d^*)}$ are less than $-\beta_{t_0}$. Therefore, $E_{t_0}(z_0)^\intercal \nabla^2\log p_{t_0}(z_0) E_{t_0}(z_0)$ is invertible and 
\begin{align*}
    \| (E_{t_0}(z_0)^\intercal \nabla^2\log p_{t_0}(z_0) E_{t_0}(z_0))^{-1} \|\le 1/\beta_{t_0}.
\end{align*}
Hence the normal velocity field $v_{t_0}(z_0)$ is unique and 
\begin{align*}
    v_{t_0}(z_0) & = E_{t_0}(z_0) \theta_0 \\
    & = E_{t_0}(z_0) (E_{t_0}(z_0)^\intercal \nabla^2\log p_{t_0}(z_0) E_{t_0}(z_0))^{-1} E_{t_0}(z_0)^\intercal \partial_t \nabla\log p_{t}(z_0)|_{t=t_0}\\
    &=(E_{t_0}(z_0)^\intercal \nabla^2\log p_{t_0}(z_0) E_{t_0}(z_0))^{-1} E_{t_0}(z_0)^\intercal \partial_t \nabla\log p_{t}(z_0)|_{t=t_0}.
\end{align*}
Therefore, the Lemma is proved by varying $t_0$ and $z_0$.
\end{proof}
\begin{lemma}[Expression of ridge motion]\label{lem:ridge motion expression} Let $E_t\in \mb{R}^{d\times(d-d^*)}$ be the normal eigen-matrix in the definition of $\mc{R}_t$, i.e., $\mc{R}_t = \{ x\in \mb{R}^d \mid E_t(x)\nabla \log p_t(x)=0,, \lambda_{d^*+1}\le -\beta_t  \}$. Then we have
\begin{align*}
    E_t(z)^\intercal \partial_t \nabla\log p_t(z) = \frac{1}{h_t} E_t(z)^\intercal \partial_t m(t,z).
\end{align*}   
\end{lemma}
\begin{proof}[Proof of Lemma \ref{lem:ridge motion expression}] According to Tweedie's formula, we have
\begin{align}
    E_t(x)\nabla\log p_t(x) & = \frac{1}{h_t} E_t(x) (m(t,x)-x), \nonumber \\
    \implies  E_t(x)\partial_t \nabla\log p_t(x) & = \frac{1}{h_t} E_t(x) \partial_t m(t,x) - \frac{\partial_t h_t}{h_t^2} E_t(x)(m(t,x) -x) .\label{eq:inter 1}
\end{align}
For $x=z\in \mc{R}_t$, we have
\begin{align*}
    E_t(x)\nabla \log p_t(x) = \frac{1}{h_t} E_t(x) (m(t,x)-x) =0.
\end{align*}
Therefore, the last term in \eqref{eq:inter 1} cancels. The Lemma is proved.
\end{proof}
\begin{lemma}\label{lem:key estimation} For any $z\in \mc{R}_t$, let $E_t(z)$ be the eigen-matrix in Definition \ref{def:ridge}. We have
\begin{align*}
     \| E_t(z)^\intercal \partial_t m(t,z) \| \lesssim (1+\Dot{h}_t) a_t R .
\end{align*}  
\end{lemma}
\begin{proof}[Proof of Lemma \ref{lem:key estimation}] By Lemma \ref{lem:explicit formula},
\begin{align*}
     &\quad \partial_t m(t,z)  \\
     & =  \partial_t (h_t \nabla\log p_t(z) + h_t z  ) \\
    & = \Dot{h}_t\nabla\log p_t(z) + h_t\nabla ( \frac{\partial_t p_t(z)}{p_t(z)} ) + \Dot{h}_t z  \\
    &= \Dot{h}_t\nabla\log p_t(z) + h_t\nabla ( \frac{z\cdot \nabla p_t(z)+\Delta p_t(z)+d p_t(z)}{p_t(z)} ) + \Dot{h}_t z \\
    & = \Dot{h}_t\nabla\log p_t(z) + h_t \nabla^2 \log p_t(z) z + d h_t \nabla\log p_t(z) + h_t \nabla (\frac{\Delta p_t(z)}{p_t(z)})  + \Dot{h}_t z \\
    & = (\Dot{h}_t+d h_t)\nabla\log p_t(z)+ h_t \nabla^2 \log p_t(z) z + \Dot{h}_t z \\
    &\quad + h_t( \mathrm{cont}_{23} \nabla^3\log p_t(z)+2\nabla^2\log p_t(z)\nabla\log p_t(z) ) .
\end{align*}
Since $E_t(z)^\intercal \nabla\log p_t(x)=0$ and $E_t(z)^\intercal \nabla^2\log p_t(z)\nabla\log p_t(z)=0$ for all $z\in \mc{R}_t$, we have
\begin{align*}
    E_t(z)^\intercal \partial_t m(t,z) &= h_t E_t(z)^\intercal\nabla^2 \log p_t(z)  z + \Dot{h}_t E_t(z)^\intercal z  + h_tE_t(z)^\intercal \mathrm{cont}_{23} \nabla^3\log p_t(z).
\end{align*}
Next we bound the three terms on the RHS. First, according to Lemma \ref{lem:explicit formula} and the fact that $E_t(z)^\intercal \nabla \log p_t(z)=0$,
\begin{align*}
    \| E_t(z)^\intercal z \| = \|  E_t(z)^\intercal m(t,z) \| \le \| m(t,z) \| \le a_t R.
\end{align*}
Next, according to Lemma \ref{lem:explicit formula}, $\nabla^2\log p_t=-\tfrac{1}{h_t} I_d + \tfrac{\Sigma(t,z)}{h_t^2}$ with $\Sigma(t,z)\succeq 0$. Therefore, all negative eigenvalues of $\nabla^2\log p_t$ are at least $-\tfrac{1}{h_t}$. On the other hand, $E_t(z)^\intercal$ only preserve eigenvalues that are smaller than $-\beta_t$. Therefore, $\nabla^2\log p_t(z)$ only contributes eigenvalues that are smaller than $-\beta_t$, hence negative, in $E_t(z)^\intercal\nabla^2 \log p_t(z)  z$. We have
\begin{align*}
    \| E_t(z)^\intercal\nabla^2 \log p_t(z)  z \| \le \frac{1}{h_t}\| E_t(z)^\intercal z \| \le \frac{a_t R}{h_t}.
\end{align*}
Regarding the last term, according to Lemma \ref{lem:important estimate},
\begin{align*}
    \| E_t(z)^\intercal \mathrm{cont}_{23} \nabla^3\log p_t(z)\| =  \| E_t(z)^\intercal (\nabla^3\log p_t(z) : I)\| \le \sqrt{d} \| \nabla^3\log p_t(z) \| = \mc{O}(\frac{\sqrt{d}\theta_t R^3}{h_t^3}) .
\end{align*}
Since $\theta_t$ can be arbitrarily chosen with order $\exp(- o(h_{t}^{-1}) )$, the last term is negligible in the final order estimation. Hence we proved Lemma \ref{lem:key estimation}.
\end{proof}

\section{Analysis of Stage 1}\label{append:stage 1}  In this section, we analyze the first stage: the reverse-time inference process enters the tube neighborhood of the log-density ridge with high probability. The formal version of Theorem \ref{thm:informal stage 1} is derived from considering exactly the reverse OU process as the reference process:
\begin{align*}
    \dee X_t^{\leftarrow} = \big( (1-\frac{2}{h_{T-t}})X_t^\leftarrow +\frac{2}{h_{T-t}}m(T-t,X_t^\leftarrow) \big) \dee t + \sqrt{2}\dee B_t^\leftarrow, \quad X_0^\leftarrow\sim p_T,
\end{align*}
and $X_t^\leftarrow\sim p_{T-t}$ for all $0\le t\le T_\delta$. We define the entering time of $X_t^\leftarrow$:
\begin{align*}
    \tin^\leftarrow \coloneqq \inf \{ 0\le t\le T-\delta | X_t^\leftarrow \in \mc{T}_{T-t}(\rho_{T-t}) \}.
\end{align*}
And we also analyze the reverse-time inference dynamics $Y_t$ with $$\tin \coloneqq \inf \{ 0\le t\le T-\delta \mid Y_t\in \mc{T}_{T-t}(\rho_{T-t}) \}.$$

Then the formal Theorem that describing our stage 1 states as follows:
\begin{theorem}\label{thm:formal stage 1} Under Assumption \ref{assump:data}, we have
    \begin{itemize}
        \item [(1)] $\mb{P}(\tin^\leftarrow \le T-\delta) \ge 1-e_\delta$;
        \item [(2)] $\mb{P}(\tin \le T-\delta) \ge 1-e_\delta - \varepsilon(T)$;
        \item [(3)] $\mb{P}(\Tilde{t}_{\mathrm{in}} \le T-\delta) \ge 1-e_\delta - \varepsilon(T)-\sqrt{\varepsilon_A(T,\delta)/8}$,
    \end{itemize}
    where $e_\delta=h_\delta^\zeta$ for any $\zeta>0$ and $e_\delta\to 0$ at any polynomial order as $\delta\to 0^+$. $\varepsilon(T)= \frac{a_T}{2}(\frac{\sqrt{d}}{\sqrt{h_T}} + R)\to 0$ exponentially fast as $T\to \infty$.  
\end{theorem}
\begin{proof}[Proof of Theorem \ref{thm:formal stage 1}] First, for the exact reverse OU process, we have
    \begin{align*}
    \mb{P}(X_{T-\delta}^\leftarrow\in \mc{T}_\delta(\rho_\delta)) = \mb{P}(X_\delta\in \mc{T}_\delta(\rho_\delta)) = \mb{P}(\mathrm{dist}(X_\delta,\mc{R}_\delta)\le \rho_\delta)\ge \mb{P}( \mathrm{dist}(a_\delta X_0,\mc{R}_\delta) + \sqrt{h_\delta}\|z\|\le \rho_\delta ).
\end{align*}
We will first show $\mathrm{dist}(a_\delta X_0,\mc{R}_\delta)$ is very small when $\delta\ll 1$, and then apply the concentration of $d$-dimensional Gaussian to bound the probability.

According to Lemma \ref{lem:existence of ridge points}, if $\delta$ is small enough such that $h_\delta^{-1}> \tfrac{2}{a_\delta^2 \Delta^2}\ln(\tfrac{4a_\delta^2R^2(n-1)}{\rho_\delta})$, then 
\begin{align*}
    \mathrm{dist}(a_\delta X_0,\mc{R}_\delta) \le \| z_i - a_\delta x_0^{(i)} \| \le \rho_\delta/2.
\end{align*} 
Then we have 
\begin{align*}
     \mb{P}(X_{T-\delta}^\leftarrow\in \mc{T}_\delta(\rho_\delta)) \ge \mb{P}(\sqrt{h_\delta}\|z\|\le \rho_\delta) = \mb{P}(\|z\|^2\le \rho_\delta^2/h_\delta)
\end{align*}
According to the order estimation of $r_t$ in Proposition \ref{prop:constant estimation}, we can pick $\delta$ small such that $\rho_\delta^2/h_\delta\le d+2\sqrt{d \zeta \log (1/h_\delta)}+2\zeta\log(1/h_\delta)$ for any $\zeta>0$. Then according to the LMI inequality~\citep{moshksar2024refining}, we have
\begin{align*}
    \mb{P}(X_{T-\delta}^\leftarrow\in \mc{T}_\delta(\rho_\delta)) \ge \mb{P}(\|z\|^2\le d+2\sqrt{d \zeta \log (1/h_\delta)}+2\zeta\log(1/h_\delta))\ge 1- h_\delta^\zeta.
\end{align*}

Next, for the reverse-time inference process $Y_t$. Notice that $Y_t$ and $X_t^\leftarrow$ have the same generator but with different initializations: $Y_0\sim \mc{N}(0,I_d)$ while $X_0^\leftarrow\sim p_T$. Therefore, we have
\begin{align*}
      |\mb{P}(Y_{T-\delta}\in \mc{T}_\delta(\rho_\delta))- \mb{P}(X_{T-\delta}^\leftarrow\in \mc{T}_\delta(\rho_\delta))| &\le  \TV( \mathrm{Law}(Y_{T-\delta}), p_\delta )\le \TV(\mc{N}(0,I_d), p_T) \\
    &\le \sqrt{\frac{\KL(p_T|\mc{N}(0,I_d))}{2}} \le \frac{a_T}{2}(\frac{\sqrt{d}}{\sqrt{h_T}} + R),
\end{align*}
where the second inequality follows from data processing inequality. The third inequality follows from Pinsker's inequality. The last inequality follows from property of OU process, see \cite[Proposition C.1]{he2024zeroth}. Therefore, we proved
\begin{align*}
    \mb{P}(Y_{T-\delta}\in \mc{T}_\delta(\rho_\delta))\ge \mb{P}(X_{T-\delta}^\leftarrow\in \mc{T}_\delta(\rho_\delta)) - \frac{a_T}{2}(\frac{\sqrt{d}}{\sqrt{h_T}} + R)\ge 1-h_\delta^\zeta - \frac{a_T}{2}(\frac{\sqrt{d}}{\sqrt{h_T}} + R).
\end{align*}
Last, for the inference process $\Tilde{Y}_t$, except for initialization error, there are extra trajectory error from $X_t^\leftarrow$ due to the approximate posterior mean $m_A$. We have
\begin{align*}
    |\mb{P}(\Tilde{Y}_{T-\delta}\in \mc{T}_\delta(\rho_\delta))- \mb{P}(X_{T-\delta}^\leftarrow\in \mc{T}_\delta(\rho_\delta))| \le \TV( \mathrm{Law}(\Tilde{Y}_{T-\delta}), p_\delta )\le \sqrt{\frac{\KL(\mb{P}^\leftarrow|\Tilde{\mb{P}})}{2}},
\end{align*}
where $\mb{P}^\leftarrow$ is the path measure of $X_t^\leftarrow$ and $\Tilde{\mb{P}}$ is the path measure of $\Tilde{Y}_t$. Then we can apply the traditional analysis of the inference process via Girsanov's Theorem. We get 
\begin{align*}
    &\quad |\mb{P}(\Tilde{Y}_{T-\delta}\in \mc{T}_\delta(\rho_\delta))- \mb{P}(X_{T-\delta}^\leftarrow\in \mc{T}_\delta(\rho_\delta))| \\
    &\le \sqrt{\frac{\KL(p_T|\mc{N}(0,I_d))}{2} + \frac{1}{8}\int_\delta^T h_t^{-2}\mb{E}[\| m(t,X_t)-m_A(t,X_t) \|^2] \dee t }\\
    &\le \frac{a_T}{2}(\frac{\sqrt{d}}{\sqrt{h_T}} + R) + \sqrt{\frac{1}{8}\int_\delta^T h_t^{-2}\mb{E}[\| m(t,X_t)-m_A(t,X_t) \|^2] \dee t }.
\end{align*}
The statements follows from transferring the probability to the defined stopping times.
\end{proof}
\begin{lemma}\label{lem:existence of ridge points} Under Assumption \ref{assump:data}, when $\delta\ll 1$, for each $i\in [n]$, there exists a point $z_i\in \mc{R}_\delta$ such that
    \begin{align*}
        \| z_i - a_\delta x_0^{(i)} \| \le 2a_\delta R (n-1)\exp(-\frac{a_\delta^2\Delta^2}{2h_\delta}).
    \end{align*}
\end{lemma}
\begin{proof}[Proof of Lemma \ref{lem:existence of ridge points}]
    According to Lemma \ref{lem:explicit formula}, we know 
\begin{align*}
    \nabla\log p_\delta (x) &= -\frac{x}{h_\delta} + \frac{1}{h_\delta}\sum_{i=1}^n \Softmax(-\frac{\|x-a_\delta x_0\|^2}{2h_\delta})_i a_\delta x_0^{(i)} \\
    \nabla^2 \log p_\delta(x) & = -\frac{1}{h_\delta} I_d + \frac{1}{h_\delta^2}\sum_{i=1}^n \Softmax(-\frac{\|x-a_\delta x_0\|^2}{2h_\delta})_i (a_\delta x_0^{(i)}-m(\delta,x))(a_\delta x_0^{(i)}-m(\delta,x))^\intercal.
\end{align*}
Next, we find critical points of $\nabla\log p_\delta(x)$ and show that they are on the log-density ridge $\mc{R}_\delta$. 

According to the expression of $\nabla\log p_\delta(x)$, the critical points are fixed points of $$x=\sum_{i=1}^n \Softmax(-\frac{\|x-a_\delta x_0\|^2}{2h_\delta})_i a_\delta x_0^{(i)}=m(\delta,x).$$ We claim: for each $i\in[n]$, within $B(a_\delta x_0^{(i)}, a_\delta \Delta/4)$, there exists a unique fixed point, denoted as $z_i$. We prove the claim by the contraction mapping theorem. First, $B(a_\delta x_0^{(i)}, a_\delta \Delta/4)\subset \mb{R}^d$ is a closed ball. Second, for all $x\in B(a_\delta x_0^{(i)}, a_\delta \Delta/4)$, for all $j\neq i$, we have
\begin{align*}
    & \frac{\Softmax(-\frac{\|x-a_\delta x_0\|}{2h_\delta})_j}{\Softmax(-\frac{\|x-a_\delta x_0\|}{2h_\delta})_i} = \exp ( - \frac{\|x-a_\delta x_0^{(j)}\|}{2h_\delta}+\frac{\|x-a_\delta x_0^{(i)}\|}{2h_\delta})\le \exp(-\frac{a_\delta^2\Delta^2}{2h_\delta}), \\
    \implies &\sum_{j\neq i} \Softmax(-\frac{\|x-a_\delta x_0\|}{2h_\delta})_j \le (n-1)\exp(-\frac{a_\delta^2\Delta^2}{2h_\delta}).
\end{align*}
Hence, we can show $m(\delta,\cdot ):B(a_\delta x_0^{(i)}, a_\delta \Delta/4)\to B(a_\delta x_0^{(i)}, a_\delta \Delta/4)$, i.e., for all $x\in B(a_\delta x_0^{(i)}, a_\delta \Delta/4)$,
\begin{align*}
    \| m(\delta,x)- a_\delta x_0^{(i)} \| &\le \sum_{j\neq i} \Softmax(-\frac{\|x-a_\delta x_0\|}{2h_\delta})_j a_\delta \|x_0^{(j)}-x_0^{(i)}  \|\\
    &\le 2a_\delta R (n-1)\exp(-\frac{a_\delta^2\Delta^2}{2h_\delta}) < a_\delta \Delta/4,
\end{align*}
given the early stopping time is small enough such that $h_\delta^{-1}> \tfrac{2}{a_\delta^2 \Delta^2}\ln(\tfrac{8R(n-1)}{\Delta})$. Meanwhile, for $x\in B(a_\delta x_0^{(i)}, a_\delta \Delta/4)$, similar to the proof of Lemma \ref{lem:important estimate}, 
\begin{align*}
    \| \nabla m(\delta,x) \|&\le \frac{1}{h_\delta}\sum_{j\neq i} \Softmax(-\frac{\|x-a_\delta x_0\|}{2h_\delta})_j a_\delta^2 \| x_0^{(j)}-x_0^{(i)} \|^2 \\
    &\le \frac{a_\delta^2 R^2 (n-1)}{h_\delta} \exp(-\frac{a_\delta^2\Delta^2}{2h_\delta})<1,
\end{align*}
given the early stopping time is small enough such that $h_\delta^{-1}> \tfrac{2}{a_\delta^2 \Delta^2}\ln(\tfrac{a_\delta^2R^2(n-1)}{h_\delta})$. Therefore, according to the contraction mapping theorem, there exists a unique fixed point $z_i\in B(a_\delta x_0^{(i)}, a_\delta \Delta/4)$. At each $z_i$, according to Lemma \ref{lem:explicit formula}, we have
\begin{align*}
    \nabla^2 \log p_\delta(z_i) &= -\frac{1}{h_\delta} I_d + \frac{\mathrm{Cov}(U(\delta,x))}{h_\delta^2} \\
    &\preceq -\frac{1}{h_\delta} I_d + \frac{1}{h_\delta^2} a_\delta^2 R^2 (n-1)\exp(-\frac{a_\delta^2\Delta^2}{2h_\delta}) I_d \\
    &\preceq -\frac{1}{2h_\delta }I_d,
\end{align*}
given the early stopping time is small enough such that $h_\delta^{-1}> \tfrac{2}{a_\delta^2 \Delta^2}\ln(\tfrac{a_\delta^2R^2(n-1)}{2h_\delta})$. Therefore, we proved that $z_i\in \mc{R}_\delta$ with $\beta_\delta = \Theta(\tfrac{1}{h_\delta}) $. Last, we estimate the distance from $z_i$ to $a_\delta x_0^{(i)}$.
\begin{align*}
    \| z_i-a_\delta x_0^{(i)} \| & = \| m(t,z_i)-x_0^{(i)} \|\le \sum_{j\neq i} \Softmax(-\frac{\|z_i-a_\delta x_0\|}{2h_\delta})_j a_\delta \| x_0^{(j)}-x_0^{(i)} \| \\
    &\le 2a_\delta R (n-1)\exp(-\frac{a_\delta^2\Delta^2}{2h_\delta}).
\end{align*}
\end{proof}
\section{Analysis of Stage 2}\label{append:stage 2} In this section, we analyze the stage 2 - align along normal directions. We start from stating the formal version of Theorem \ref{them:normal contraciton} which includes the dynamical property of the squared normal distance to the ridge.
\begin{theorem}\label{thm: formal normal contraciton} Under Assumption \ref{assump:data}, define $\kappa_{s,t}=2\int_s^t (\beta_{T-u}-1)\dee u$, $B_{s,t}(d,R)=\int_{s}^t e^{-\kappa_{u,t} } ( \rho_{T-u}R + d ) \dee u$ and $e_A^\perp(t,x)\coloneqq P^\perp(\Pi_t(x))e_A(t,x)$ with $e_A=m_A-m$. Then for any $t\in [\Tilde{t}_{\mathrm{in}},T-\delta]$,
{
    \begin{align*}
    \mb{E}[D_{T-t}(\Tilde{Y}_t)] &\lesssim e^{-\kappa_{\Tilde{t}_{\mathrm{in}},t} }\rho_{T-\Tilde{t}_{\mathrm{in}}}  + B_{\Tilde{t}_{\mathrm{in}},t}(d,R)  + \int_{\Tilde{t}_{\mathrm{in}}}^t e^{-\kappa_{u,t} } \frac{\mb{E}[\| e_A^\perp(T-u,\Tilde{Y}_u)\|^2]}{h_{T-u}^2\beta_{T-u}} \dee u. 
    \end{align*}}
    Furthermore, picking picking $\beta_t=c/h_t$ for any $c\in [\frac{1}{2},1)$\footnote{$c$ can be chosen arbitrarily between $[\tfrac{1}{2},1)$ due to the property of $\nabla^2\log p_t(x)$ as explained in Remark~\ref{rem:choice of beta}. }, when $\delta\ll 1$, we have
    {
\begin{align*}
    \mb{E}[D_{\delta}(\Tilde{Y}_{T-\delta})] = \mc{O}(d\delta^c+ \delta^c\int_{\Tilde{t}_{\mathrm{in}}}^{T-\delta}  h_{T-u}^{-1-c}\,{\mb{E}[\| e_A^\perp(T-u,\Tilde{Y}_u)\|^2]} \dee u  ).
\end{align*}}
\end{theorem}
We will prove Theorem \ref{thm: formal normal contraciton} and Theorem \ref{them:normal contraciton} naturally follows from it. The proof relies on a property of the square-normal distance, which simply follows from the log-density ridge property proved in Appendix \ref{append:data dependent motion estimation}. 
\begin{corollary}\label{cor:distance} As a consequence of Proposition \ref{prop:tube well-defined}, the following properties hold for the square-distance function $D_t(x)\coloneqq \| x-\Pi_t(x) \|^2$:
    \begin{itemize}
        \item [(1)] $\mc{D}_t\in C^2(\mc{T}_t(\rho_t))$ and $\nabla D_t(x)=2n_t(x)$;
        \item [(2)] $\sup_x \| \nabla^2 D_t(x) \|\le \tfrac{2}{1-\rho_t/r_t}$. As a consequence, $\sup_x \Delta D_t(x)\le \tfrac{2d}{1-\rho_t/r_t}$.
    \end{itemize}
\end{corollary}
We now prove Theorem \ref{thm: formal normal contraciton} for both $Y_t$ and $\Tilde{Y}_t$, given below
\begin{align*}
    \dee Y_t &= \underbrace{(Y_t - \frac{2}{h_{T-t}}Y_t + \frac{2}{h_{T-t}} m(T-t,Y_t)}_{\coloneqq b(T-t,Y_t)} \dee t +\sqrt{2} \dee \Bar{B}_t,  \quad Y_0\sim \mc{N}(0,I_d), \\
    \dee \Tilde{Y}_t &= \underbrace{(\Tilde{Y}_t -\frac{2}{h_{T-t}}\Tilde{Y}_t + \frac{2}{h_{T-t}} m_A(T-t,\Tilde{Y}_t))}_{\coloneqq b_A(T-t,\Tilde{Y}_t)} \dee t +\sqrt{2} \dee \Tilde{B}_t,  \quad \Tilde{Y}_0\sim \mc{N}(0,I_d).
\end{align*}
Note that for $Y_t$, it is differed from exact reverse-time OU only in initialization. Therefore, we would like to highlight the relation between reverse OU-dynamics and the log-density ridge geometry through analysis of $Y_t$. The gap between $Y_t$ and $\Tilde{Y}_t$ lies in the approximation error of the posterior mean. Therefore, our analysis for $\Tilde{Y}_t$ will highlight the effect of posterior mean approximation error.

\begin{proof}[Proof of Theorem \ref{thm: formal normal contraciton}]
For any $t\in [0,T-\delta]$, $x\in \mc{T}_{T-t}(\rho_t)$, recall that $D_{T-t}(x)=\| n_{T-t}(x) \|^2$ and $n_{T-t}(x)=x-\Pi_{T-t}(x)$. Once the processes enter $\mc{T}_{T-t}(\rho_t)$, we can track the evolution of $D_{T-t}$ via It\^o's formula. For the reverse process $Y_t$, we have
\begin{align*}
    \dee D_{T-t}(Y_t) &= \big( \partial_t D_{T-t}(Y_t) + 2\langle n_{T-t}(Y_t), b(T-t,Y_t) \rangle + 2\Tr(\nabla^2 D_{T-t}(Y_t))\big)\dee t \\
    &\quad + 2\sqrt{2}\langle n_{T-t}(Y_t),\dee \bar{B}_t \rangle,
\end{align*}
where we have the following estimations for the drift terms:
\begin{itemize}
    \item [(1)] According to Corollary \ref{cor:distance}, $2\Tr(\nabla^2 D_{T-t}(Y_t))\le \tfrac{4d}{1-\rho_{T-t}/r_{T-t}}$. Picking $\rho_t=r_t/2$ for all $t$ and we get $2\Tr(\nabla^2 D_{T-t}(Y_t))\le 8d$.
    \item [(2)] According to Remark \ref{rem:ridge motion bound}, 
    \begin{align*}
        \partial_t D_{T-t}(Y_t)&=-\partial_s D_s(Y_t)|_{s=T-t}= 2\langle n_{T-t}(Y_t), \partial_s \Pi_s (Y_t) |_{s=T-t}  \rangle \\
        &\le 2\rho_{T-t} \sup_{x\in \mc{T}_{T-t}(\rho_{T-t})}\| \partial_s\Pi_s(x)|_{s=T-t} \|  \coloneqq 2\rho_{T-t} S_{T-t} ,
    \end{align*}
    and $S_{T-t}= \mc{O}(R)$ when $T-t\to 0^+$.
    \item [(3)] if we denote $z=z_{T-t}\coloneqq \Pi_{T-t}(Y_t)$, $n_{T-t}(Y_t)\perp T_z(\mc{R}_{T-t})$ and we have
    \begin{align*}
         2\langle n_{T-t}(Y_t) , b(T-t,Y_t) \rangle  & = 2\langle  n_{T-t}(Y_t) ,  Y_t \rangle + 4 \langle  n_{T-t}(Y_t) , \nabla\log p_{T-t}(Y_t) \rangle \\
        & = 2\| n_{T-t}(Y_t) \|^2 + 4 \langle  n_{T-t}(Y_t) , \nabla\log p_{T-t}(Y_t) \rangle,
    \end{align*}
    In the last term, we can write $Y_t=z+n_{T-t}(Y_t)$ and expand $\nabla\log p_{T-t}(Y_t)$ at $z\in \mc{R}_{T-t}$. Then we have
    \begin{align*}
        &\quad \langle  n_{T-t}(Y_t) , \nabla\log p_{T-t}(Y_t) \rangle \\
        & = \langle  n_{T-t}(Y_t) ,  \nabla\log p_{T-t}(z) \rangle + \langle  n_{T-t}(Y_t) ,  \nabla^2\log p_{T-t}(z) n_{T-t}(Y_t) \rangle \\
        & \quad  + \frac{1}{2}\langle  n_{T-t}(Y_t) ,  \nabla^3\log p_{T-t}(z') n_{T-t}(Y_t)^{\otimes 2} \rangle ,
    \end{align*}
    where
    \begin{align*}
       & \langle  n_{T-t}(Y_t) ,  \nabla\log p_{T-t}(z) \rangle = 0 ,\quad &\text{definition of }\mc{R}_{T-t},\\
       & \langle  n_{T-t}(Y_t) ,  \nabla^2\log p_{T-t}(z) n_{T-t}(Y_t) \rangle  \quad & \\
       = & \langle  n_{T-t}(Y_t) ,  P^\perp(z)\nabla^2\log p_{T-t}(z) n_{T-t}(Y_t) \rangle \le  -\beta_{T-t}\| n_{T-t}(Y_t) \|^2, \quad & \text{definition of }\mc{R}_{T-t},\\ 
       &\langle  n_{T-t}(Y_t) ,  \nabla^3\log p_{T-t}(z') n_{T-t}(Y_t)^{\otimes 2} \rangle \quad & \\
       \le & \sup_x \| \nabla \log p_{T-t}(x) \| \| n_{T-t}(Y_t) \|^3 \le  \frac{80a_{T-t}^3 R^3}{h_{T-t}^3} \rho_{T-t} \| n_{T-t}(Y_t) \|^2 \quad & \text{Lemma }\ref{lem:important estimate} \\
       \le & \frac{a_{T-t}^3\beta_{T-t}}{2} D_{T-t}(Y_t) \quad & \text{Proposition }\ref{prop:constant estimation}
    \end{align*}
    The last identity follows from Proposition \ref{prop:constant estimation} by picking $\theta_t$ such that $r_t= \beta_t h_t^3R^{-3}/80$ and $\rho_t=r_t/2$. Combining the above inequalities, we have
    \begin{align*}
        \langle  n_{T-t}(Y_t) , \nabla\log p_{T-t}(Y_t) \rangle &\le -\beta_{T-t}(1-a_{T-t}^3/4) \| n_{T-t}(Y_t) \|^2\le -\frac{3}{4}\beta_{T-t} D_{T-t}(Y_t).
    \end{align*}
    Therefore, we have
\begin{align*}
    2\langle n_{T-t}(Y_t), b(T-t,Y_t) \rangle  \le -(3\beta_{T-t}-2)D_{T-t}(Y_t).
\end{align*}
\end{itemize}
Combining all the estimations and taking expectations of $D_{T-t}$, we obtain the following inequality
\begin{align*}
    \frac{\dee}{\dee t}\mb{E}[D_{T-t}(Y_t)] \le -(3\beta_{T-t}-2)\mb{E}[D_{T-t}(Y_t)] + 2\rho_{T-t}S_{T-t} + 8d.
\end{align*}
Last, apply Gronwall's inequality, for any $\tin\in (0, T-\delta)$ and $t\in (\tin,T-\delta]$, we obtain
\begin{align*}
    \mb{E}[D_{T-t}(Y_t)] &\le \exp\big( -\int_{\tin}^t 3\beta_{T-u}-2 \dee u \big)\mb{E}[D_{t-\tin}(Y_{\tin})]  \\
    &\quad + \int_{\tin}^t \exp\big( -\int_{u}^t 3\beta_{T-s}-2 \dee s \big) \big( 2\rho_{T-u}S_{T-u} + 8d \big) \dee u .
\end{align*}
When $t=T-\delta$, $\beta_t=c/h_t$ and $\delta\ll 1$, we have
\begin{align*}
    \int_{\tin}^{T-\delta} \exp\big( -\int_{u}^{T-\delta
    } 3\beta_{T-s}-2 \dee s \big) \big( 2\rho_{T-u}S_{T-u} + 8d \big) \dee u= \mc{O}(d\delta^{\frac{3c}{2}}).
\end{align*}   

\noindent For the reverse-time inference process $\Tilde{Y}_t$, we have
\begin{align*}
    \dee D_{T-t}(\Tilde{Y}_t) &= \big( \partial_t D_{T-t}(\Tilde{Y}_t) + 2\langle n_{T-t}(\Tilde{Y}_t), b_A(T-t,\Tilde{Y}_t) \rangle \\
    &\quad + 2\Tr(\nabla^2 D_{T-t}(\Tilde{Y}_t))\big)\dee t + 2\sqrt{2}\langle n_{T-t}(\Tilde{Y}_t),\dee \Tilde{B}_t \rangle.
\end{align*}
Notice that the only difference to that of $Y_t$ is the error $e_A(t,\cdot)=m(t,\cdot)-m_A(t,\cdot)$ within the normal space, i.e., $e_A^\perp(t,x)\coloneqq P^\perp(\Pi_t(x)) e_A(t,x)$. Similarly, we have the inequality
\begin{align*}
     &\quad \frac{\dee}{\dee t}\mb{E}[D_{T-t}(\Tilde{Y}_t)] \\
     &\le -(3\beta_{T-t}-2)\mb{E}[D_{T-t}(\Tilde{Y}_t)] + 2\rho_{T-t}S_{T-t} + 8d + 2\mb{E}[\langle n_{T-t}(\Tilde{Y}_t), \frac{1}{h_{T-t}}e_A(T-t,\Tilde{Y}_t) \rangle]\\
      &\le -2(\beta_{T-t}-1)\mb{E}[D_{T-t}(\Tilde{Y}_t)] + 2\rho_{T-t}S_{T-t} + 8d + \frac{\mb{E}[\| e_A^\perp(T-t,\Tilde{Y}_t)\|^2]}{h_{T-t}^2\beta_{T-t}},
\end{align*}
where the last inequality follows from Young's inequality. Therefore, according to the Gronwall's inequality,
\begin{align*}
    \mb{E}[D_{T-t}(\Tilde{Y}_t)] & \le  \exp\big( -2\int_{\Tilde{t}_{\mathrm{in}}}^t \beta_{T-u}-1 \dee u \big)\mb{E}[D_{t-\Tilde{t}_{\mathrm{in}}}(\Tilde{Y}_{\Tilde{t}_{\mathrm{in}}})] \\
    &\quad + \int_{\Tilde{t}_{\mathrm{in}}}^t \exp\big( -2\int_{u}^t \beta_{T-s}-1 \dee s \big) \big( 2\rho_{T-u}S_{T-u}+\frac{\mb{E}[\| e_A^\perp(T-u,\Tilde{Y}_u)\|^2]}{h_{T-u}^2\beta_{T-u}}+ 8d \big) \dee u .
\end{align*}
When $t=T-\delta$, $\beta_t=c/h_t$ and $\delta\ll 1$, we have
\begin{align*}
    &\quad\int_{\Tilde{t}_{\mathrm{in}}}^{T-\delta} \exp\big( -2\int_{u}^{T-\delta} \beta_{T-s}-1 \dee s \big) \dee u =\mc{O}(\delta^c) ,\\
    &\quad\int_{\Tilde{t}_{\mathrm{in}}}^{T-\delta} \exp\big( -2\int_{u}^{T-\delta} \beta_{T-s}-1 \dee s \big) h_{T-u}^{-2}\beta_{T-u}^{-1}\mb{E}[\| e_A^\perp(T-u,\Tilde{Y}_u)\|^2] \dee u \\
    & =\mc{O}(\delta^c \int_{\Tilde{t}_{\mathrm{in}}}^{T-\delta}  h_{T-u}^{-1-c}\mb{E}[\| e_A^\perp(T-u,\Tilde{Y}_u)\|^2] \dee u ) .
\end{align*}
\end{proof}

\section{Analysis of Stage 3}\label{append:stage 3} In this section, we analyze Stage 3, namely the tangential sliding behavior of both the ideal reverse process $Y_t$ and the learned reverse process $\Tilde{Y}_t$ relative to the log-density ridge geometry. We first state a formal dynamical version of Theorem~\ref{them:tangent contraciton}, which controls the evolution of the squared tangential residual along inference. The terminal-time estimate stated in Theorem~\ref{them:tangent contraciton} in the main text follows directly by integrating this formal bound.
\begin{theorem}\label{thm:formal tangent contraciton} Under Assumption \ref{assump:data}, for any $t\in [\Tilde{t}_{\mathrm{in}},T-\delta]$, define  $e_A^{\parallel,i}(t,x)=(U_{t}^{(i)})^\intercal e_A(t,x)$ with $e_A=m_A-m$. If $\Tilde{Y}_t\in \mc{B}_{T-t}^{(i)}(\theta_{T-t})$,  then we have
{
    \begin{align*}
     \frac{\dee}{\dee t} \mb{E}[\| \Tilde{u}_t^{(i)} \|^2]  \le &-(\frac{1}{h_{T-t}}-2)\mb{E}[\| \Tilde{u}_t^{(i)} \|^2]  + \Tilde{\epsilon}_{T-t}^{(i)}+ \frac{4}{h_{T-t}}\|e_A^{\parallel,i}(T-t,\Tilde{Y}_t) \|^2,
\end{align*}}
where $\Tilde{\epsilon}_{T-t}^{(i)}=\mc{O}(d)$.
    Furthermore, with $\delta\ll 1$, we have
    {
\begin{align*}
     \mb{E}[\| \Tilde{u}_{T-\delta}^{(i)} \|^2]= \mc{O}( d\sqrt{\delta} + \sqrt{\delta} \int_{\Tilde{t}_{\mathrm{in}}}^{T-\delta}  h_{T-u}^{-\frac{3}{2}}{\mb{E}[\|e_A^{\parallel,i}(T-u,\Tilde{Y}_{u}) \|^2]} \dee u  ).
\end{align*}}
\end{theorem}
We now introduce the objects used in the proof. Recall that the ideal and learned reverse processes satisfy
\begin{align*}
    \dee Y_t &= \underbrace{(Y_t - \frac{2}{h_{T-t}}Y_t + \frac{2}{h_{T-t}} m(T-t,Y_t)}_{\coloneqq b(T-t,Y_t)} \dee t +\sqrt{2} \dee \Bar{B}_t,  \quad Y_0\sim \mc{N}(0,I_d), \\
    \dee \Tilde{Y}_t &= \underbrace{(\Tilde{Y}_t -\frac{2}{h_{T-t}}\Tilde{Y}_t + \frac{2}{h_{T-t}} m_A(T-t,\Tilde{Y}_t))}_{\coloneqq b_A(T-t,\Tilde{Y}_t)} \dee t +\sqrt{2} \dee \Tilde{B}_t,  \quad \Tilde{Y}_0\sim \mc{N}(0,I_d).
\end{align*} 

In Section \ref{sec:stage3}, we introduced the tangent frame  $U_{T-t}^{(i)}\in \mb{R}^{d\times d^*}$ at the center $m_{T-t}^{(i)}$, whose columns are the top-$d^*$ eigenvectors of $\nabla\log p_{T-t}(m_{T-t}^{(i)})$. For trajectories lying in the tube neighborhood and the $i^{th}$ center-dominant region, the natural residuals relative to the local center are $r_t^{(i)} = Y_t-m_{T-t}^{(i)}$ and $\Tilde{r}_t^{(i)} = \Tilde{Y}_t-m_{T-t}^{(i)}$. We then define the corresponding tangential coordinates $\Tilde{u}_t^{(i)} \coloneqq (U_{T-t}^{(i)})^\intercal  r_t^{(i)}  \in \mb{R}^{d^*}$ and $u_t^{(i)} \coloneqq (U_{T-t}^{(i)})^\intercal  \Tilde{r}_t^{(i)}  \in \mb{R}^{d^*}$. These are the appropriate quantities to study because Stage 3 concerns motion along the ridge rather than distance to the ridge. Strictly speaking, the ideal tangent frame wold be taken at the projected ridge point $\Pi_t(m_{T-t}^{(i)})$, not at the center $m_{T-t}^{(i)}$ itself. However, Lemma \ref{lem:existence of ridge points} shows that $m_{T-t}^{(i)}$ is exponentially close to $\Pi_t(m_{T-t}^{(i)})$ as $T-t\to \delta^+\ll 1$. Therefore, $u_t^{(i)}$ and $\Tilde{u}_t^{(i)}$ provide accurate local proxies for the tangential components of the residuals, while keeping the analysis explicit and tractable.

We next prove the formal tangential contraction theorem.
\begin{proof}[Proof of Theorem \ref{them:tangent contraciton}]
    Apply It\^o's formula to $u_t^{(i)}$ and we get
\begin{align*}
    \dee u_t^{(i)} & = \partial_t (U_{T-t}^{(i)})^\intercal (Y_t-m_{T-t}^{(i)}) \dee t + (U_{T-t}^{(i)})^\intercal \big( b(T-t,Y_t)\dee t - \partial_t a_{T-t}x_0^{(i)}\dee t + \sqrt{2}\dee \bar{B}_t \big) \\
    & = -(\frac{2}{h_{T-t}}-1) u_t^{(i)}\dee t + \partial_t (U_{T-t}^{(i)})^\intercal (Y_t-m_{T-t}^{(i)}) \dee t  + \sqrt{2}(U_{T-t}^{(i)})^\intercal\dee \bar{B}_t \\
    & \quad + (U_{T-t}^{(i)})^\intercal\big( b(T-t,Y_t)- ( Y_t +\frac{2}{h_{T-t}}(m_{T-t}^{(i)}-Y_t) ) \big) \dee t.
\end{align*}
Now apply It\^o's formula again to $\| u_t^{(i)} \|^2$ and we get,
\begin{align*}
    \dee \| u_t^{(i)} \|^2 & = 2\langle u_t^{(i)},\dee u_t^{(i)} \rangle + 2\Tr((U_{T-t}^{(i)})^\intercal U_{T-t}^{(i)}) \dee t \\
    & = -2(\frac{2}{h_{T-t}}-1) \|u_t^{(i)}\|^2 \dee t + 2 \langle u_t^{(i)}, \partial_t (U_{T-t}^{(i)})^\intercal (Y_t-m_{T-t}^{(i)}) \rangle \dee t \\
    &\quad + 2\langle u_t^{(i)}, (U_{T-t}^{(i)})^\intercal\big( b(T-t,Y_t)- ( Y_t +\frac{2}{h_{T-t}}(m_{T-t}^{(i)}-Y_t) ) \big)\rangle \dee t \\
    & \quad + 2d^* \dee t + 2\sqrt{2} \langle u_t^{(i)}, (U_{T-t}^{(i)})^\intercal \dee \bar{B}_t\rangle,
\end{align*}
where according to Lemma \ref{lem:property local frame},
\begin{align*}
      &\quad \langle u_t^{(i)}, \partial_t (U_{T-t}^{(i)})^\intercal (Y_t-m_{T-t}^{(i)}) \rangle \\
      & = \langle u_t^{(i)}, \partial_t (U_{T-t}^{(i)})^\intercal U_{T-t}^{(i)} u_t^{(i)} \rangle +  \langle u_t^{(i)}, \partial_t (U_{T-t}^{(i)})^\intercal (I_d - U_{T-t}^{(i)}(U_{T-t}^{(i)})^\intercal) (Y_t-m_{T-t}^{(i)}) \rangle \\
     & \le  0 + \frac{1}{h_{T-t}}\| u_t^{(i)} \|^2 + h_{T-t}\|P(m_{T-t}^{(i)}) \partial_t P(m_{T-t}^{(i)}) (I_d-P(m_{T-t}^{(i)})) \|^2 \| r_t^{(i)} \|^2,
\end{align*}
and according to Lemma \ref{lem:application closeness} and Assumption \ref{assump:data},
\begin{align*}
    &\quad \langle u_t^{(i)}, (U_{T-t}^{(i)})^\intercal\big( b(T-t,Y_t)- ( Y_t +\frac{2}{h_{T-t}}(m_{T-t}^{(i)}-Y_t) ) \big)\rangle \\
    & \le \| b(T-t,Y_t)- ( Y_t +\frac{2}{h_{T-t}}(m_{T-t}^{(i)}-Y_t) ) \| \| Y_t - a_{T-t}x_0^{(i)} \| \\
    &\le 8 R^2 a_{T-t}^2 h_{T-t}^{-1} \theta_{T-t},
\end{align*}
Therefore, taking expectations on both sides of the SDE for $\| u_t^{(i)} \|^2$, we have
\begin{align*}
     \frac{\dee}{\dee t} \mb{E}[\| u_t^{(i)} \|^2]  &\le -2(\frac{1}{h_{T-t}}-1)\mb{E}[\| u_t^{(i)} \|^2] + 2h_{T-t}\|P(m_{T-t}^{(i)}) \partial_t P(m_{T-t}^{(i)}) (I_d-P(m_{T-t}^{(i)})) \|^2 \| r_t^{(i)} \|^2\\
     &\quad+ 8 R^2 a_{T-t}^2 h_{T-t}^2 \eta_{T-t} + 2d^*,
\end{align*}
where as $T-t\to \delta^+\ll 1$,
\begin{itemize}
    \item [(1)] $-2(\frac{1}{h_{T-t}}-1)\mb{E}[\| u_t^{(i)} \|^2]$ is a strict contraction term with infinite force;
    \item [(2)] $2h_{T-t}\|P(m_{T-t}^{(i)}) \partial_t P(m_{T-t}^{(i)}) (I_d-P(m_{T-t}^{(i)})) \|^2 \| r_t^{(i)} \|^2=O(1)$: 
     \begin{itemize}
         \item [(i)] Under Assumption \ref{assump:data}, $\| r_t^{(i)} \|^2 = \mc{O}(R^2)$;
         \item [(ii)] According to the definition of $P(m_{T-t}^{(i)})$ and estimations Lemma \ref{lem:important estimate}, $$\|P(m_{T-t}^{(i)}) \partial_t P(m_{T-t}^{(i)}) (I_d-P(m_{T-t}^{(i)})) \|=\mc{O}(\theta_{T-t} \mathrm{poly}(R,h_t^{-1}))$$ for any $\theta_{T-t}=\exp(-o(\tfrac{1}{h_{T-t}}))$; 
     \end{itemize} 
     Combining the two estimations, we have the whole term is $\mc{O}(1)$. 
    \item [(3)] $8 R^2 a_{T-t}^2 h_{T-t}^2 \theta_{T-t}=\mc{O}(1)$ since we can choose $\theta_{T-t}=\exp(-o(\tfrac{1}{h_{T-t}}))$ as discussed in Remark \ref{rem:sufficient closeness}.
\end{itemize}
Therefore, we obtain that $\frac{\dee}{\dee t} \mb{E}[\| u_t^{(i)} \|^2]  \le -2(\tfrac{1}{h_{T-t}}-1)\mb{E}[\| u_t^{(i)} \|^2] + \epsilon_{T-t}^{(i)} $ with $\epsilon_{T-t}^{(i)}=\mc{O}(d)$ as $t\to T^-$. By Gronwall's inequality,
\begin{align*}
    \mb{E}[\| u_t^{(i)} \|^2]  \le  \exp\big( - 2\int_{\tin}^t \frac{1}{h_{T-u}}-1 \dee u\big) \mb{E}[\| u_{\tin}^{(i)} \|^2] + \int_{\tin}^t \exp\big( -2\int_u^t \frac{1}{h_{T-s}}-1 \dee s\big) \epsilon_{T-u}^{(i)}\dee u.
\end{align*}
When $t=T-\delta$ and $\delta\ll 1$, we have
\begin{align*}
 \int_{\tin}^{T-\delta} \exp\big( -2\int_u^{T-\delta} \frac{1}{h_{T-s}}-1 \dee s\big) \epsilon_{T-u}^{(i)}\dee u =\mc{O}( d\delta \log (1/\delta) ) .
\end{align*}

Hence we prove that as $t\to T-\delta$, with high probability (tends to $1$ as $\delta\to 0^+$), $Y_t$ will enter some region dominated by one of the center $\mc{B}_{T-t}^{(i)}(\theta_{T-t})$ and then be pulled towards the center $a_{T-t}x_0^{(i)}$. Furthermore, quantitatively, the expected square-norm of $Y_{t-\delta}-a_{\delta}x_0^{(i)}$ is of order $\mc{O}(\delta)$.   

For the reverse-time inference process $\Tilde{Y}_t$, we utilize the same idea. Define 
\begin{align*}
    \Tilde{r}_t^{(i)} = \Tilde{Y}_t-m_{T-t}^{(i)},\quad \Tilde{u}_t^{(i)} \coloneqq (U_{T-t}^{(i)})^\intercal  (\Tilde{Y}_t-m_{T-t}^{(i)}) .
\end{align*}
The following the same estimations, we have
\begin{align*}
    \frac{\dee}{\dee t} \mb{E}[\| \Tilde{u}_t^{(i)} \|^2]  \le -2(\frac{1}{h_{T-t}}-1)\mb{E}[\| \Tilde{u}_t^{(i)} \|^2] - \frac{4}{h_{T-t}}\langle \Tilde{u}_t^{(i)}, (U_{T-t}^{(i)})^\intercal e_A(T-t,\Tilde{Y}_t) \rangle + \Tilde{\epsilon}_{T-t}^{(i)}, 
\end{align*}
where $\Tilde{\epsilon}_{T-t}^{(i)}=\mc{O}(1)$ as $t\to T^-$. Apply Young's inequality, we have $\tfrac{4}{h_{T-t}}\langle \Tilde{u}_t^{(i)}, (U_{T-t}^{(i)})^\intercal e_A(T-t,\Tilde{Y}_t) \rangle\le \tfrac{1}{h_{T-t}}\| \Tilde{u}_t^{(i)} \|^2 + \tfrac{4}{h_{T-t}}\|(U_{T-t}^{(i)})^\intercal e_A(T-t,\Tilde{Y}_t) \|^2$. Therefore,
\begin{align*}
     \frac{\dee}{\dee t} \mb{E}[\| \Tilde{u}_t^{(i)} \|^2]  \le -(\frac{1}{h_{T-t}}-2)\mb{E}[\| \Tilde{u}_t^{(i)} \|^2] + \frac{4}{h_{T-t}}\|(U_{T-t}^{(i)})^\intercal e_A(T-t,\Tilde{Y}_t) \|^2 + \Tilde{\epsilon}_{T-t}^{(i)}.
\end{align*}
Apply Gronwall's inequality and we get
\begin{align*}
    \mb{E}[\| \Tilde{u}_t^{(i)} \|^2]  &
    \le \exp\big( -\int_{\Tilde{t}_{\mathrm{in}}}^t \frac{1}{h_{T-u}}-2 \dee u\big) \mb{E}[\| \Tilde{u}_{\Tilde{t}_{\mathrm{in}}}^{(i)} \|^2] + \int_{\Tilde{t}_{\mathrm{in}}}^t \exp\big( -\int_u^t \frac{1}{h_{T-s}}-2 \dee s\big) \epsilon_{T-u}^{(i)}\dee u\\
    &\quad + \int_{\Tilde{t}_{\mathrm{in}}}^t \exp\big( -\int_u^t \frac{1}{h_{T-s}}-2 \dee s\big) \frac{4}{h_{T-u}}\|(U_{T-u}^{(i)})^\intercal e_A(T-u,\Tilde{Y}_u) \|^2\dee u.
\end{align*}
When $t=T-\delta$ and $\delta\ll 1$, we have
\begin{align*}
    &\quad \int_{\Tilde{t}_{\mathrm{in}}}^{T-\delta} \exp\big( -\int_u^{T-\delta} \frac{1}{h_{T-s}}-2 \dee s\big) \epsilon_{T-u}^{(i)}\dee u =\mc{O}( \sqrt{\delta} ) , \\
    &\quad\int_{\Tilde{t}_{\mathrm{in}}}^{T-\delta} \exp\big( -\int_u^{T-\delta} \frac{1}{h_{T-s}}-2 \dee s\big) \frac{4}{h_{T-u}}\|(U_{T-u}^{(i)})^\intercal e_A(T-u,\Tilde{Y}_u) \|^2\dee u\\
    &= \mc{O}(\sqrt{\delta}\int_{\Tilde{t}_{\mathrm{in}}}^{T-\delta}  h_{T-t}^{-\frac{3}{2}}\|(U_{T-u}^{(i)})^\intercal e_A(T-u,\Tilde{Y}_u) \|^2\dee u).
\end{align*}
\end{proof}
\begin{remark}[Tangential distribution of the residual] To study the distribution of residual $\Tilde{Y}_{T-\delta}-m_{\delta}^{(i)}$ along the tangent direction, we look at the SDE of $\|\Tilde{u}_t^{(i)}\|^2$ and only look at the leading order drift as $t\to T^-$ and the diffusion term:
\begin{align*}
    \dee \| \Tilde{u}_t^{(i)}\|^2 &= -2(\frac{2}{h_{T-t}}-1) \|\Tilde{u}_t^{(i)}\|^2 \dee t + 2 \langle u_t^{(i)}, \partial_t (U_{T-t}^{(i)})^\intercal (Y_t-m_{T-t}^{(i)}) \rangle \dee t \\
    &\quad + 2\langle \Tilde{u}_t^{(i)}, (U_{T-t}^{(i)})^\intercal\big( b_A(T-t,\Tilde{Y}_t)- ( \Tilde{Y_t} +\frac{2}{h_{T-t}}(m_{T-t}^{(i)}-\Tilde{Y}_t) ) \big)\rangle \dee t \\
    & \quad + 2d^* \dee t + 2\sqrt{2} \langle \Tilde{u}_t^{(i)}, (U_{T-t}^{(i)})^\intercal \dee \Tilde{B}_t\rangle\\
    & = -\frac{4}{h_{T-t}} \|\Tilde{u}_t^{(i)}\|^2 \dee t  - \frac{4}{h_{T-t}}\langle \Tilde{u}_t^{(i)}, (U_{T-t}^{(i)})^\intercal e_A(T-t,\Tilde{Y}_t)\rangle \dee t \\
    & \quad + \mc{O}(1) \dee t + 2\sqrt{2} \langle \Tilde{u}_t^{(i)}, (U_{T-t}^{(i)})^\intercal \dee \Tilde{B}_t\rangle
\end{align*}
which implies
\begin{align*}
    \dee \Tilde{u}_t^{(i)} = -\underbrace{\frac{2}{h_{T-t}}\big( \Tilde{u}_t^{(i)} +  (U_{T-t}^{(i)})^\intercal e_A(T-t,\Tilde{Y}_t) \big)\dee t}_{\text{dominant drift}} + \sqrt{2}\dee \Tilde{B}_t + \mc{O}(1)\dee t.
\end{align*}
Notice that 
\begin{align*}
    \Tilde{Y}_t = m_{T-t}^{(i)} + U_{T-t}^{(i)}\Tilde{u}_t^{(i)} + \underbrace{(I_d - U_{T-t}^{(i)}(U_{T-t}^{(i)})^\intercal) (Y_t-m_{T-t}^{(i)})}_{\text{normal component}}
\end{align*}
Assume the normal component is negligible, then we have 
\begin{align*}
    (U_{T-t}^{(i)})^\intercal e_A(T-t,\Tilde{Y}_t)  \approx (U_{T-t}^{(i)})^\intercal e_A(T-t,m_{T-t}^{(i)} + U_{T-t}^{(i)}\Tilde{u}_t^{(i)}) .
\end{align*}
Then the approximate SDE for $\Tilde{u}_t^{(i)}$ is given by
\begin{align}\label{eq:tangential sde}
    \dee \Tilde{u}_t^{(i)} = -\frac{2}{h_{T-t}}\big( \Tilde{u}_t^{(i)} +  (U_{T-t}^{(i)})^\intercal e_A(T-t,m_{T-t}^{(i)} + U_{T-t}^{(i)}\Tilde{u}_t^{(i)}) \big)\dee t + \sqrt{2}\dee \Tilde{B}_t .
\end{align}
    
\end{remark}

\begin{lemma}\label{lem:application closeness} Under Assumption \ref{assump:data}, if $x\in \mc{B}_{T-t}^{(i)}(\eta_{T-t})$, then
    \begin{align*}
        \| b(T-t,x) - \big( x + \frac{2}{h_{T-t}}( m_{T-t}^{(i)} -x ) \big) \| \le 4R a_{T-t}h_{T-t}^{-1} \theta_{T-t}.
    \end{align*}
\end{lemma} 
\begin{proof}[Proof of Lemma \ref{lem:application closeness}] According to Lemma \ref{lem:explicit formula}, we have
\begin{align*}
    m(T-t,x) - m_{T-t}^{(i)} & =\sum_{j\neq i} \Softmax(-\frac{\| x-m_{T-t} \|^2}{2h_{T-t}})_j a_{T-t}(x_0^{(i)}-x_0^{(i)}) .
\end{align*}
Under Assumption \ref{assump:data} and the definition of $\mc{B}_{T-t}^{(i)}(\eta_{T-t})$, we immediately have $\| m(T-t,x) - m_{T-t}^{(i)} \|\le 2Ra_{T-t}\theta_{T-t}$. Therefore,
\begin{align*}
     &\quad\| b(T-t,x) - \big( x + \frac{2}{h_{T-t}}( m_{T-t}^{(i)} -x ) \big) \| \\
     & = \| x + \frac{2}{h_{T-t}}(m(T-t,x)-x) - \big( x + \frac{2}{h_{T-t}}( m_{T-t}^{(i)} -x )\big) \| \\
    & = \frac{2}{h_{T-t}}\| m(T-t,x) - m_{T-t}^{(i)} \|  \le 4R a_{T-t}h_{T-t}^{-1} \theta_{T-t}.
\end{align*}
\end{proof}
\begin{lemma}\label{lem:property local frame} The local tangent frame $U_t^{(i)}$ satisfies that $\langle u,  (\partial_t U_t^{(i)})^\intercal U_t^{(i)} u \rangle=0$ for any $u\in \mb{R}^{d^*}$.
\end{lemma}
\begin{proof}[Proof of Lemma \ref{lem:property local frame}] Starting from $(U_t^{(i)})^\intercal U_t^{(i)} = I_{d^*}$ and take derivative wrt. $t$ on both side and we get
    \begin{align*}
      (\partial_t U_t^{(i)})^\intercal U_t^{(i)}  = - \big(  (\partial_t U_t^{(i)})^\intercal U_t^{(i)} \big)^\intercal.
    \end{align*}
    Therefore, $ (\partial_t U_t^{(i)})^\intercal U_t^{(i)}$ is skew-symmetric. Hence we proved (1).
\end{proof}
\section{From Inference to Training}\label{append:from inference to training} This section converts the error-bound terms in Theorems~\ref{them:normal contraciton} and~\ref{them:tangent contraciton}, which are expectations along the simulated inference path $\Tilde{Y}_t$, into quantities that depend only on the exact reverse-time OU path $X_t^\leftarrow$. The key tool is a change of measure (Girsanov's theorem). A minor technical issue is that $\Tilde{Y}_0\sim \mc{N}(0,I_d)$ while $X_0^\leftarrow\sim p_T$. To isolate the initialization error, we introduce an auxiliary process $\hat{Y}_t$ that shares the transition kernel with $\Tilde{Y}_t$ but starts from the correct initialization $p_T$.  

\noindent \textbf{Reverse-time processes.} Let $e_A(t,x)\coloneqq m_A(t,x)-m(t,x)$. Consider the following reverse-time processes on $[0,T-\delta]$:
\begin{align*}
    \dee X_t^{\leftarrow} &= \underbrace{( X_t^\leftarrow-\frac{2}{h_{T-t}}X_t^\leftarrow +\frac{2}{h_{T-t}}m(T-t,X_t^\leftarrow) ) }_{b(T-t,X_t^\leftarrow)}\dee t + \sqrt{2}\dee B_t^\leftarrow, \quad X_0^\leftarrow\sim p_T, \\
       \dee \hat{Y}_t &= \underbrace{(\hat{Y}_t -\frac{2}{h_{T-t}}\hat{Y}_t + \frac{2}{h_{T-t}} m_A(T-t,\hat{Y}_t))}_{\coloneqq b_A(T-t,\hat{Y}_t)} \dee t +\sqrt{2} \dee \Tilde{B}_t,  \quad \hat{Y}_0\sim p_T, \\  
    \dee \Tilde{Y}_t &= \underbrace{(\Tilde{Y}_t -\frac{2}{h_{T-t}}\Tilde{Y}_t + \frac{2}{h_{T-t}} m_A(T-t,\Tilde{Y}_t))}_{\coloneqq b_A(T-t,\Tilde{Y}_t)} \dee t +\sqrt{2} \dee \Tilde{B}_t,  \quad \Tilde{Y}_0\sim \mc{N}(0,I_d).
\end{align*}
We denote the path measures of $X_t^\leftarrow$ and $\hat{Y}_t$ respectively by $\mb{P}$ and $\mb{Q}$. Since $X_0\leftarrow$ and $\hat{Y}_0$ have the same initial distribution, we can apply Girsanov's theorem without evolving extra initial-density ratio factor. The remaining gap between $\Tilde{Y}_t$ and $\hat{Y}_t$ is purely from initialization, and it will be controlled since $p_T$ is close to $\mc{N}(0,I_d)$ when $T$ is large.

\begin{assumption}\label{assump:transfer} We assume the following hold:
\begin{itemize}
    \item [(0)] Novikov's condition: $\exp\big( 2\int_0^{T-\delta} \tfrac{\| e_A(T-s,X_t^\leftarrow) \|^2}{h_{T-s}^2} \dee s  \big)<\infty$; 
    \item [(1)] Bound $\chi^2$ along trajectory: $\sup_{t\in [0,T-\delta]} \chi^2(\mb{Q}_t|\mb{P}_t) \le C_{\chi}^2 =\mc{O}(1) $;
    \item [(2)] Normal/tangent posterior mean error satisfies: for any $\dagger\in \{\perp,\parallel\}$, $\mb{E}[\| e_A^\dagger(T-t,X_t^\leftarrow)\|^4]^{\frac{1}{2}}\le C_m \mb{E}[\| e_A^\dagger(T-t,X_t^\leftarrow)\|^2]$ for some $C_m=\mc{O}(1)$;
    \item [(3)] Uniform boundedness of the weighted posterior mean error: $$\sup_{t\in [0,T-\delta]}\sup_{x\in \mc{T}_{T-t}(\rho_{T-t})}\tfrac{w(T-t)\| e_A(T-t,x) \|^2}{h_{T-t}^2}\le \tfrac{C_u}{T-\delta}$$ for some $C_u=\mc{O}(1)$ and decreasing weight $t\mapsto w(t)$. 
\end{itemize}
\end{assumption}
\begin{theorem}\label{thm:inference to training} Under Assumption \ref{assump:transfer}, the normal/tangent-error floors in Theorems \ref{them:normal contraciton} and \ref{them:tangent contraciton} can be estimated by the projected posterior mean matching loss in normal/tangent direction, i.e., 
\begin{align}
    \text{normal-error floor} &\lesssim C_\delta^\perp\int_{\delta}^T \frac{w(t)\mb{E}[\| e_A^\perp(t,X_t)\|^2]}{h_{t}^2}\dee t + C_\delta^\perp(\sqrt{d}+R) e^{-T} +d\delta^{c}, \label{eq:normal error from training}\\
    \text{tangent-error floor} &\lesssim C_\delta^\parallel\int_{\delta}^T \frac{\mb{E}[w(t)\| e_A^\parallel(t,X_t)\|^2]}{h_{t}^2}\dee t + C_\delta^\parallel(\sqrt{d}+R) e^{-T}+d\delta^{\frac{1}{2}}. \label{eq:tangent error from training}
\end{align}   
where $C_\delta^\perp
\coloneqq \delta^c  \big(1 \vee \frac{\delta^{1-c}}{w(\delta)}\big)$ ($c=\lim_{t\to \delta^+} h_t\beta_t$) and $C_\delta^\parallel=\sqrt{\delta}  \big(1 \vee \frac{\sqrt{\delta}}{w(\delta)}\big) $.
\end{theorem}

\begin{proof}[Proof of Theorem \ref{thm:inference to training}]
    Applying Girsanov's Theorem, define $\mc{L}_t\coloneqq \sqrt{2}\int_0^t \tfrac{e_A(T-s,X_s^\leftarrow)}{h_{T-s}}\dee B_s^\leftarrow $ and according to Assumption \ref{assump:transfer}-(0), we have $\mc{E}(\mc{L})_t \coloneqq \exp \big( \mc{L}_t - \frac{1}{2} [\mc{L}]_t  \big)$ is a $\mb{P}$-martingale and
\begin{align*}
    t\mapsto  B_t^\leftarrow - \sqrt{2}\int_0^t \frac{e_A(T-s,X_s^\leftarrow)}{h_{T-s}}\dee s
\end{align*}
is a Brownina motion under $\mc{E}(\mc{L})_{T-\delta}\mb{P}$ and $\mb{Q}=\mc{E}(\mc{L})_{T-\delta}\mb{P}$. Therefore, we can change the path measure from $\mb{Q}$ to $\mb{P}$, hence relating the inference bounds to the training error.

\noindent Recall that with $e_A(t,x)=m_A(t,x)-m(t,x)$ and $e_A^\perp(t,x)=P^\perp_t(x) e_A(t,x)$. When $\beta_t=c/h_t$ for some $c\in [1/2,1)$, the normal-error term related to training in Theorem \ref{them:normal contraciton} is
\begin{align*}
&\quad\int_{\Tilde{t}_{\mathrm{in}}}^{T-\delta} \exp(-2\int_u^{T-\delta} (\beta_{T-s}-1)\dee s ) \frac{\mb{E}[\| e_A^\perp(T-u,\Tilde{Y}_u)\|^2]}{h_{T-u}^2\beta_{T-u}} \dee u \\
&\lesssim \delta^c \int_{\Tilde{t}_{\mathrm{in}}}^{T-\delta} \frac{h_{T-t}^{1-c}}{w(T-t)} \frac{w(T-t)\mb{E}[\| e_A^\perp(T-t,\Tilde{Y}_t)\|^2]}{h_{T-t}^2} \dee t\\
&\lesssim \delta^c  \big(1 \vee \frac{\delta^{1-c}}{w(\delta)}\big) \int_{\Tilde{t}_{\mathrm{in}}}^{T-\delta} \frac{w(T-t)\mb{E}[\| e_A^\perp(T-t,\Tilde{Y}_t)\|^2]}{h_{T-t}^2} \dee t,
\end{align*}
WLOG, assume that $\Tilde{t}_{\mathrm{in}}=0$. According to Girsanov's theorem,
\begin{align*}
    &\quad \int_0^{T-\delta} \frac{\mb{E}[w(T-t)\| e_A^\perp(T-t,\Tilde{Y}_t)\|^2]}{h_{T-t}^2} \dee t \\
    & = \int_0^{T-\delta} \frac{w(T-t)\mb{E}[\| e_A^\perp(T-t,\Tilde{Y}_t)\|^2]-\| e_A^\perp(T-t,\hat{Y}_t)\|^2] }{h_{T-t}^2} \dee t  \\
    &\quad + \int_0^{T-\delta} \frac{w(T-t)\mb{E}[\mc{E}(\mc{L})_t\| e_A^\perp(T-t,X_t^\leftarrow)\|^2]}{h_{T-t}^2}\dee t .
\end{align*}
Notice that
\begin{align*}
    \mb{E}[\mc{E}(\mc{L})_t\| e_A^\perp(T-t,X_t^\leftarrow)\|^2] & = \mb{E}[\| e_A^\perp(T-t,X_t^\leftarrow)\|^2] + \mb{E}[(\mc{E}(\mc{L})_t-1)\| e_A^\perp(T-t,X_t^\leftarrow)\|^2] \\
    & \le \mb{E}[\| e_A^\perp(T-t,X_t^\leftarrow)\|^2] + \mb{E}[(\mc{E}(\mc{L})_t-1)^2]^{\frac{1}{2}}  \mb{E}[\| e_A^\perp(T-t,X_t^\leftarrow)\|^4]^{\frac{1}{2}} \\
    & = \mb{E}[\| e_A^\perp(T-t,X_t^\leftarrow)\|^2] + \sqrt{\chi^2(\mb{Q}_t|\mb{P}_t)}  \mb{E}[\| e_A^\perp(T-t,X_t^\leftarrow)\|^4]^{\frac{1}{2}}
\end{align*}
Under Assumption \ref{assump:transfer}-(1)(2), we have
\begin{align*}
    \int_0^{T-\delta} \frac{w(T-t)\mb{E}[\mc{E}(\mc{L})_t\| e_A^\perp(T-t,X_t^\leftarrow)\|^2]}{h_{T-t}^2}\dee t &\le (1+C_\chi C_m)\int_0^{T-\delta} \frac{w(T-t)\mb{E}[\| e_A^\perp(T-t,X_t^\leftarrow)\|^2]}{h_{T-t}^2}\dee t .
\end{align*}
Meanwhile, under Assumption \ref{assump:transfer}-(3), we have
\begin{align*}
    \int_0^{T-\delta} \frac{w(T-t)\mb{E}[\| e_A^\perp(T-t,\Tilde{Y}_t)\|^2]-\| e_A^\perp(T-t,\hat{Y}_t)\|^2] }{h_{T-t}^2} \dee t \le C_u \TV(\mc{N}(0,I_d),p_T)\lesssim (\sqrt{d}+R)e^{-T}.
\end{align*}
Therefore, the normal-error floor for aligning can be estimated as
\begin{align*}
   &\quad \delta^c  \big(1 \vee \frac{\delta^{1-c}}{w(\delta)}\big) \int_0^{T-\delta} \frac{\mb{E}[\| e_A^\perp(T-u,\Tilde{Y}_u)\|^2]}{h_{T-u}^2} \dee u  \\
   &\lesssim \delta^c  \big(1 \vee \frac{\delta^{1-c}}{w(\delta)}\big)\int_0^{T-\delta} \frac{w(T-t)\mb{E}[\| e_A^\perp(T-t,X_t^\leftarrow)\|^2]}{h_{T-t}^2}\dee t + \delta^c  \big(1 \vee \frac{\delta^{1-c}}{w(\delta)}\big)(\sqrt{d}+R) e^{-T}.
\end{align*}
Similarly, along the tangent direction, the tangent floor is of the form
\begin{align*}
    &\quad \int_{\Tilde{t}_{\mathrm{in}}}^{T-\delta} \exp\big( -\int_u^{T-\delta} \frac{1}{h_{T-s}}-2 \dee s\big) \frac{4}{h_{T-u}}\|(U_{T-u}^{(i)})^\intercal e_A(T-u,\Tilde{Y}_u) \|^2\dee u\\
    &\lesssim \sqrt{\delta}  \big(1 \vee \frac{\sqrt{\delta}}{w(\delta)}\big) \int_{\Tilde{t}_{\mathrm{in}}}^{T-\delta} \frac{w(T-u)\mb{E}[\| (U_{T-u}^{(i)})^\intercal e_A(T-u,\Tilde{Y}_u)\|^2]}{h_{T-u}^2} \dee u,
\end{align*}
where the order estimation follows from the choice of $\beta_t=\Theta(1/h_t)$. When $\Tilde{Y}_t\in \mc{T}_{T-t}(\rho_{T-t})\cap \mc{B}^{(i)}(\theta_{T-t})$ and $t\to T-\delta$, we have $U_{T-t}^{(i)}(U_{T-t}^{(i)})^\intercal\approx P^\parallel_{T-t}(m_{T-t}^{(i)})\approx P^\parallel_{T-t}(\Tilde{Y}_t) $. Hence we can use the following tangent error to reflect the tangent-floor of the sliding phase:
\begin{align*}
    &\quad \sqrt{\delta}  \big(1 \vee \frac{\sqrt{\delta}}{w(\delta)}\big) \int_{\Tilde{t}_{\mathrm{in}}}^{T-\delta} \frac{w(T-t)\mb{E}[\| P^\parallel_{T-t}(\Tilde{Y}_t) e_A(T-t,\Tilde{Y}_t)\|^2]}{h_{T-t}^2} \dee t \\
    &\coloneqq  \sqrt{\delta}  \big(1 \vee \frac{\sqrt{\delta}}{w(\delta)}\big) \int_{\Tilde{t}_{\mathrm{in}}}^{T-\delta} \frac{w(T-t)\mb{E}[\|  e_A^\parallel(T-t,\Tilde{Y}_t)\|^2]}{h_{T-t}^2} \dee t.
\end{align*}
Then following the same change of measure discussion, the tangent-error floor for sliding can be estimated as
\begin{align*}
    &\quad \sqrt{\delta}  \big(1 \vee \frac{\sqrt{\delta}}{w(\delta)}\big)  \int_0^{T-\delta} \frac{w(T-u)\mb{E}[\| e_A^\parallel(T-u,\Tilde{Y}_u)\|^2]}{h_{T-u}^2} \dee u  \\
    &\lesssim \sqrt{\delta}  \big(1 \vee \frac{\sqrt{\delta}}{w(\delta)}\big) \int_0^{T-\delta} \frac{w(T-t)\mb{E}[\| e_A^\parallel(T-t,X_t^\leftarrow)\|^2]}{h_{T-t}^2}\dee t  +  \sqrt{\delta}  \big(1 \vee \frac{\sqrt{\delta}}{w(\delta)}\big) (\sqrt{d}+R) e^{-T}.
\end{align*}
Last, the theorem follows from writing the reverse-time OU $X_t^\leftarrow$ into the forward-time OU $X_t$ and adding the contraction terms depending on $d,R$ in Theorems \ref{them:normal contraciton} and \ref{them:tangent contraciton}.
\end{proof}
\section{Analysis of the Training Process}\label{append:gradient descent}
Recall from \eqref{eq:gd} that
\begin{align*}
    A_{k+1} = A_k(I_p - 2\eta \Tilde{U}) + 2\eta \Tilde{V} ,
\end{align*}
with $U_t = \mb{E}[ \sigma_t(X_t(z))\sigma_t(X_t(z))^\intercal ]\in \mb{R}^{p\times p}$ and $V_t = \mb{E}_z [ a_t X_0 \sigma_t(X_t(z))^\intercal ]\in \mb{R}^{d\times p}$
\begin{align*}
    \Tilde{U} &= \int_\delta^T \frac{w(t)}{h_t^2} \frac{\mb{E}[\sigma_t(X_t(z))\sigma_t(X_t(z))^\intercal]}{p}\dee t ,\quad \Tilde{V}  = \int_\delta^T \frac{w(t)}{h_t^2} \frac{\mb{E}[a_tX_0\sigma_t(X_t(z))^\intercal]}{\sqrt{p}}\dee t.
\end{align*}
$\Tilde{U}$ is the RFNN kernel matrix in \eqref{eq:gd} with rank $r\le p$ and eigen-decomposition $\Tilde{U}=\sum_{j=1}^r \lambda_j u_j u_j^\intercal$ with $\{u_j\}_{j\in [r]}$ being a set of orthonormal vectors in $\mb{R}^p$ and $\lambda_1>\lambda_2>\cdots>\lambda_r>0$.

We first provide initialization-dependent expressions for $A_k$ and the equilibrium of gradient descent.

Apply the iteration recursively and we get
\begin{align*}
    A_k = A_0(I_p-2\eta \Tilde{U})^k + 2\eta \Tilde{V}\sum_{j=0}^{k-1} (I_p-2\eta \Tilde{U})^j. 
\end{align*}
If $2\eta{\|\Tilde{U}\|}\le 1$, the gradient descent is stable. As $k\to\infty$, 
\begin{align*}
    A_k\to A_\infty = A_0(I_p  - \Tilde{U}\Tilde{U}^+) + \Tilde{V}\Tilde{U}^+,
\end{align*}
where $\Tilde{U}^+$ is the pseudo-inverse of $\Tilde{U}$.

Next, we present the error decomposition of $\MM(A_k)$ into architecture-driven error and optimization-driven error.
\begin{proposition}\label{prop:MM decomposition} Let $\{A_k\}$ be the matrix iterates from gradient descent \eqref{eq:gd} with learning rate $\eta<2/\lambda_1$ and initialization $A_0$, then the $\MM$ can be decomposed as
    \begin{align*}
        \MM(A_k) = \Err_{arc} + \Err_{train}(k), \quad  \MM(A_\infty) = \Err_{arc},
    \end{align*}
with 
\begin{align*}
    & \Err_{arc} = \int_\delta^T \frac{w(t)}{h_t^2}\mb{E}[  \| m(t,X_t(z))-\frac{\Tilde{V}\Tilde{U}^+}{\sqrt{p}}\sigma_t(X_t(z))  \|^2 ] \dee t.\\
    & \Err_{train}(k) = \sum_{i=1}^r \lambda_i (1-2\eta \lambda_i)^{2k}\| a_i \|^2,
\end{align*}
and $a_j=(A_0-\Tilde{V}\Tilde{U}^+)u_j\in \mb{R}^d$ for all $j\in [r]$.
\end{proposition}
\begin{proof}[Proof of Proposition \ref{prop:MM decomposition}]
 Based on the dynamics of $\{A_k\}_{k\ge 0}$, we can study the dynamics of the DMM loss $\{\DMM(A_k)\}_{k\ge 0}$. According to \eqref{eq:DMM}, we have
\begin{align}
    \DMM(A) & = \int_\delta^T \frac{w(t)}{h_t^2}\mb{E}[  \| -a_t X_0 + \frac{A}{\sqrt{p}}\sigma_t(X_t(z)) \|^2 ] \dee t \nonumber\\
    & = \int_\delta^T \frac{w(t)}{h_t^2}\mb{E}[  \| -a_t X_0  \|^2 ] \dee t + \int_\delta^T \frac{w(t)}{h_t^2}\mb{E}[  \| \frac{A}{\sqrt{p}}\sigma_t(X_t(z)) \|^2 ] \dee t \nonumber \\
    &\quad - {2\int_\delta^T \frac{w(t)}{h_t^2}\mb{E}[  \langle a_t X_0 , \frac{A}{\sqrt{p}}\sigma_t(X_t(z)) \rangle ] \dee t}\nonumber\\
    & =  C + \Tr(A\Tilde{U}A^\intercal) -2\Tr(A\Tilde{V}^\intercal), \label{eq:DMM decomposition 1}
\end{align}
where the constant $C=\int_\delta^T \frac{w(t)}{h_t^2}\mb{E}[  \| a_t X_0  \|^2 ] \dee t$ is independent to $A$. Notice that $\nabla_A \DMM(A) = 2(A\Tilde{U}-\Tilde{V})$. Therefore, the equilibrium satisfy $A^*$ satisfy $A^* \Tilde{U}=\Tilde{V}$ and they share the same DMM loss value $\DMM(A^*)$. Next, we study the decay of DMM loss along gradient descent by tracking $\DMM(A_k)-\DMM(A^*)$. Define $\Delta A_k=A_k-A^*$, according to \eqref{eq:DMM decomposition 1}, we have
\begin{align*}
    \DMM(A_k) & = C + \Tr\big( (A^*+\Delta A_k)\Tilde{U} (A^*+\Delta A_k)^\intercal\big) - 2\Tr\big( (A^*+\Delta A_k)\Tilde{V}^\intercal \big) \\
    & = C + \Tr(A^*\Tilde{U}(A^*)^\intercal) -2 \Tr(A^* \Tilde{V}^\intercal) + \Tr(\Delta A_k\Tilde{U}\Delta A_k^\intercal) \\
    &\quad + 2\big( \bcancel{\Tr(\Delta A_k \Tilde{U} (A^*)^\intercal) -\Tr(\Delta A_k \Tilde{V}^\intercal)} \big) \\
    & = \DMM(A^*) + \Tr(\Delta A_k\Tilde{U}\Delta A_k^\intercal) \\
    & = \DMM(A^*) + \|\Delta A_k\Tilde{U}^{\frac{1}{2}}\|_F^2 ,
\end{align*}
where the last part in the second equation cancel due to the property of the equilibrium $A^*$. Meanwhile, according to the iteration formula, we can easily get that 
\begin{align*}
    \Delta A_{k+1} & = A_{k+1} - {A^*} = \Delta A_k(I_p-2\eta \Tilde{U}) + 2\eta ( {A^*\Tilde{U}} +\Tilde{V} ) = \Delta A_k(I_p-2\eta \Tilde{U}),
\end{align*}
where the last identity follows from the property of the equilibrium $A^*$. Combined with our equation of $\DMM(A_k)$, we have
\begin{align*}
    \DMM(A_k) = \DMM(A^*) + \|\Delta A_k\Tilde{U}^{\frac{1}{2}}\|_F^2 =  \DMM(A^*) + \|\Delta A_0 (I_p-2\eta \Tilde{U})^k \Tilde{U}^{\frac{1}{2}}\|_F^2.
\end{align*}
Therefore, based on the spectral information of $\Tilde{U}$ and $A_0$, we can represent the DMM loss along gradient descent as follows
\begin{align}\label{eq:loss decomposition}
    \DMM(A_k) = \DMM(A^*) + \sum_{i=1}^r \lambda_i (1-2\eta \lambda_i)^{2k}\| a_i \|^2 .
\end{align}
In order to look at a loss with minimum strictly zero, we need to go back to the (non-denoising) posterior mean matching loss, i.e., $\mc{L}_{\mathrm{MM}}(A)=\int_\delta^T \frac{w(t)}{h_t^2}\mb{E}[  \| - m(t,X_t(z)) + \frac{A}{\sqrt{p}}\sigma_t(X_t(z)) \|^2 ] \dee t $. The DMM loss $\DMM$ can be easily related to $\mc{L}_{\mathrm{MM}}$:
\begin{align*}
    \DMM(A) & = \int_\delta^T \frac{w(t)}{h_t^2}\mb{E}[  \| -a_t X_0 + \frac{A}{\sqrt{p}}\sigma_t(X_t(z)) \|^2 ] \dee t \\
    & = \int_\delta^T \frac{w(t)}{h_t^2}\mb{E}[  \| -m(t,X_t(z)) + \frac{A}{\sqrt{p}}\sigma_t(X_t(z)) \|^2 ] \dee t \\
    &\quad + \int_\delta^T \frac{w(t)}{h_t^2}\mb{E}[  \| -a_t X_0 + m(t,X_t(z)) \|^2 ] \dee t \\
    &\quad + \bcancel{2\int_\delta^T \frac{w(t)}{h_t^2}\mb{E}[  \langle m(t,X_t(z))-a_t X_0 , \frac{A}{\sqrt{p}}\sigma_t(X_t(z)) \rangle ] \dee t} \\
    & = \mc{L}_{\mathrm{MM}}(A) + \int_\delta^T \frac{w(t)}{h_t^2}\mb{E}[  \| -a_t X_0 + m(t,X_t(z)) \|^2 ] \dee t,
\end{align*}
where the cross term is canceled based on definition of $m(t,x)=\mb{E}[a_t X_0|X_t=x]$ and the tower property. Therefore, we can derive a decomposition of $\mc{L}_{\mathrm{MM}}(A_k)$ along the gradient descent dynamics:
\begin{align}\label{eq:error decomposition}
    &\quad\mc{L}_{\mathrm{MM}}(A_k) \\
    & = \DMM(A_k) - \int_\delta^T \frac{w(t)}{h_t^2}\mb{E}[  \| -a_t X_0 + m(t,X_t(z)) \|^2 ] \dee t \nonumber\\
    & =  \sum_{i=1}^r \lambda_i (1-2\eta \lambda_i)^{2k}\| a_i \|^2 +  \DMM(A^*) - \int_\delta^T \frac{w(t)}{h_t^2}\mb{E}[  \| -a_t X_0 + m(t,X_t(z)) \|^2 ] \dee t \nonumber \\
    &  =  \sum_{i=1}^r \lambda_i (1-2\eta \lambda_i)^{2k}\| a_i \|^2 +  \int_\delta^T \frac{w(t)}{h_t^2}\mb{E}[  \langle a_t X_0, a_t X_0 -\frac{A^*}{\sqrt{p}}\sigma_t(X_t(z))  \rangle ] \dee t \nonumber\\
    &\quad - \int_\delta^T \frac{w(t)}{h_t^2}\mb{E}[  \| -a_t X_0 + m(t,X_t(z)) \|^2 ] \dee t \nonumber\\
    &  =  \sum_{i=1}^r \lambda_i (1-2\eta \lambda_i)^{2k}\| a_i \|^2 +  \int_\delta^T \frac{w(t)}{h_t^2}\mb{E}[  \langle m(t,X_t(z)), m(t,X_t(z))-\frac{A^*}{\sqrt{p}}\sigma_t(X_t(z))  \rangle ] \dee t \nonumber\\
    &  =  \sum_{i=1}^r \lambda_i (1-2\eta \lambda_i)^{2k}\| a_i \|^2 +  \int_\delta^T \frac{w(t)}{h_t^2}\mb{E}[  \| m(t,X_t(z))-\frac{A^*}{\sqrt{p}}\sigma_t(X_t(z))  \|^2 ] \dee t, 
\end{align}
where the second last property follows from the tower property and the last identity follows from the property of $A^*$ and the tower property. Therefore, Equation \eqref{eq:error decomposition} decompose the posterior mean error into two parts which induce different type of implicit bias.
\begin{itemize}
    \item [(1)] Architecture Implicit Bias: the term $\int_\delta^T \frac{w(t)}{h_t^2}\mb{E}[  \| m(t,X_t(z))-\frac{A^*}{\sqrt{p}}\sigma_t(X_t(z))  \|^2 ] \dee t$ orients from the inadequate representation ability of the RFNN with finite $p$. Even though the equilibrium $A^*$ can be any matrix of the form $A_0(I_p  - \Tilde{U}\Tilde{U}^+) + \Tilde{V}\Tilde{U}^+$ (depending on initialization), the value of the architecture implicit bias stays the same. Therefore, we can pick a special $A^*=\Tilde{V}\Tilde{U}^+$ to represent it, i.e.,
    \begin{align*}
        \Err_{arc} = \int_\delta^T \frac{w(t)}{h_t^2}\mb{E}[  \| m(t,X_t(z))-\frac{\Tilde{V}\Tilde{U}^+}{\sqrt{p}}\sigma_t(X_t(z))  \|^2 ] \dee t.
    \end{align*}
    \item [(2)] Training Dynamical Implicit Bias: the term $\sum_{i=1}^r \lambda_i (1-2\eta \lambda_i)^{2k}\| a_i \|^2$ is produced by running the optimization algorithm for finite time. If we can ideally run the gradient descent for infinite many step, this bias will vanish. We represent the training dynamical bias as follows
    \begin{align*}
        \Err_{train}(k) = \sum_{i=1}^r \lambda_i (1-2\eta \lambda_i)^{2k}\| a_i \|^2.
    \end{align*} 
\end{itemize}   
\end{proof}
Based on Proposition \ref{prop:MM decomposition}, we can proceed to prove Theorem \ref{thm:directional decomposition}.
\begin{proof}[Proof of Theorem \ref{thm:directional decomposition}]
    The decomposition of training dynamical/architecture implicit bias requires to repeat our analysis to the total DMM loss.
\begin{align*}
    \DMM^\perp(A) &\coloneqq \int_\delta^T \frac{w(t)}{h_t^2}\mb{E}[  \|  P_t^\perp(X_t(z))\big( -a_t X_0 + \frac{A}{\sqrt{p}}\sigma_t(X_t(z)) \big) \|^2 ] \dee t \\
     & = \underbrace{\int_\delta^T \frac{w(t)}{h_t^2}\mb{E}[  \| P_t^\perp(X_t(z)) (-a_t X_0)  \|^2 ] \dee t}_{\coloneqq C^\perp} + \underbrace{\int_\delta^T \frac{w(t)}{h_t^2}\mb{E}[  \| P_t^\perp(X_t(z)) (\frac{A}{\sqrt{p}}\sigma_t(X_t(z))) \|^2 ] \dee t}_{\text{quadratic in }A} \\
     &\quad - \underbrace{{2\int_\delta^T \frac{w(t)}{h_t^2}\mb{E}[  \langle P_t^\perp(X_t(z))(a_t X_0) , \frac{A}{\sqrt{p}}\sigma_t(X_t(z)) \rangle ] \dee t}}_{\text{linear in }A} .
\end{align*}
Then we have
\begin{align*}
    \DMM^\perp(A_k)& = \DMM^\perp (A^*+\Delta A_k) \\
    &= C^\perp + \int_\delta^T \frac{w(t)}{h_t^2}\mb{E}[  \| P_t^\perp(X_t(z)) (\frac{A^* + \Delta A_k }{\sqrt{p}}\sigma_t(X_t(z))) \|^2 ] \dee t \\
    & \quad - 2\int_\delta^T \frac{w(t)}{h_t^2}\mb{E}[  \langle P_t^\perp(X_t(z))(a_t X_0) , \frac{A^*+\Delta A_k}{\sqrt{p}}\sigma_t(X_t(z)) \rangle ] \dee t \\
    & = \DMM^\perp(A^*) + \int_\delta^T \frac{w(t)}{h_t^2}\mb{E}[  \| P_t^\perp(X_t(z)) (\frac{\Delta A_k }{\sqrt{p}}\sigma_t(X_t(z))) \|^2 ] \dee t \\
    &\quad + 2\int_\delta^T \frac{w(t)}{h_t^2}\mb{E}[  \langle P_t^\perp(X_t(z))(\frac{A^*}{\sqrt{p}}\sigma_t(X_t(z)) -a_tX_0 ) , \frac{\Delta A_k}{\sqrt{p}}\sigma_t(X_t(z)) \rangle ] \dee t .
\end{align*}
Therefore, the posterior mean loss is
\begin{align*}
    \mc{L}_{\mathrm{MM}}^\perp(A_k) & = \DMM^\perp(A_k) - \int_\delta^T \frac{w(t)}{h_t^2}\mb{E}[  \| P_t^\perp(X_t(z)) \big(-a_t X_0 + m(t,X_t(z)) \big) \|^2 ] \dee t \\
    & = \DMM^\perp(A^*) + \int_\delta^T \frac{w(t)}{h_t^2}\mb{E}[  \| P_t^\perp(X_t(z)) (\frac{\Delta A_k }{\sqrt{p}}\sigma_t(X_t(z))) \|^2 ] \dee t \\
    &\quad + 2\int_\delta^T \frac{w(t)}{h_t^2}\mb{E}[  \langle P_t^\perp(X_t(z))(\frac{A^*}{\sqrt{p}}\sigma_t(X_t(z)) -a_tX_0 ) , \frac{\Delta A_k}{\sqrt{p}}\sigma_t(X_t(z)) \rangle ] \dee t \\
    &\quad - \int_\delta^T \frac{w(t)}{h_t^2}\mb{E}[  \| P_t^\perp(X_t(z)) \big(-a_t X_0 + m(t,X_t(z)) \big) \|^2 ] \dee t \\
    & = \int_\delta^T \frac{w(t)}{h_t^2}\mb{E}[  \| P_t^\perp(X_t(z)) \big(m(t,X_t(z))-\frac{A^*}{\sqrt{p}}\sigma_t(X_t(z)) \big) \|^2 ] \dee t \\
    &\quad + \int_\delta^T \frac{w(t)}{h_t^2}\mb{E}[  \| P_t^\perp(X_t(z)) (\frac{\Delta A_k }{\sqrt{p}}\sigma_t(X_t(z))) \|^2 ] \dee t \\
    &\quad + 2\int_\delta^T \frac{w(t)}{h_t^2}\mb{E}[  \langle P_t^\perp(X_t(z))(\frac{A^*}{\sqrt{p}}\sigma_t(X_t(z)) -m(t,X_t(z) ) , \frac{\Delta A_k}{\sqrt{p}}\sigma_t(X_t(z)) \rangle ] \dee t.
\end{align*}
Similarly, along the tangent directions, we have
\begin{align*}
    \mc{L}_{\mathrm{MM}}^\parallel(A_k) & = \int_\delta^T \frac{w(t)}{h_t^2}\mb{E}[  \| P_t^\parallel(X_t(z)) \big(m(t,X_t(z))-\frac{A^*}{\sqrt{p}}\sigma_t(X_t(z)) \big) \|^2 ] \dee t \\
    &\quad + \int_\delta^T \frac{w(t)}{h_t^2}\mb{E}[  \| P_t^\parallel(X_t(z)) (\frac{\Delta A_k }{\sqrt{p}}\sigma_t(X_t(z))) \|^2 ] \dee t \\
    &\quad + 2\int_\delta^T \frac{w(t)}{h_t^2}\mb{E}[  \langle P_t^\parallel(X_t(z))(\frac{A^*}{\sqrt{p}}\sigma_t(X_t(z)) -m(t,X_t(z) ) , \frac{\Delta A_k}{\sqrt{p}}\sigma_t(X_t(z)) \rangle ] \dee t.
\end{align*}
Then along the normal and tangent directions, the architecture implicit bias are 
\begin{align*}
    &\Err_{arc}^\perp  = \int_\delta^T \frac{w(t)}{h_t^2}\mb{E}[  \| P_t^\perp(X_t(z))\big(  m(t,X_t(z))-\frac{A^*}{\sqrt{p}}\sigma_t(X_t(z))\big)  \|^2 ] \dee t,\\
    &\Err_{arc}^\parallel  = \int_\delta^T \frac{w(t)}{h_t^2}\mb{E}[  \| P_t^\parallel(X_t(z))\big(  m(t,X_t(z))-\frac{A^*}{\sqrt{p}}\sigma_t(X_t(z))\big)  \|^2 ] \dee t.
\end{align*}
The training dynamical implicit bias along the tangent directions are
\begin{align*}
    \Err_{train}^\perp(k) & = \int_\delta^T \frac{w(t)}{h_t^2}\mb{E}[  \| P_t^\perp(X_t(z)) (\frac{\Delta A_k }{\sqrt{p}}\sigma_t(X_t(z))) \|^2 ] \dee t \\
    &\quad + 2\int_\delta^T \frac{w(t)}{h_t^2}\mb{E}[  \langle P_t^\perp(X_t(z))(\frac{A^*}{\sqrt{p}}\sigma_t(X_t(z)) -m(t,X_t(z) ) , \frac{\Delta A_k}{\sqrt{p}}\sigma_t(X_t(z)) \rangle ] \dee t, \\
    \Err_{train}^\parallel(k) & = \int_\delta^T \frac{w(t)}{h_t^2}\mb{E}[  \| P_t^\parallel(X_t(z)) (\frac{\Delta A_k }{\sqrt{p}}\sigma_t(X_t(z))) \|^2 ] \dee t \\
    &\quad + 2\int_\delta^T \frac{w(t)}{h_t^2}\mb{E}[  \langle P_t^\parallel(X_t(z))(\frac{A^*}{\sqrt{p}}\sigma_t(X_t(z)) -m(t,X_t(z) ) , \frac{\Delta A_k}{\sqrt{p}}\sigma_t(X_t(z)) \rangle ] \dee t.
\end{align*}
Again, since $\Delta A_k = \Delta A_0 (1-2\eta \Tilde{U})^k$, we have $\tfrac{\Delta A_k}{\sqrt{p}}\sigma_t(X_t(z)) = \sum_{j=1}^r (1-2\eta \lambda_j)^k \tfrac{u_j^\intercal \sigma_t(X_t(z))}{\sqrt{p}} a_j  \coloneqq \sum_{j=1}^r (1-2\eta \lambda_j)^k \sigma_{t,u_j} a_j $. Then we have
\begin{align*}
    & \int_\delta^T \frac{w(t)}{h_t^2}\mb{E}[  \| P_t^\perp(X_t(z)) (\frac{\Delta A_k }{\sqrt{p}}\sigma_t(X_t(z))) \|^2 ] \dee t \\
    = & \sum_{j,l=1}^r (1-2\eta \lambda_j)^k (1-2\eta \lambda_l)^k a_j^\intercal \int_\delta^T \frac{w(t)}{h_t^2} \mb{E}[ \sigma_{t,u_j}\sigma_{t,u_l} P_t^\perp(X_t(z))  ]a_l \\
    \coloneqq  & \sum_{j,l=1}^r (1-2\eta \lambda_j)^k (1-2\eta \lambda_l)^k a_j^\intercal P^\perp_{jl} a_l \\
    &2\int_\delta^T \frac{w(t)}{h_t^2}\mb{E}[  \langle P_t^\perp(X_t(z))(\frac{A^*}{\sqrt{p}}\sigma_t(X_t(z)) -m(t,X_t(z) ) , \frac{\Delta A_k}{\sqrt{p}}\sigma_t(X_t(z)) \rangle ] \dee t\\
 = & 2\sum_{j=1}^r (1-2\eta \lambda_j)^k a_j^\intercal \int_\delta^T \frac{w(t)}{h_t^2} \mb{E}[\sigma_{t,u_j} P_t^\perp(X_t(z)) ( \frac{A^*}{\sqrt{p}}\sigma_t(X_t(z)) -m(t,X_t(z) ) ] \dee t\\
 \coloneqq & 2\sum_{j=1}^r (1-2\eta \lambda_j)^k a_j^\intercal b_j^\perp.
\end{align*}
We can work on the tangent directions similarly. Therefore, we have
\begin{align*}
     \Err_{train}^\perp(k) & = \sum_{j,l=1}^r (1-2\eta \lambda_j)^k (1-2\eta \lambda_l)^k a_j^\intercal P^\perp_{jl} a_l + 2\sum_{j=1}^r (1-2\eta \lambda_j)^k a_j^\intercal b_j^\perp,\\
     \Err_{train}^\parallel(k) & = \sum_{j,l=1}^r (1-2\eta \lambda_j)^k (1-2\eta \lambda_l)^k a_j^\intercal P^\parallel_{jl} a_l + 2\sum_{j=1}^r (1-2\eta \lambda_j)^k a_j^\intercal b_j^\parallel,
\end{align*}
where 
\begin{align}
    &P^\perp_{jl} =\int_\delta^T \frac{w(t)}{h_t^2} \mb{E}[ \sigma_{t,u_j}\sigma_{t,u_l} P_t^\perp(X_t(z))  ] \dee t, \quad b_j^\perp = \int_\delta^T \frac{w(t)}{h_t^2} \mb{E}[\sigma_{t,u_j} P_t^\perp(X_t(z)) ( \frac{A^*}{\sqrt{p}}\sigma_t(X_t(z)) -m(t,X_t(z) ) ] \dee t, \nonumber \\
    &P^\parallel_{jl} =\int_\delta^T \frac{w(t)}{h_t^2} \mb{E}[ \sigma_{t,u_j}\sigma_{t,u_l} P_t^\parallel(X_t(z))  ] \dee t, \quad b_j^\parallel = \int_\delta^T \frac{w(t)}{h_t^2} \mb{E}[\sigma_{t,u_j} P_t^\parallel(X_t(z)) ( \frac{A^*}{\sqrt{p}}\sigma_t(X_t(z)) -m(t,X_t(z) ) ] \dee t \label{eq:training dynamical error decomp}
\end{align}
and recall $\sigma_{t,u_j}=\tfrac{u_j^\intercal \sigma_t(X_t(z))}{\sqrt{p}}$ for all $j\in [r]$. It is worth mentioning that
\begin{align*}
    P^\perp_{jl} + P^\parallel_{jl} = \int_\delta^T \frac{w(t)}{h_t^2} \mb{E}[ \sigma_{t,u_j}\sigma_{t,u_l}  ] \dee t = \lambda_j 1_{j=l},\quad b_j^\perp + b_l^\parallel = 0.
\end{align*}
Hence $\Err_{train}(k)=\Err_{train}^\perp(k)+\Err_{train}^\parallel(k)$.
\end{proof}

\section{Experimental Details}\label{append:experimental detail}
In this section, we list the specific model architectures and hyperparameters that were not included in the main text.
\subsection{Two Points in 2D Plane}\label{sec: settingof2points}
\textbf{Model Architectures} Our RFNN model strictly follows the architecture defined in Section \ref{sec:prelim}, with a model width of $p=2000$ and a temporal embedding dimension of $K_t=128$.

\noindent\textbf{Training and Loss Computation} We optimize the model via gradient descent for $10^6$ epochs with a learning rate of $0.005$. The gradient descent is performed using the full-batch SGD optimizer in PyTorch and loss function \eqref{eq:DMM}. We numerically evaluate the integral in the loss function \eqref{eq:DMM} using the trapezoidal rule with 2,000 discretization points.

\noindent\textbf{SDE and Sampling Details} We model the diffusion process using a stochastic differential equation (SDE) with a drift coefficient of $f(x,t) = -x$ and a constant diffusion coefficient of $g(t) = \sqrt{2}$. For generation, we employ the Euler-Maruyama sampler with $N=1000$ discrete time steps. The time schedule follows a geometric progression decaying from $T=10$ down to $\delta = t_{min}=10^{-3}$. Specifically, the time points are defined as $t_i = T \cdot (t_{min}/T)^{i/N}$ for $i=0, \dots, N$.

\noindent\textbf{Computational Resources} \noindent\textbf{Computational Resources} All experiments were conducted on a single NVIDIA RTX 4090 GPU.\footnote{For completeness, the machine used in our experiments is equipped with an NVIDIA RTX 4090 GPU with 24.0 GB of memory, a 14-core AMD EPYC 7453 CPU, and 60.1 GB of system memory. Since nearly all model training and evaluation computations are performed on the GPU, the CPU is not the computational bottleneck in our experiments. The relatively long  time mainly comes from the large number of training epochs rather than from heavy per-epoch computation.} For the RFNN-based experiments, the training is more computationally intensive, as the model is trained for $10^6$ epochs; completing the full set of RFNN experiments takes approximately 1.5 hours.

\subsection{More Points in 2D Plane}
We use the RFNN model with the same architecture, training, and sampling settings as described in Section~\ref{sec: settingof2points}. The only differences are the training epochs: for the four-point case (Figure \ref{fig:4pointsquare}), we train for 500,000 epochs with $lr = 0.005$; for the three-point case (Figure \ref{fig:3point3005}), we train for 300,000 epochs with $lr = 0.005$. These two experiments were still conducted on a single NVIDIA RTX 4090 GPU. Completing the experiments takes approximately 1 hour. 
\subsection{MNIST}\label{sec:MNIST training detail}
\textbf{Model Architectures} We utilize Multi-Layer Perceptrons (MLPs) for all components of our framework.

\begin{itemize}
    \item \textbf{VAE Architecture:} The encoder maps the flattened input image ($\mathbb{R}^{784}$) to the latent space through two hidden layers with 400 and 200 units, respectively. It uses ReLU activations and outputs the mean and log-variance parameters. The decoder mirrors this structure (latent $\to$ 200 $\to$ 400 $\to$ 784) and uses a Sigmoid activation at the final layer to ensure pixel values lie within $[0, 1]$.
    \item \textbf{Latent Score Network:} The score model is a time-conditioned MLP. The time step $t$ is first mapped to a 128-dimensional feature vector using Sinusoidal Embedding. This embedding is concatenated with the input vector and passed through three fully connected layers (sizes: input $\to$ 256 $\to$ 256 $\to$ latent dimension) with ReLU activations to estimate the score function.
\end{itemize}
\textbf{Training Configuration}
All models are trained on a single GPU using the Adam optimizer with a batch size of 1024.
\begin{itemize}
    \item \textbf{VAE Training:} The VAE is trained for 200 epochs with a learning rate of $10^{-3}$. It minimizes the standard Evidence Lower Bound loss.
    \item \textbf{Score Model Training:} The latent diffusion model is trained for 300 epochs with a learning rate of $10^{-4}$. The loss function follows the denoising score matching objective \eqref{eq:DMM}. The weight schedule in the loss function is set to a constant $w(t)=1$.
\end{itemize}
\textbf{SDE and Sampling Details} We model the diffusion process using a stochastic differential equation (SDE) with a drift coefficient of $f(x,t) = -x$ and a constant diffusion coefficient of $g(t) = \sqrt{2}$. For generation, we employ the Euler-Maruyama sampler with $N=1000$ discrete time steps. The time schedule follows a geometric progression decaying from $T=10$ down to $\delta = t_{min}=0.005$. Specifically, the time points are defined as $t_i = T \cdot (t_{min}/T)^{i/N}$ for $i=0, \dots, N$.

\noindent\textbf{Computational Resources} All experiments were conducted on a single NVIDIA RTX 4090 GPU. The training stage is relatively lightweight: training the VAE and the latent score model together takes less than 10 minutes under our experimental setup. The sampling and geometric evaluation stage is more time-consuming. Completing this evaluation for all four  dimensions $d^* = 3,5,8,12$ takes approximately 40 minutes.

\subsection{Projection onto Log-density Ridge Sets} \label{app:newton_distance}
During the reverse sampling process, given a sample $x_0 \in \mathbb{R}^d$ generated by the diffusion model at a specific time step, our objective is to find its exact projection $y^*$ onto the ridge manifold $\mathcal{R}_{d^*}(p;\beta)$. This naturally formulates as a constrained non-linear optimization problem:
\begin{equation}
\begin{aligned}
\min_{y \in \mathbb{R}^d} \quad & \frac{1}{2} | y - x_0 |^2 \\
\text{s.t.} \quad & F(y) \coloneqq E(y)^\intercal \nabla \log p(y) = 0, \\
& \lambda_{d^*+1}(y) \le -\beta
\end{aligned}
\end{equation}
where $F(y) \in \mathbb{R}^{d-d^*}$ represents the residual projection of the score function onto the normal space. To solve this projection problem, we introduce the Lagrange multiplier $\nu \in \mathbb{R}^{d-d^*}$ for the equality constraint $F(y)=0$. The unconstrained Lagrangian function is given by:
\begin{equation}
\mathcal{L}(y, \nu) = \frac{1}{2} | y - x_0 |^2 + \nu^\intercal F(y)
\end{equation}
According to the KKT conditions, the exact projection point must satisfy the following stationarity conditions:
\begin{equation}
\begin{cases}
\nabla_y \mathcal{L}(y, \nu) = (y - x_0) + J(y)^\intercal \nu = 0 \\
F(y) = 0
\end{cases}
\end{equation}
where $J(y) \coloneqq \nabla_y F(y) \in \mathbb{R}^{(d-d^*) \times d}$ is the Jacobian matrix of the normal residual. In high-dimensional spaces, approximating $J(y)$ via finite differences introduces severe numerical instability. Instead, we construct the Jacobian analytically. Applying the product rule to $F(y)$, we obtain:\begin{equation}J(y) = E(y)^\intercal \nabla^2 \log p(y) + \big(\nabla_y E(y)\big)^\intercal \nabla \log p(y)\end{equation}Let $H(y) \coloneqq \nabla^2 \log p(y)$ denote the Hessian matrix. Since the column vectors of $E(y)$ are precisely the eigenvectors of $H(y)$ spanning the normal space, the first term simplifies directly to $\Lambda_{\perp}(y) E(y)^\intercal$, where $\Lambda_{\perp}(y) = \text{diag}(\lambda_{d^*+1}, \dots, \lambda_d)$.For the second term, which involves the derivative of the invariant subspace, we use matrix perturbation theory. Let $T(y) \in \mathbb{R}^{d \times d^*}$ be the orthogonal basis of the tangent space, and $\Lambda_{\parallel}(y)$ be the corresponding diagonal matrix of the tangent eigenvalues. The normal basis $E(y)$ satisfies the eigenvalue equation $H(y) E(y) = E(y) \Lambda_{\perp}(y)$. Differentiating both sides with respect to the coordinate component $y^{(\ell)}$ yields:$$\frac{\partial H(y)}{\partial y^{(\ell)}} E(y) + H(y) \frac{\partial E(y)}{\partial y^{(\ell)}} = \frac{\partial E(y)}{\partial y^{(\ell)}} \Lambda_{\perp}(y) + E(y) \frac{\partial \Lambda_{\perp}(y)}{\partial y^{(\ell)}}$$To isolate the variation of the normal space that alters the manifold's geometry—namely, its "tilt" towards the tangent space—we pre-multiply both sides by $T(y)^\intercal$. Utilizing the orthogonality $T(y)^\intercal E(y) = 0$ and the symmetry $T(y)^\intercal H(y) = \Lambda_{\parallel}(y) T(y)^\intercal$, the equation simplifies to a standard Sylvester equation:$$\Lambda_{\parallel}(y) \left( T(y)^\intercal \frac{\partial E(y)}{\partial y^{(\ell)}} \right) - \left( T(y)^\intercal \frac{\partial E(y)}{\partial y^{(\ell)}} \right) \Lambda_{\perp}(y) = - T(y)^\intercal \frac{\partial H(y)}{\partial y^{(\ell)}} E(y)$$Given the strict eigengap $\lambda_{d^*}(y) > \lambda_{d^*+1}(y)$, the spectra of $\Lambda_{\parallel}(y)$ and $\Lambda_{\perp}(y)$ are disjoint, guaranteeing a unique solution. Using the Kronecker sum $\oplus$, the projection of the derivative onto the tangent space can be solved analytically. By mapping this tangent projection back to the ambient space, the correction term induced by the partial derivative of $E(y)$ can be explicitly expressed as:$$\frac{\partial E(y)}{\partial y^{(\ell)}} = - T(y) \big(\Lambda_{\parallel}(y) \oplus (-\Lambda_{\perp}(y))\big)^{-1} T(y)^\intercal \frac{\partial H(y)}{\partial y^{(\ell)}} E(y)$$By aggregating this correction term across all dimensions $\ell \in \{1,\dots,d\}$, we compute the exact Jacobian matrix $J(y)$. This analytical construction effectively avoids the systematic biases inherent in second-order root-finding processes when tracking the "tilt" of the normal space.

Because $J(y)$ can become ill-conditioned when the initial point is far from the manifold, we employ a Levenberg-Marquardt (LM) approach with a damping parameter $\sigma > 0$ to solve the KKT system. At the $k$-th iteration, given the current point $y_k$, we linearize the constraint as $F(y_k + s) \approx F(y_k) + J(y_k)s$. We then solve the following regularized linear system to obtain the dual variable $\nu$ and the primal step $s$:
\begin{equation}
\begin{aligned}
\big(J(y_k) J(y_k)^\intercal \big) \nu &= (1 + \sigma) F(y_k) - J(y_k)(y_k - x_0) \\
s &= -\frac{(y_k - x_0) + J(y_k)^\intercal \nu}{1 + \sigma}
\end{aligned}
\end{equation}

To ensure that the iteration sequence strictly converges to a point within the valid ridge manifold boundaries (i.e., satisfying the strict curvature condition $\lambda_{d^*+1}(y) \le -\beta$ defined in Definition \ref{def:ridge}), we design a merit function with boundary penalties to guide the line search:\begin{equation}M(y) = \frac{1}{2} | y - x_0 |^2 + \frac{c}{2} | F(y) |^2 + \gamma \max\big(0, \lambda_{d^*+1}(y) + \beta\big)^2\end{equation}where $c > 0$ is the penalty weight for the equality constraint, and $\gamma > 0$ is the penalty factor for curvature violations. During each iteration, if the trial step $y_{\text{trial}} = y_k + s$ yields a decrease in the merit function ($M(y_{\text{trial}}) < M(y_k)$), the step is accepted, and the damping parameter $\sigma$ is reduced. Otherwise, the step is rejected, $\sigma$ is increased, and the LM system is solved again. This procedure is repeated until the convergence criteria $\|F(y)\| \le \epsilon_{\text{tol}}$ and $\lambda_{d^*+1}(y) \le -\beta$ are simultaneously met.

The complete pseudocode can be found in Algorithm \ref{alg:ridge_projection}. In our experiments, to find the exact projection of generated samples onto the log-density ridge manifold during the reverse sampling phase, we set the intrinsic manifold dimension to $d^*=3,5,8,12$ and enforce a strict normal curvature margin of $\beta=10^{-3}$. For solving the KKT system, we employ Levenberg-Marquardt optimization algorithm with a maximum of $30$ iterations and a convergence tolerance of $\epsilon_{\text{tol}}=10^{-6}$. A diagonal perturbation of $\epsilon_{\text{reg}}=10^{-6}$ is introduced during the Jacobian inversion to ensure numerical stability, and the penalty weights for both the equality constraint and curvature violation in the merit function are set to $c=\gamma=100.0$. Throughout the $1000$-step diffusion generative trajectory, we dynamically monitor the geometric distance to the manifold every $100$ time steps using $200$ samples.

Since the projection algorithm relies on the local spectral separation of the Hessian $\nabla^2 \log p_t(x)$, its reliability depends on the presence of a sufficiently large eigengap between the tangent and normal eigenspaces. This requirement is especially important in the early high-noise stages of reverse sampling, where the local geometry may still be close to isotropic Gaussian noise and the tangent-normal decomposition is poorly separated. We therefore explicitly track the eigengap and projection failures along the sampling trajectory, and analyze this stability issue in Appendix~\ref{sec:stability}.

\newpage
\begin{algorithm}[tb]
\caption{Projection onto Log-density Ridge Sets via Regularized LM}
\label{alg:ridge_projection}
\begin{algorithmic}[1]
\State {\bfseries Input:} Initial sample $x_0 \in \mathbb{R}^d$, score function $\nabla \log p(\cdot)$, intrinsic dimension $d^*$.
\State {\bfseries Parameters:} Curvature margin $\beta=10^{-3}$, tolerance $\epsilon_{\text{tol}}=10^{-6}$, max iterations $N_{\max}=30$, initial damping $\sigma=10^{-3}$, KKT regularization $\epsilon_{\text{reg}}=10^{-6}$, merit weights $c=100, \gamma=100$.
\State {\bfseries Initialize:} $y \leftarrow x_0$
\For{$k = 0$ {\bfseries to} $N_{\max}-1$}
    \State Compute score $g \leftarrow \nabla \log p(y)$ and Hessian $H \leftarrow \nabla^2 \log p(y)$
    \State Compute eigendecomposition of $H$ to yield tangent basis $T \in \mathbb{R}^{d \times d^*}$, normal basis $E \in \mathbb{R}^{d \times (d-d^*)}$, and eigenvalue matrices $\Lambda_{\parallel}, \Lambda_{\perp}$
    \State Evaluate normal residual $F(y) \leftarrow E^\intercal g$
    \If{$\|F(y)\| \le \epsilon_{\text{tol}}$ {\bfseries and} $\lambda_{d^*+1} \le -\beta$}
        \State \Return $y$ \Comment{Converged to exact ridge manifold}
    \EndIf
    
    \State \Comment{\textit{Construct analytical Jacobian $J(y) \in \mathbb{R}^{(d-d^*) \times d}$ via matrix perturbation}}
    \For{$\ell = 1$ {\bfseries to} $d$}
        \State Compute normal basis derivative: $\frac{\partial E}{\partial y^{(\ell)}} \leftarrow - T \big(\Lambda_{\parallel} \oplus (-\Lambda_{\perp})\big)^{-1} T^\intercal \frac{\partial H}{\partial y^{(\ell)}} E$
    \EndFor
    \State Assemble Jacobian via column aggregation: $J(y) \leftarrow \Lambda_{\perp} E^\intercal + \left[ \left(\frac{\partial E}{\partial y^{(1)}}\right)^\intercal g, \dots, \left(\frac{\partial E}{\partial y^{(d)}}\right)^\intercal g \right]$
    
    \State \Comment{\textit{Solve the regularized KKT system for dual variable $\nu$ and primal step $s$}}
    \State Update dual variable: $\nu \leftarrow (J(y) J(y)^\intercal + \epsilon_{\text{reg}} I)^{-1} \big((1+\sigma)F(y) - J(y)(y-x_0)\big)$
    \State Compute primal step: $s \leftarrow -\frac{1}{1+\sigma} \big((y - x_0) + J(y)^\intercal \nu\big)$
    \State Define trial step: $y_{\text{trial}} \leftarrow y + s$
    
    \State \Comment{\textit{Evaluate merit function to ensure constraints and boundary penalties}}
    \State Evaluate trial merit: $M(z) = \frac{1}{2}\|z-x_0\|^2 + \frac{c}{2}\|F(z)\|^2 + \gamma \max\big(0, \lambda_{d^*+1}(z) + \beta\big)^2$
    \If{$M(y_{\text{trial}}) < M(y)$}
        \State Accept step and decrease damping: $y \leftarrow y_{\text{trial}}$, $\sigma \leftarrow \max(\sigma / 2, \epsilon_{\text{reg}})$
    \Else
        \State Reject step and increase damping: $\sigma \leftarrow 10 \cdot \sigma$
    \EndIf
\EndFor
\State \Return $y$ \Comment{Return best approximation if max iterations reached}
\end{algorithmic}
\end{algorithm}

\clearpage
\subsection{Stability}\label{sec:stability}
In the tables below, we track the dynamic geometric relationship between the generated samples and the local ridge manifold during the reverse sampling process. The columns are defined as follows: \textbf{Step} and $t$ represent the iteration step and the corresponding noise scale. \textbf{Mean Dist.} is the average geometric projection distance from the samples to the ridge manifold. \textbf{Curv. OK} counts the samples that strictly meet the negative curvature condition ($\lambda_{d^*+1} \le -\beta$). \textbf{Mean Resid.} shows the initial normal score residual norm ($\|F(x_t)\|$) before projection. \textbf{Mean Gap} is the average eigengap between the tangent and normal spaces, which measures how well the manifold structure is separated. \textbf{Failures} counts how many times the Levenberg-Marquardt projection algorithm failed to converge in 30 steps.

Importantly, the Reliable column counts samples that successfully converged with an eigengap strictly greater than zero ($\lambda_{d^*} > \lambda_{d^*+1}$). However, we must note a numerical behavior in the very early stages of generation (e.g., Steps 0 and 100). Although these samples are counted as "Reliable" because their gap is technically $> 0$, their actual eigengaps are extremely small (typically $< 10^{-6}$). Physically, this means the local geometry is very close to random isotropic Gaussian noise, causing the tangent and normal spaces to blend together. Because a clear manifold structure has not yet formed in this high-noise phase, the algorithm's numerical stability is weak. Therefore, the projection distances measured at these earliest steps have limited geometric meaning. Based on this observation , we deliberately omit the initial high-noise stages at $t=10.0$ and $t=4.67$ (i.e., Steps 0 and 100) in Figure \ref{fig:combined_ridge_metrics}(a) and \ref{fig:combined_ridge_metrics}(b).

\paragraph{Why the eigengap is small at large noise scales.}
The small eigengap in the early high-noise regime can be understood directly from the structure of the probability density. Under
$X_t\mid X_0\sim \mathcal{N}(a_tX_0,h_tI)$. We have computed the correspoding Hessian matrices:
\[
\nabla^2\log p_t(x)
=
-\frac{1}{h_t}I
+
\frac{a_t^2}{h_t^2}
\operatorname{Cov}(X_0\mid X_t=x).
\]
The dominant term $-h_t^{-1}I$ is isotropic and shifts all eigenvalues equally, while the anisotropic part that separates tangent and normal directions is scaled by $a_t^2/h_t^2$. Therefore,
\[
\lambda_{d^*}\big(\nabla^2\log p_t(x)\big)
-
\lambda_{d^*+1}\big(\nabla^2\log p_t(x)\big)
=
\frac{a_t^2}{h_t^2}
\left(\mu_{d^*}(x)-\mu_{d^*+1}(x)\right),
\]
where $\mu_i(x)$ are the eigenvalues of $\operatorname{Cov}(X_0\mid X_t=x)$. When $t$ is large, the noise scale $h_t$ is almost $1$ but $a_t$ tends to $0$, so this eigenvalue gap becomes very small. Hence the Hessian spectrum is nearly degenerate, the local tangent and normal spaces are poorly separated, and the measured projection distance has limited geometric meaning in the earliest reverse-sampling steps.
\begin{table}[htbp]
\centering

\label{tab:manifold_monitoring}
\begin{tabular}{r c c c c c c c}
\toprule
Step & $t$ & Mean Dist. & Reliable & Curv. OK & Mean Resid. & Mean Gap & Failures \\
\midrule
0    & 10.00000 & 5.422006 & 200 & 200 & 5.422012   & 0.000000 & 0 \\
100  & 4.67624  & 5.361732 & 200 & 200 & 5.361917   & 0.000004 & 0 \\
200  & 2.18672  & 5.272104 & 200 & 200 & 5.333307   & 0.000738 & 0 \\
300  & 1.02257  & 5.149778 & 200 & 200 & 5.800307   & 0.013401 & 0 \\
400  & 0.47818  & 4.451155 & 200 & 200 & 6.642476   & 0.080860 & 0 \\
500  & 0.22361  & 3.426006 & 197 & 198 & 8.438073   & 0.409678 & 1 \\
600  & 0.10456  & 2.501570 & 160 & 164 & 11.973254  & 1.657129 & 4 \\
700  & 0.04890  & 1.838713 & 177 & 178 & 18.188848  & 4.139380 & 1 \\
800  & 0.02287  & 1.442437 & 193 & 194 & 30.535279  & 5.491860 & 1 \\
900  & 0.01069  & 1.197336 & 199 & 199 & 53.974764  & 2.059901 & 0 \\
1000 & 0.00500  & 1.233171 & 200 & 200 & 116.260759 & 0.420446 & 0 \\
\bottomrule
\end{tabular}
\vspace{0.5em}
\caption{Dynamic manifold distance monitoring data during the reverse sampling process of the diffusion model (intrinsic dimension $d^*=3$).}
\end{table}

\begin{table}[htbp]
\centering

\label{tab:manifold_monitoring_d5}
\begin{tabular}{r c c c c c c c}
\toprule
Step & $t$ & Mean Dist. & Reliable & Curv. OK & Mean Resid. & Mean Gap & Failures \\
\midrule
0    & 10.00000 & 5.130307 & 200 & 200 & 5.130310  & 0.000000 & 0 \\
100  & 4.67624  & 2.989573 & 200 & 200 & 2.989411  & 0.000007 & 0 \\
200  & 2.18672  & 2.882148 & 200 & 200 & 2.919268  & 0.001082 & 0 \\
300  & 1.02257  & 2.724726 & 200 & 200 & 3.124859  & 0.013572 & 0 \\
400  & 0.47818  & 2.207066 & 200 & 200 & 3.423936  & 0.058334 & 0 \\
500  & 0.22361  & 1.710056 & 200 & 200 & 4.212194  & 0.196339 & 0 \\
600  & 0.10456  & 1.407467 & 180 & 180 & 5.857416  & 0.852296 & 0 \\
700  & 0.04890  & 1.058827 & 192 & 192 & 9.560239  & 2.028122 & 0 \\
800  & 0.02287  & 0.799892 & 200 & 200 & 16.567018 & 0.794553 & 0 \\
900  & 0.01069  & 0.741769 & 200 & 200 & 32.952072 & 0.029825 & 0 \\
1000 & 0.00500  & 0.664260 & 200 & 200 & 65.937224 & 0.001428 & 0 \\
\bottomrule
\end{tabular}
\vspace{0.5em}
\caption{Dynamic manifold distance monitoring data during the reverse sampling process of the diffusion model (intrinsic dimension $d^*=5$).}
\end{table}

\begin{table}[htbp]
\centering
\label{tab:manifold_monitoring_d8}
\begin{tabular}{r c c c c c c c}
\toprule
Step & $t$ & Mean Dist. & Reliable & Curv. OK & Mean Resid. & Mean Gap & Failures \\
\midrule
0    & 10.00000 & 4.880301 & 200 & 200 & 4.880306  & 0.000000 & 0 \\
100  & 4.67624  & 4.852153 & 200 & 200 & 4.852658  & 0.000004 & 0 \\
200  & 2.18672  & 4.776520 & 200 & 200 & 4.910469  & 0.000521 & 0 \\
300  & 1.02257  & 4.649101 & 200 & 200 & 5.567016  & 0.009011 & 0 \\
400  & 0.47818  & 4.142639 & 200 & 200 & 6.632765  & 0.047487 & 0 \\
500  & 0.22361  & 3.048265 & 200 & 200 & 8.350910  & 0.157566 & 0 \\
600  & 0.10456  & 2.196695 & 200 & 200 & 11.629959 & 0.269339 & 0 \\
700  & 0.04890  & 1.566005 & 200 & 200 & 16.753251 & 0.163079 & 0 \\
800  & 0.02287  & 1.209245 & 200 & 200 & 26.358911 & 0.020655 & 0 \\
900  & 0.01069  & 1.014856 & 200 & 200 & 44.848112 & 0.000287 & 0 \\
1000 & 0.00500  & 1.091011 & 200 & 200 & 97.710713 & 0.000013 & 0 \\
\bottomrule
\end{tabular}
\vspace{0.5em}
\caption{Dynamic manifold distance monitoring data during the reverse sampling process of the diffusion model (intrinsic dimension $d^*=8$).}
\end{table}

\begin{table}[htbp]
\centering
\label{tab:manifold_monitoring_d12}
\begin{tabular}{r c c c c c c c}
\toprule
Step & $t$ & Mean Dist. & Reliable & Curv. OK & Mean Resid. & Mean Gap & Failures \\
\midrule
0    & 10.00000 & 4.463758 & 200 & 200 & 4.463762  & 0.000000 & 0 \\
100  & 4.67624  & 4.484123 & 200 & 200 & 4.484518  & 0.000000 & 0 \\
200  & 2.18672  & 4.381252 & 200 & 200 & 4.437409  & 0.000049 & 0 \\
300  & 1.02257  & 4.121614 & 200 & 200 & 4.737415  & 0.000577 & 0 \\
400  & 0.47818  & 3.480388 & 200 & 200 & 5.719486  & 0.002088 & 0 \\
500  & 0.22361  & 2.561228 & 200 & 200 & 7.348668  & 0.004829 & 0 \\
600  & 0.10456  & 1.824004 & 200 & 200 & 10.150477 & 0.006965 & 0 \\
700  & 0.04890  & 1.314270 & 200 & 200 & 14.600166 & 0.003107 & 0 \\
800  & 0.02287  & 0.982827 & 200 & 200 & 22.123408 & 0.000174 & 0 \\
900  & 0.01069  & 0.820614 & 200 & 200 & 36.948018 & 0.000017 & 0 \\
1000 & 0.00500  & 0.892560 & 200 & 200 & 85.796590 & 0.000000 & 0 \\
\bottomrule
\end{tabular}
\vspace{0.5em}
\caption{Dynamic manifold distance monitoring data during the reverse sampling process of the diffusion model (intrinsic dimension $d^*=12$).}
\end{table}

\clearpage
\section{More Experiment Results}\label{append:more experiments}

\subsection{Biases under Different Initializations}\label{append:initialization} \begin{figure}[htbp]
    \centering
    \captionsetup{font=footnotesize}
        \begin{minipage}{0.30\linewidth}
            \centering
            \includegraphics[width=\linewidth]{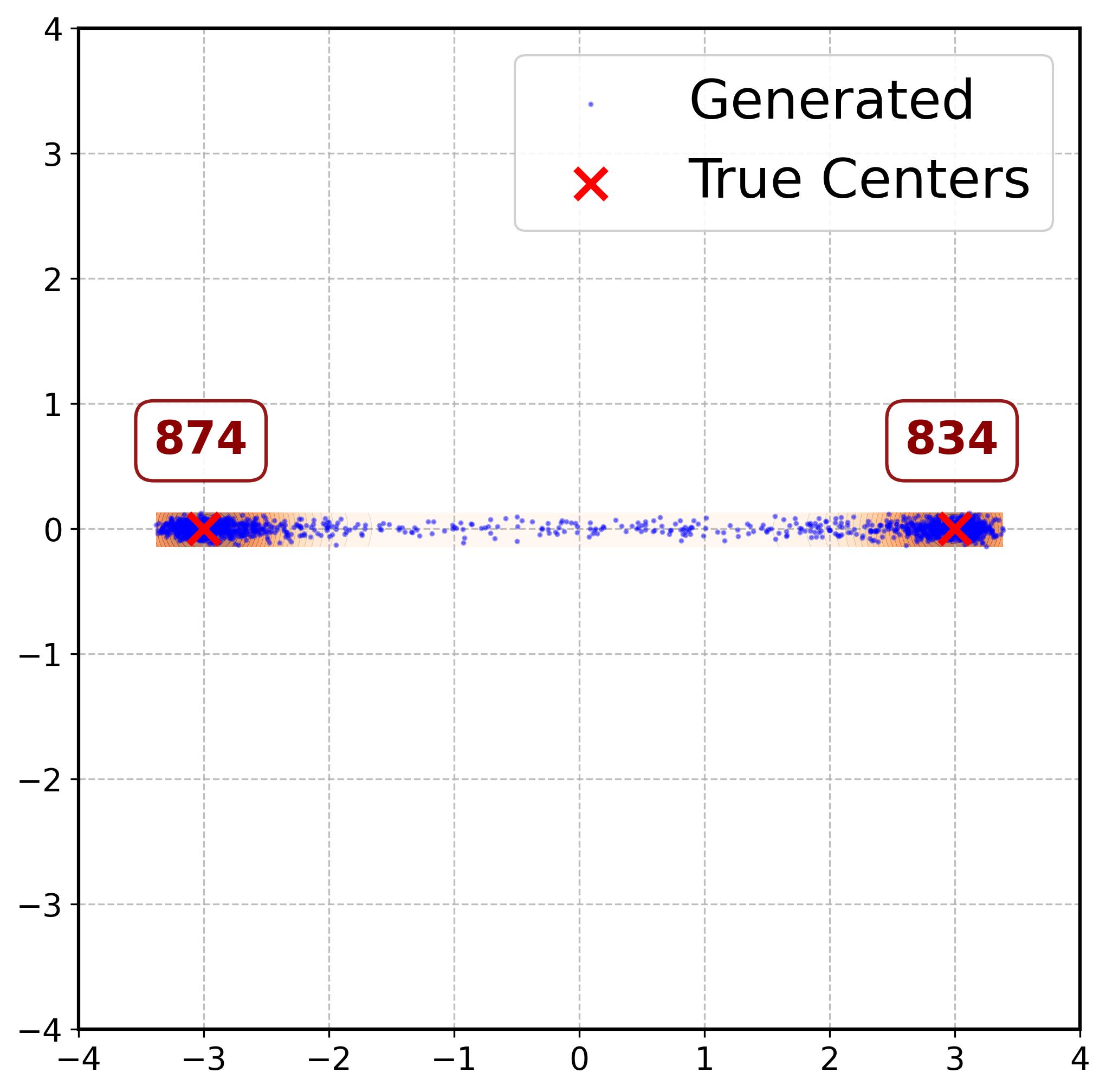}\\[-0.3em]
            {\small (a) Zero}
        \end{minipage}\hfill
        \begin{minipage}{0.30\linewidth}
            \centering
            \includegraphics[width=\linewidth]{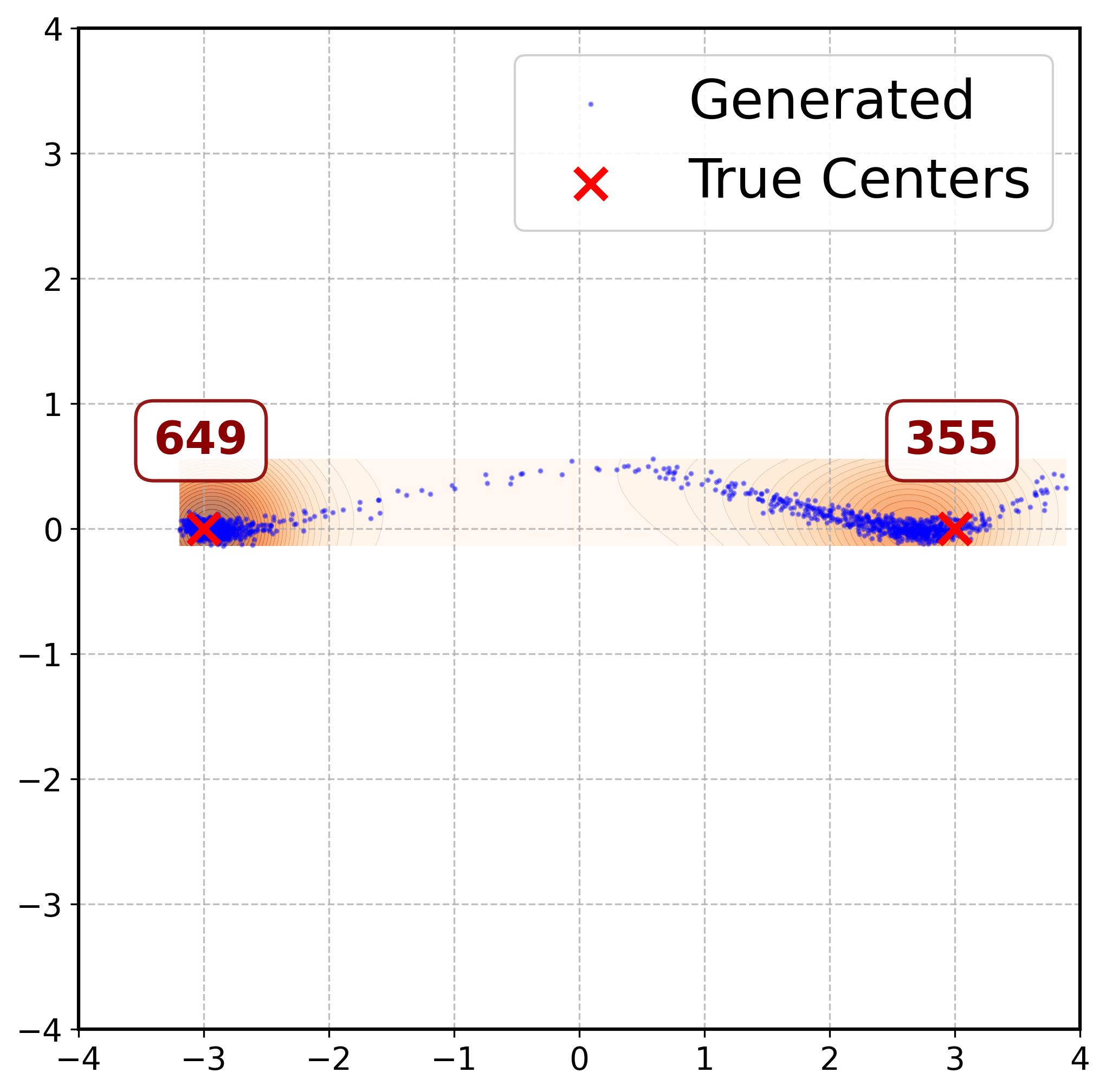}\\[-0.3em]
            {\small (b) Ones}
        \end{minipage}\hfill
        \begin{minipage}{0.30\linewidth}
            \centering
            \includegraphics[width=\linewidth]{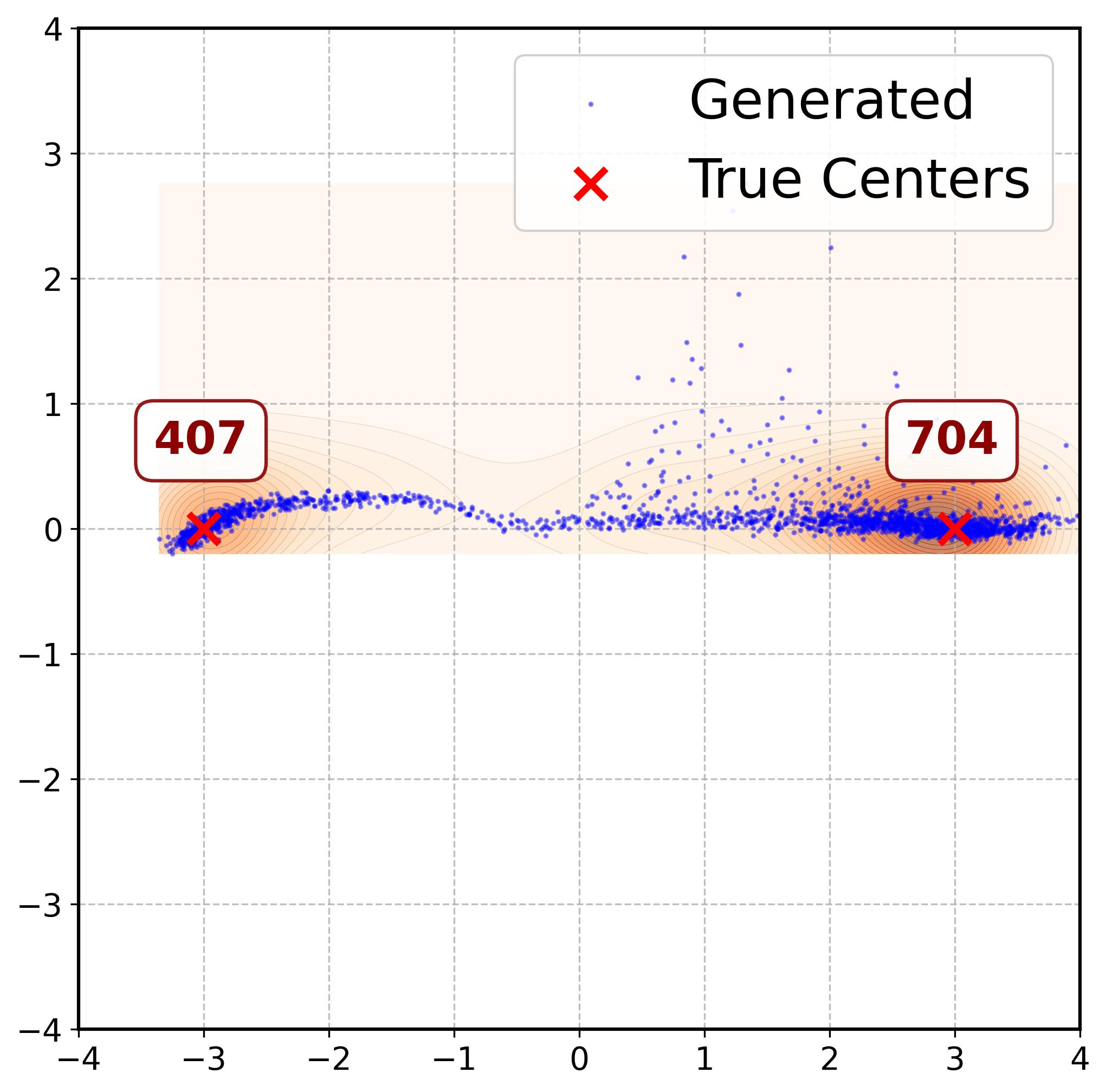}\\[-0.3em]
            {\small (c) Slow-Spec}
        \end{minipage}
        \caption{\textbf{Initialization Effects (Epoch 40k).} Comparison of generated sample configurations under different initializations. The colored shading denotes the KDE of the distribution.}
        \label{fig:RFNN_init_comparison4000}
\end{figure}

\noindent\textbf{Biases of Different Initializations}. We test the initialization effects predicted by the RFNN analysis in Section~\ref{subsec:RF+GD}. Figure~\ref{fig:RFNN_init_comparison4000} compares finite-training-time generation under three initialization schemes: zero, all-ones, and slow-spectrum. The zero initialization yields samples that remain essentially on the horizontal ridge, while the other two produce a visible arch before full alignment is reached.

This is consistent with the RFNN analysis in Section~\ref{subsec:RF+GD}: initialization mainly affects how quickly normal alignment is achieved. When the relevant slow modes are weak, alignment is nearly immediate; when they are emphasized, the approach to the ridge is much slower, producing the transient arch-shaped geometry seen in the figure.

\subsection{Numerical Verification of Geometric Biases for Two-Point on MLP}\label{append:2-point mlp}
We also verify the two-point example in the main text using a standard MLP in place of RFNN. The purpose of this experiment is to check that the geometric picture from Section~\ref{sec:inference} is not specific to the RFNN parametrization. As in Section~\ref{subsec:synthetic data}, the dataset is
$
\mathcal D=\{(-3,0),(3,0)\}\subset\mathbb R^2,
$
for which the ridge is the horizontal axis and the normal and tangent directions are explicit.

\noindent\textbf{Model Architectures} We employ a two-layer MLP for $m_A(x,t)$. The scalar $t$ is first transformed via a sinusoidal embedding of dimension $32$. This embedding is concatenated with the input $\mathbf{x} \in \mathbb{R}^2$, passed through a hidden layer of 128 units with ReLU activation, and finally projected to $\mathbb{R}^2$.

\begin{figure}[h]
\centering
\captionsetup{font=footnotesize}
\begin{minipage}[b]{0.30\linewidth}
  \centering
  \includegraphics[width=\linewidth]{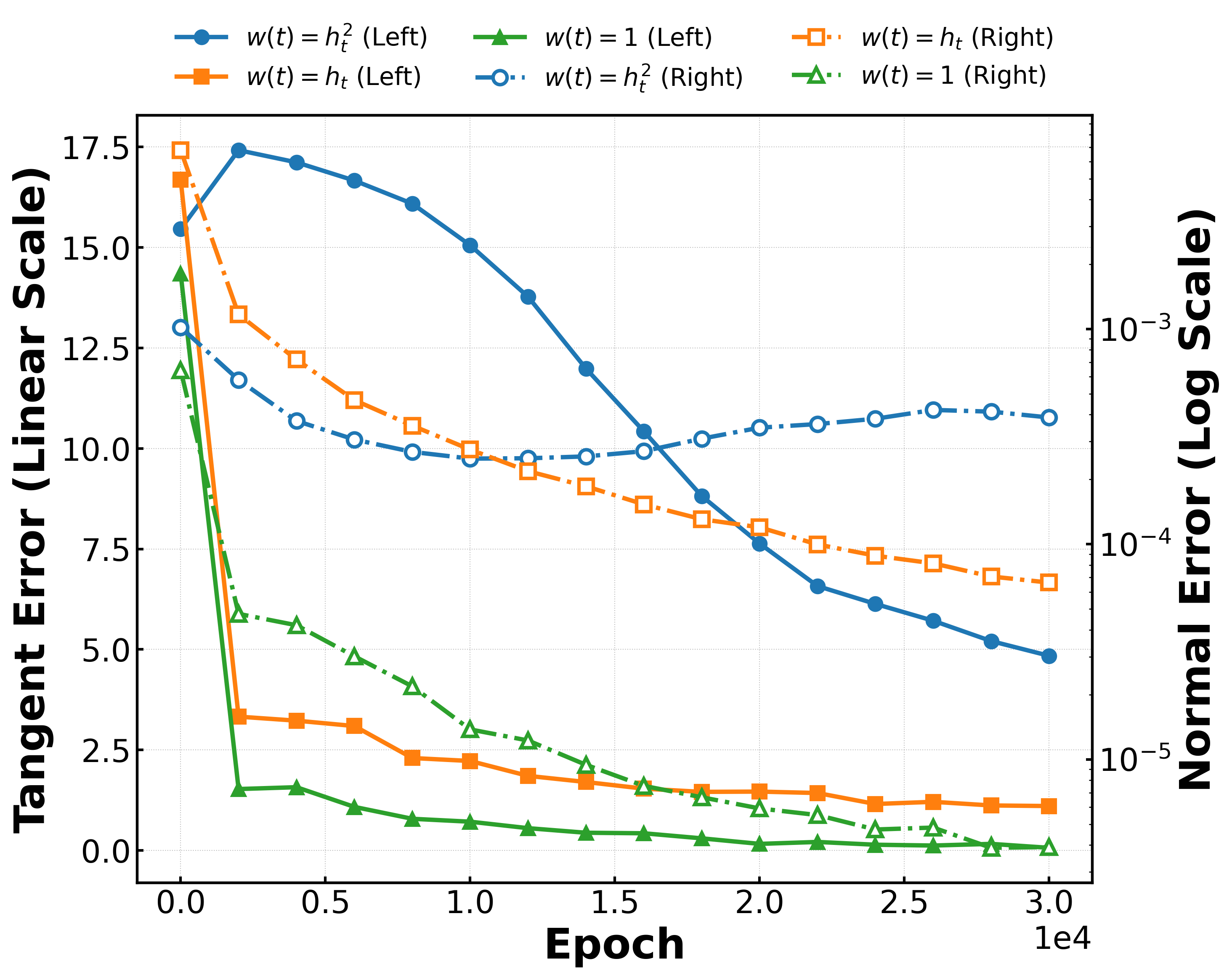}\\[-0.3em]
  {\small (a) Error Dynamics} 
\end{minipage}\hfill
\begin{minipage}[b]{0.22\linewidth}
  \centering
  \includegraphics[width=\linewidth]{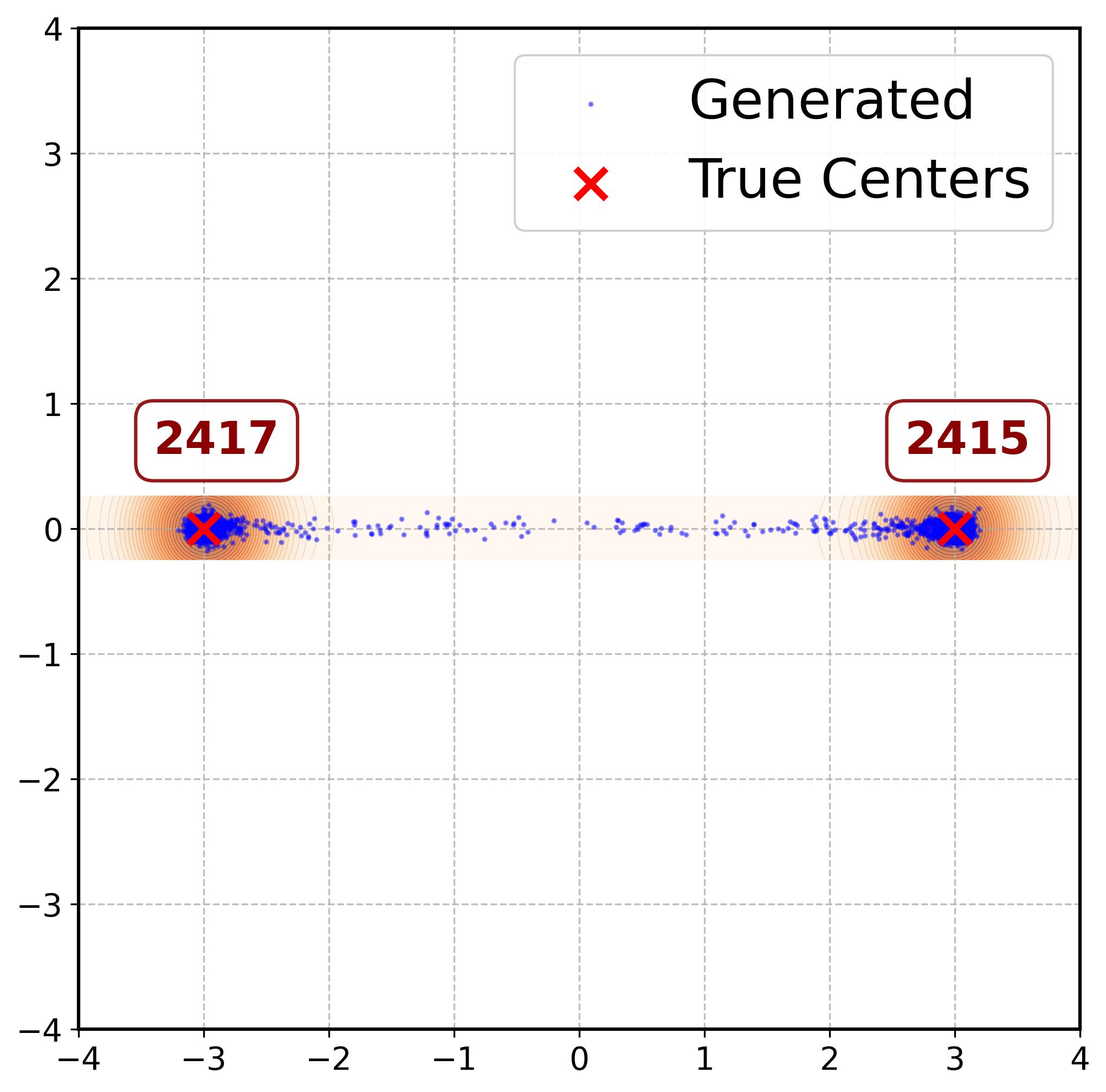}\\[-0.3em]
  {\small (b) $w(t)=1$}
\end{minipage}\hfill
\begin{minipage}[b]{0.22\linewidth}
  \centering
  \includegraphics[width=\linewidth]{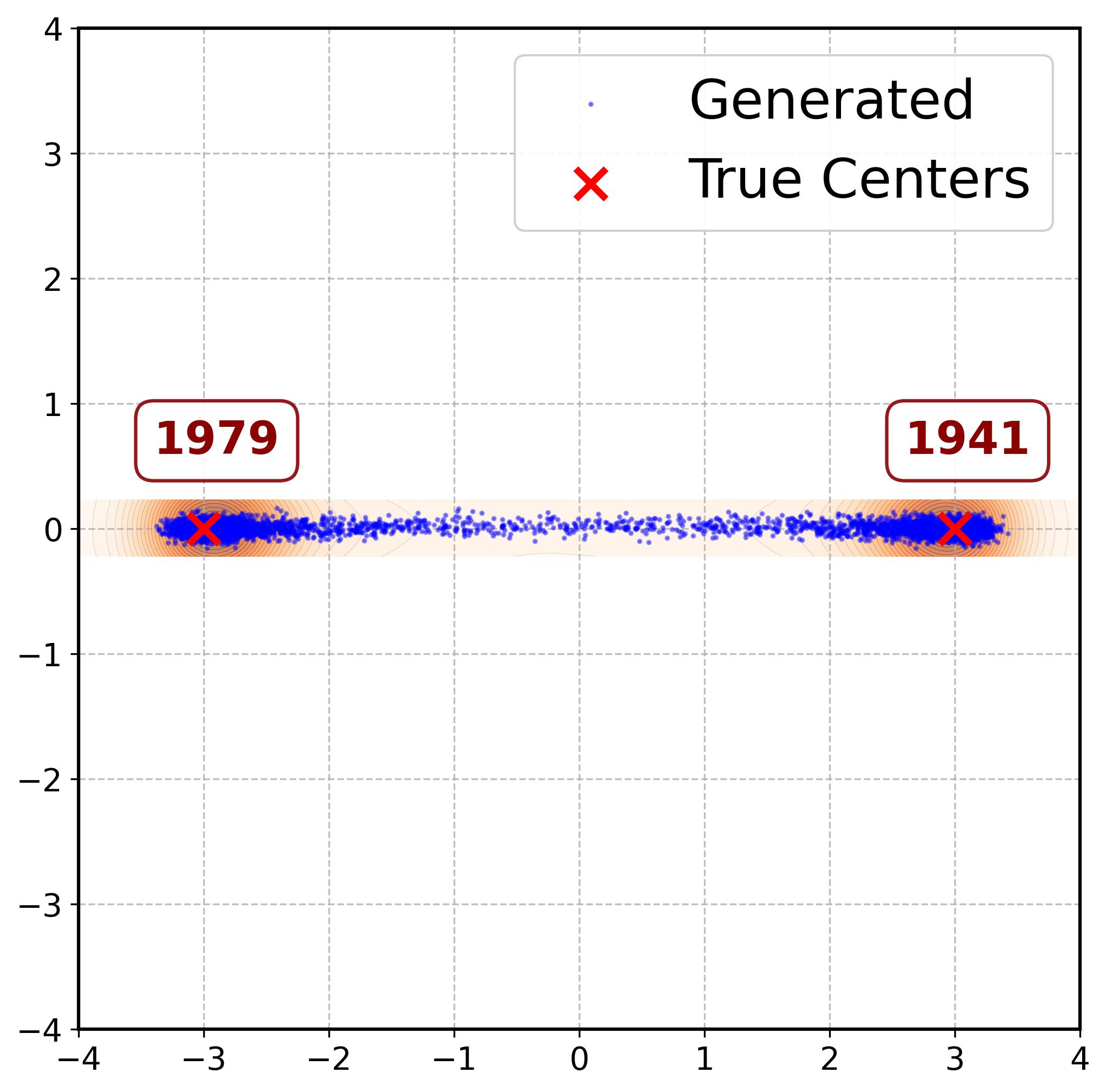}\\[-0.3em]
  {\small (c) $w(t)=h_t$}
\end{minipage}\hfill
\begin{minipage}[b]{0.22\linewidth}
  \centering
  \includegraphics[width=\linewidth]{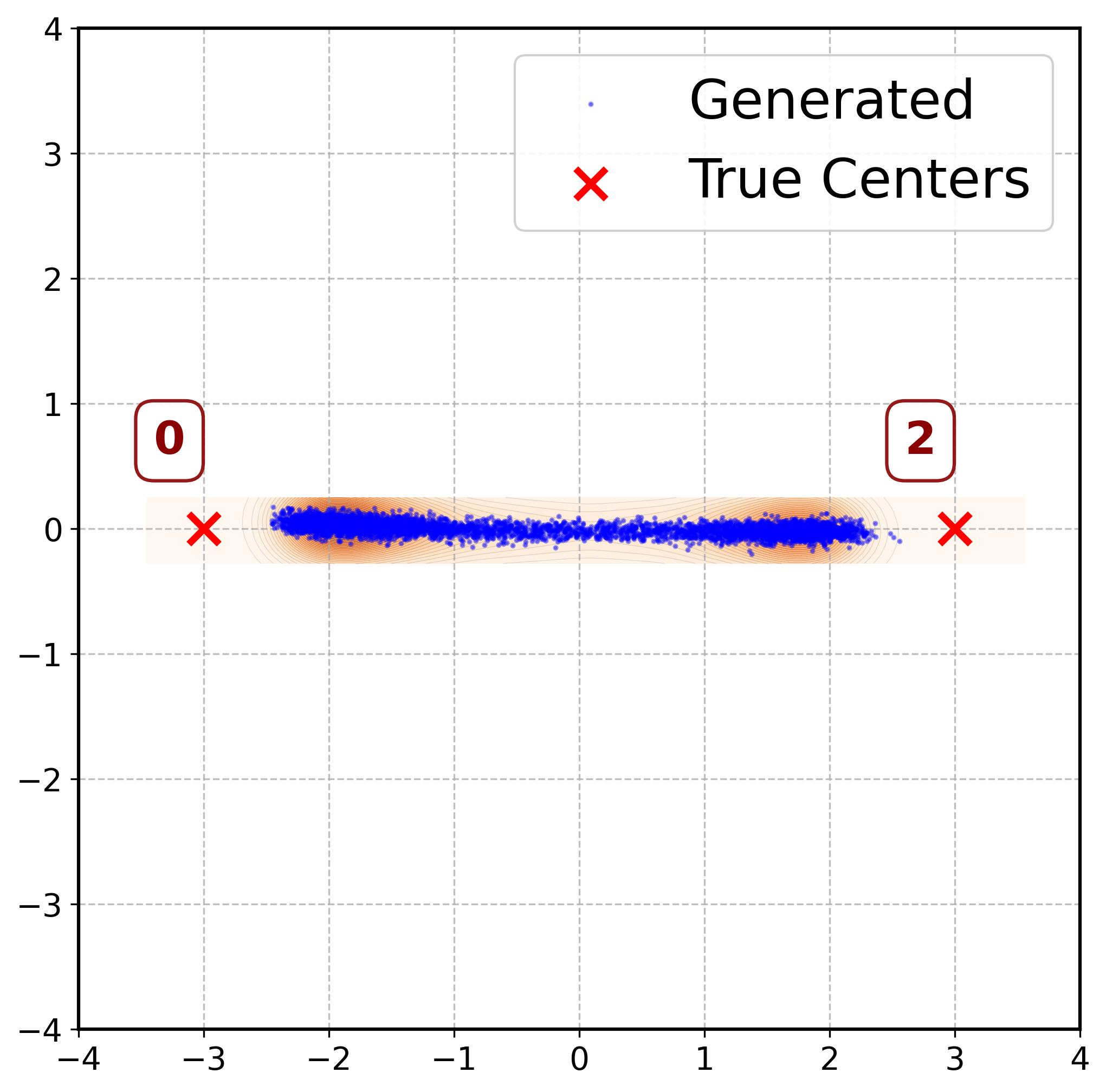}\\[-0.3em]
  {\small (d) $w(t)=h_t^2$}
\end{minipage}

\caption{\textbf{Error dynamics and generated samples of MLP.}
(a) Evolution of tangential errors (solid lines, left axis, linear scale) and normal errors (dash-dot lines, right axis, log scale) during training.
(b)--(d) Comparison of generated sample configurations under different weighting schedules. Boxed numbers indicate sample counts around the target modes ($\text{radius}=0.5$). The background color represents the KDE plot.}
\label{fig:MLP_comprehensive_row} 
\end{figure}

\noindent\textbf{Training and Loss Computation} We trained for $3 \times 10^4$ epochs with a learning rate of $1 \times 10^{-4}$. The gradient descent is performed using the full-batch SGD optimizer in PyTorch and loss function \eqref{eq:DMM}. We numerically evaluate the integral in the loss function \eqref{eq:DMM} using the trapezoidal rule with 2,000 discretization points.

\noindent\textbf{SDE and Sampling Details} We model the diffusion process using a stochastic differential equation (SDE) with a drift coefficient of $f(x,t) = -x$ and a constant diffusion coefficient of $g(t) = \sqrt{2}$. For generation, we employ the Euler-Maruyama sampler with $N=1000$ discrete time steps. The time schedule follows a geometric progression decaying from $T=10$ down to $\delta = t_{min}=10^{-3}$. Specifically, the time points are defined as $t_i = T \cdot (t_{min}/T)^{i/N}$ for $i=0, \dots, N$.

\noindent\textbf{Computational Resources} All experiments were conducted on a single NVIDIA RTX 4090 GPU. For the MLP-based experiments, since each model is trained for only $3\times 10^4$ epochs, all experiments can be completed in approximately 15 minutes. 

Figure~\ref{fig:MLP_comprehensive_row} shows that the same directional geometric pattern observed in RFNN also appears in MLP. The normal error remains very small throughout training, and the generated samples stay tightly concentrated near the ridge $y=0$. By contrast, the tangential error settles at a visibly larger floor, and the generated samples spread along the line segment between the two data points. Thus, even for MLP, the generated geometry is well explained by strong normal alignment together with persistent tangential spread.

The weighting schedule again mainly affects the tangential geometry rather than the normal geometry. In particular, $w(t)=1$ yields the smallest tangential error and the strongest concentration near the two training points, while $w(t)=h_t$ and $w(t)=h_t^2$ produce progressively larger tangential floors and more pronounced edge-like interpolation. This is fully consistent with the interpretation in Remark~\ref{rem:effect_of_weight} and shows that the directional geometric picture is not tied to the RFNN setting.

\subsection{RFNN on Different Sets}\label{sec:RFNN More Examples}

We consider additional 2D examples to illustrate that the ridge geometry continues to explain generation even when the geometry is no longer as simple as in the two-point case. The purpose of these experiments is not primarily quantitative verification, but to show that the proposed log-density ridge captures nontrivial low-dimensional generation patterns beyond a single straight line.

\subsubsection*{Two-Point Case: (-3,0), (3,0)}
\vspace{-.1cm}
 We compare three distinct weight schedules: $w(t)=1$ (Figure~\ref{fig:w_1}), $w(t)=h_t$ (Figure~\ref{fig:w_ht}), and $w(t)=h_t^2$ (Figure~\ref{fig:w_ht2}). Across all three cases, the same qualitative three-stage evolution is visible, but the later tangential behavior changes substantially with the weight schedule.
 \begin{figure}[h]
    \centering
    \captionsetup{font=footnotesize}
    \includegraphics[width=0.6\linewidth]{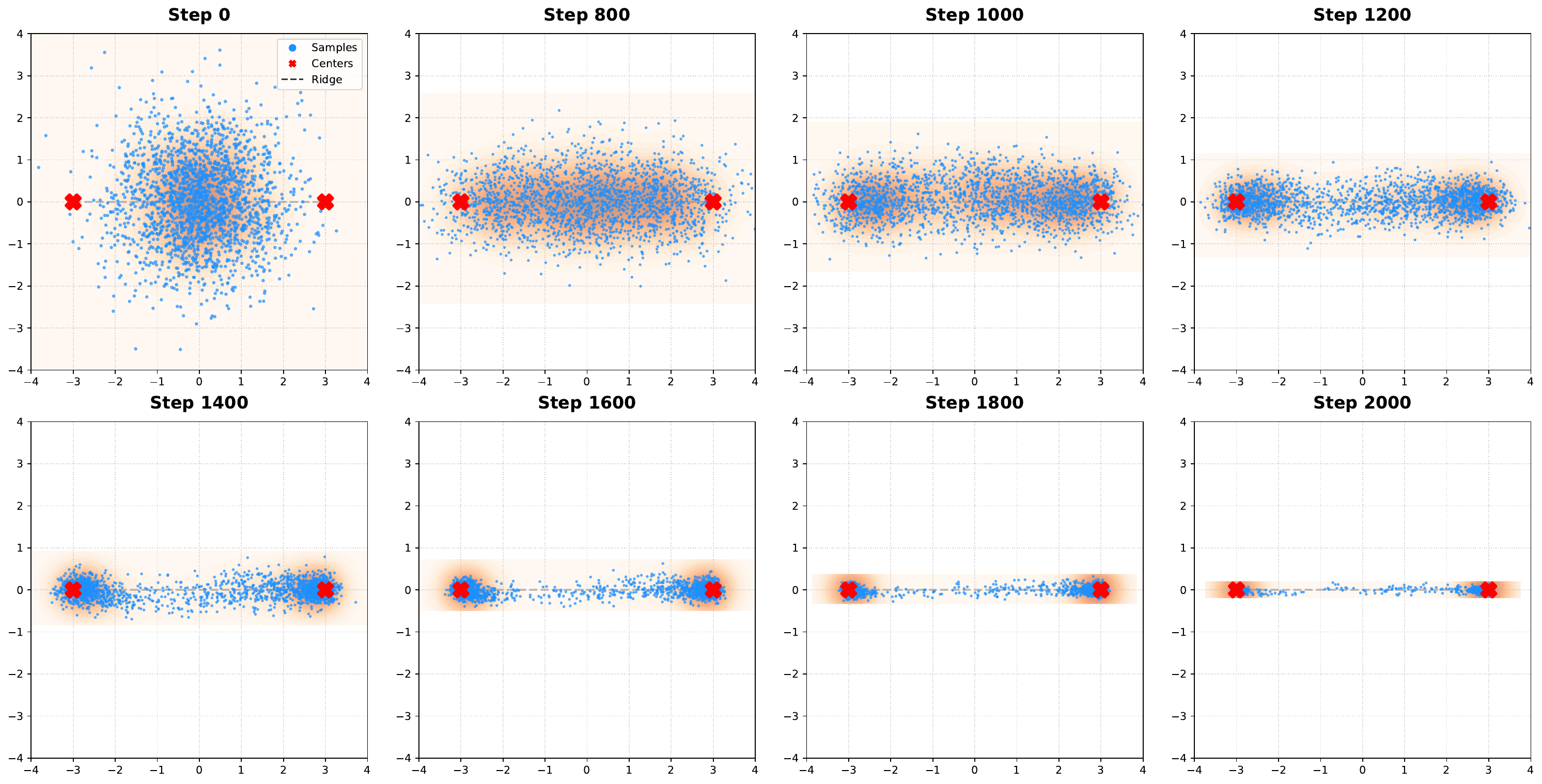} 
    \caption{\textbf{Evolution of generated samples under the proposed sampling dynamics with weight schedule $w(t)=1$.} The visualization displays snapshots of $N=2000$ particles in a 2D plane during the sampling process. The background contours depict the kernel density estimation (KDE) of the particle distribution. Samples reach the ridge neighborhood at \textbf{Step 800}. \textbf{Steps 800--1800} illustrates Normal Alignment. \textbf{Steps 1800--2000} demonstrate Tangent Sliding.}
    \label{fig:w_1} 
\end{figure}
\begin{figure}[h]
    \centering
    \captionsetup{font=footnotesize}
    \includegraphics[width=0.6\linewidth]{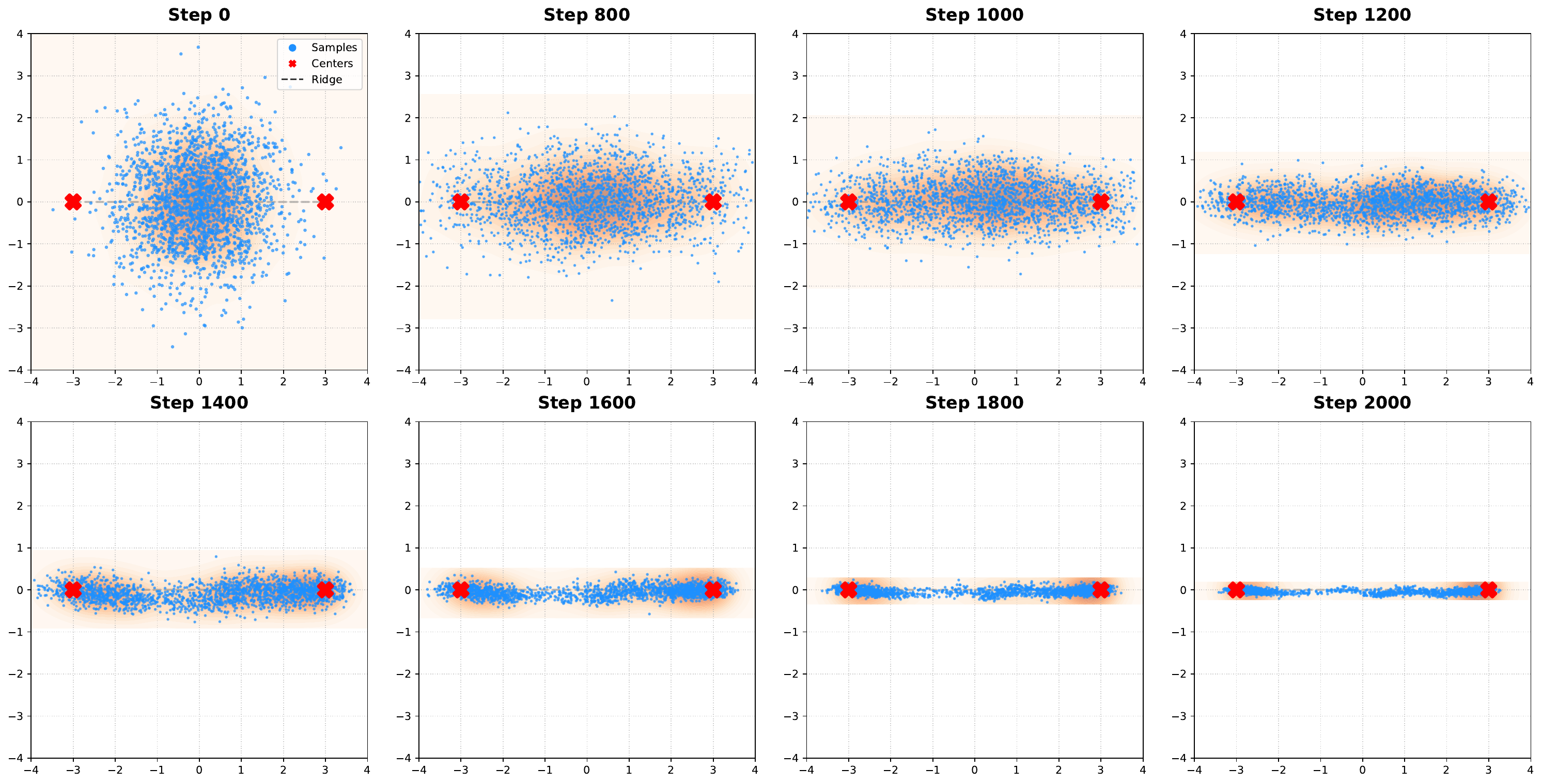}
    \caption{\textbf{Evolution of generated samples under the proposed sampling dynamics with weight schedule $w(t)=h_t$.} The visualization displays snapshots of $N=2000$ particles in a 2D plane during the sampling process. The background contours depict the kernel density estimation (KDE) of the particle distribution. Samples reach the ridge neighborhood at \textbf{Step 800}. \textbf{Steps 800--1800} illustrates Normal Alignment. \textbf{Steps 1800--2000} demonstrate Tangent Sliding.}
    \label{fig:w_ht} 
\end{figure}
\begin{figure}[h]
    \centering
    \captionsetup{font=footnotesize}
    \includegraphics[width=0.6\linewidth]{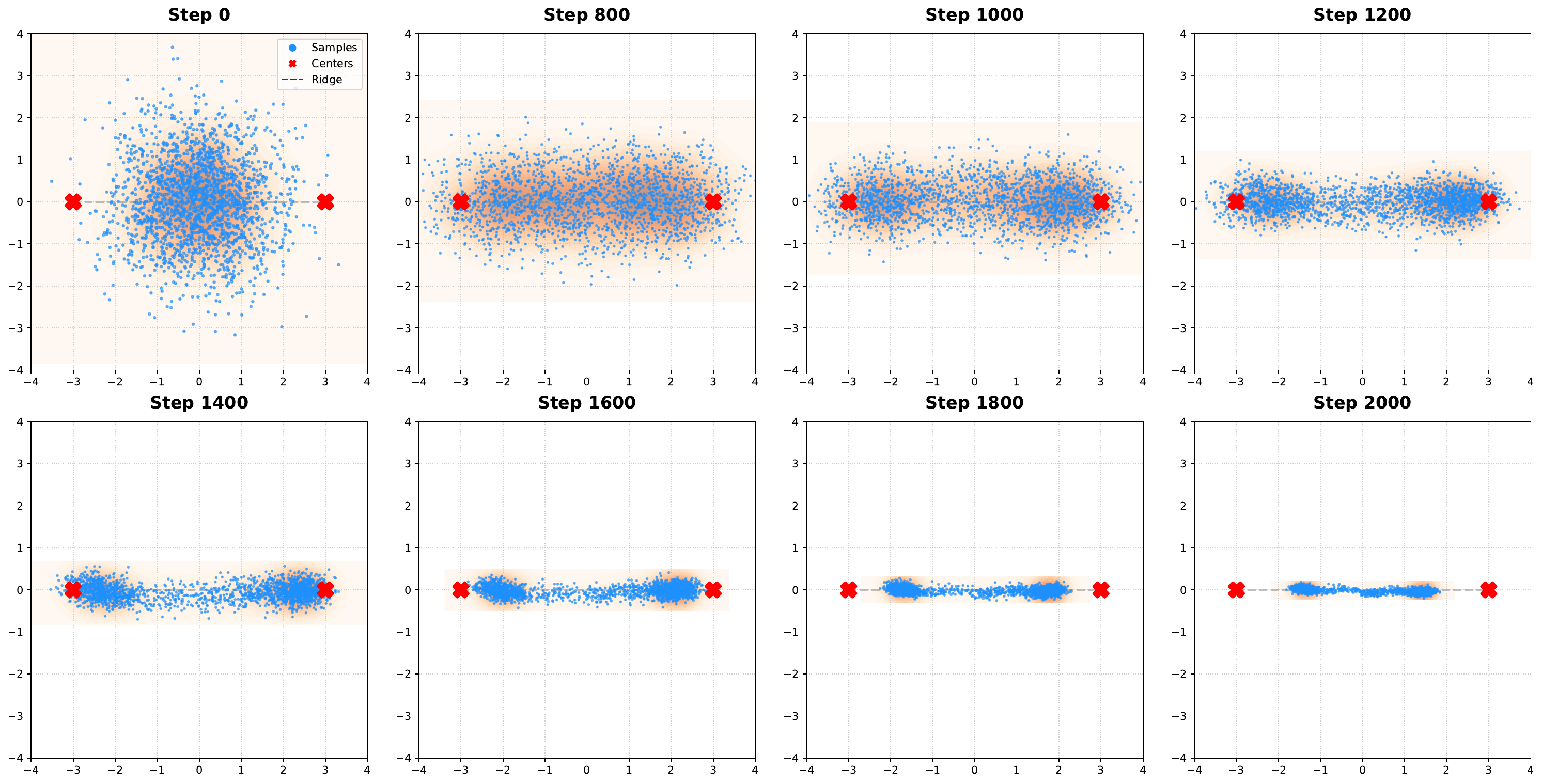}
    \caption{\textbf{Evolution of generated samples under the proposed sampling dynamics with weight schedule $w(t)=h_t^2$.} The visualization displays snapshots of $N=2000$ particles in a 2D plane during the sampling process. The background contours depict the kernel density estimation (KDE) of the particle distribution. Samples reach the ridge neighborhood at \textbf{Step 1200}. \textbf{Steps 1200--1800} illustrates Normal Alignment. \textbf{Steps 1800--2000} demonstrate Tangent Sliding.}
    \label{fig:w_ht2} 
\end{figure}

\clearpage

\subsubsection*{Three-Point Case: (0,0), (5,0), (0,5)} 
This example shows that the theory is not limited to straight-line interpolation between two modes. Here the relevant ridge geometry is asymmetric and bent, so the figures illustrate how the same reach--align--slide mechanism persists even when the low-dimensional structure is no longer a single line segment. In particular, Figures~\ref{fig:3w_1}, \ref{fig:3w_ht}, and \ref{fig:3w_1_ht2} show that for all three weight schedules $w(t)=1$, $w(t)=h_t$, and $w(t)=h_t^2$, the late-stage motion follows the curved ridge induced by the three-point configuration rather than collapsing onto a straight interpolation path.
\begin{figure}[h]
    \centering
    \captionsetup{font=footnotesize}
    \includegraphics[width=0.6\linewidth]{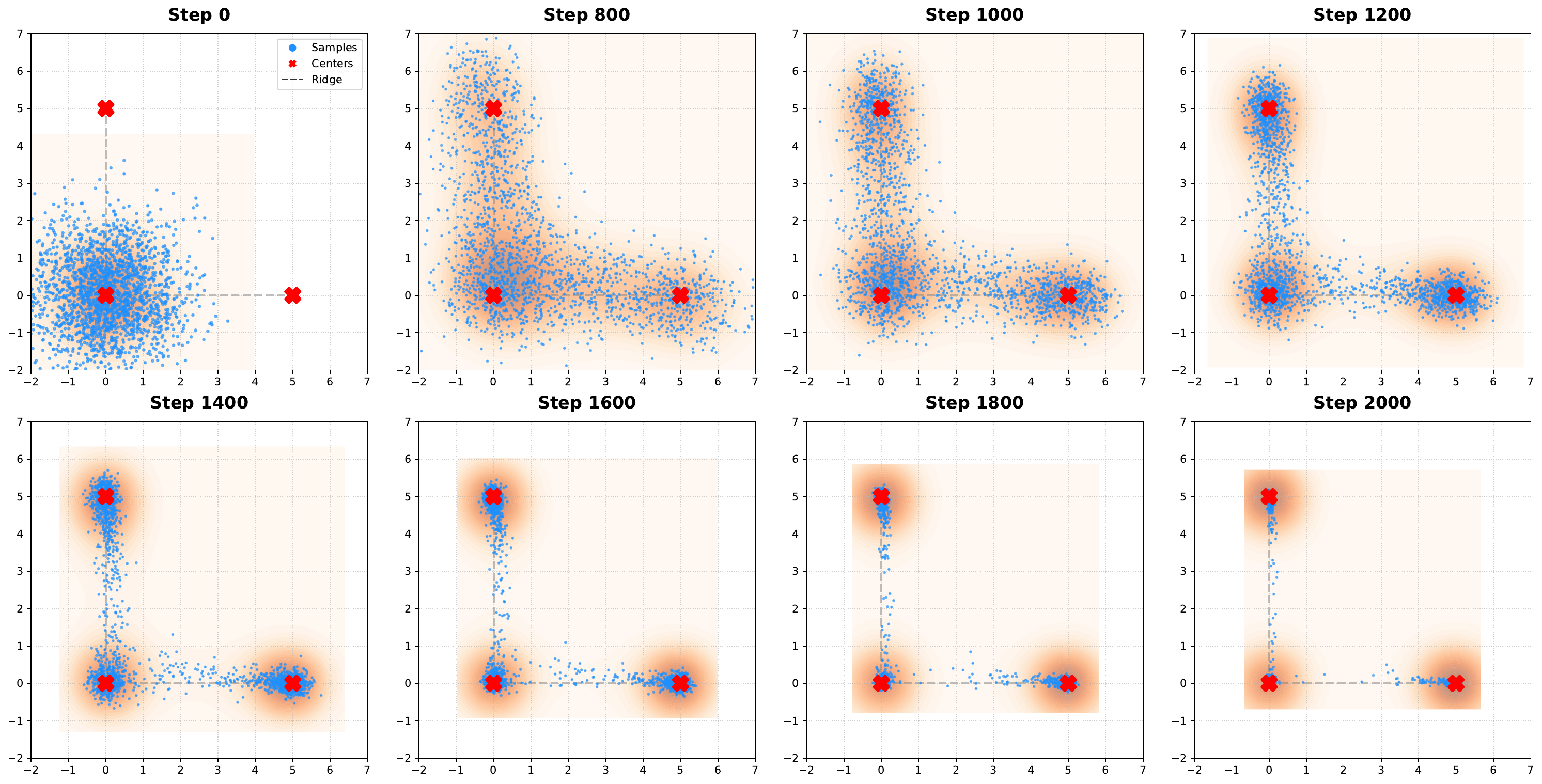} 
    \caption{\textbf{Evolution of generated samples under the proposed sampling dynamics with weight schedule $w(t)=1$.} The visualization displays snapshots of $N=2000$ particles in a 2D plane during the sampling process. The background contours depict the kernel density estimation (KDE) of the particle distribution. Samples reach the ridge neighborhood at \textbf{Step 800}. \textbf{Steps 800--1400} illustrates Normal Alignment. \textbf{Steps 1400--2000} demonstrate Tangent Sliding.\vspace{-.5cm}}
    \label{fig:3w_1} 
\end{figure}
\begin{figure}[h]
    \centering
    \captionsetup{font=footnotesize}
    \includegraphics[width=0.6\linewidth]{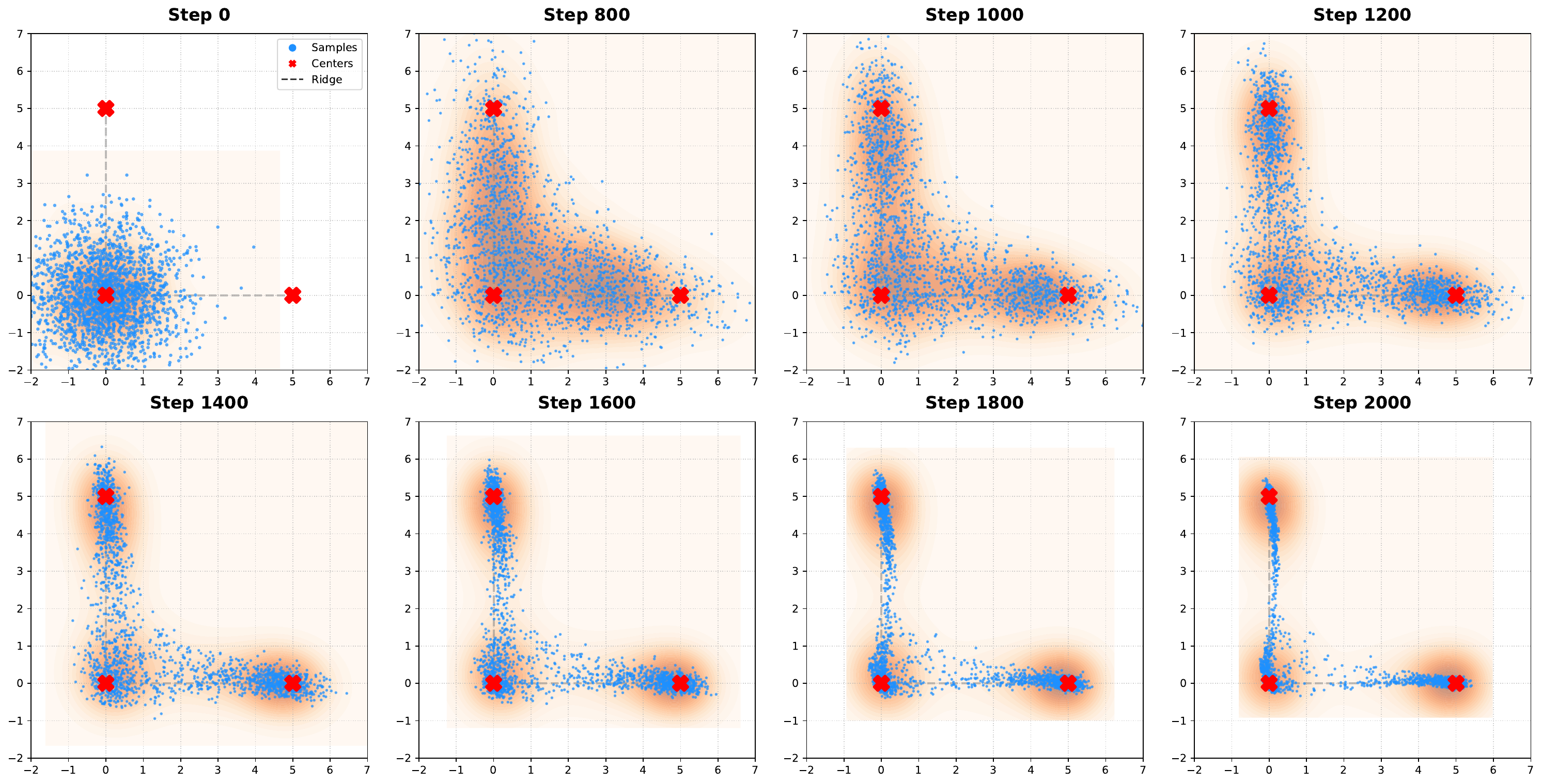} 
    \caption{\textbf{Evolution of generated samples under the proposed sampling dynamics with weight schedule $w(t)=h_t$.} The visualization displays snapshots of $N=2000$ particles in a 2D plane during the sampling process. The background contours depict the kernel density estimation (KDE) of the particle distribution. Samples reach the ridge neighborhood at \textbf{Step 1000}. \textbf{Steps 1000--1600} illustrates Normal Alignment. \textbf{Steps 1600--2000} demonstrate Tangent Sliding.\vspace{-.5cm}}
    \label{fig:3w_ht} 
\end{figure}
\begin{figure}[t]
    \centering
    \captionsetup{font=footnotesize}
    \includegraphics[width=0.6\linewidth]{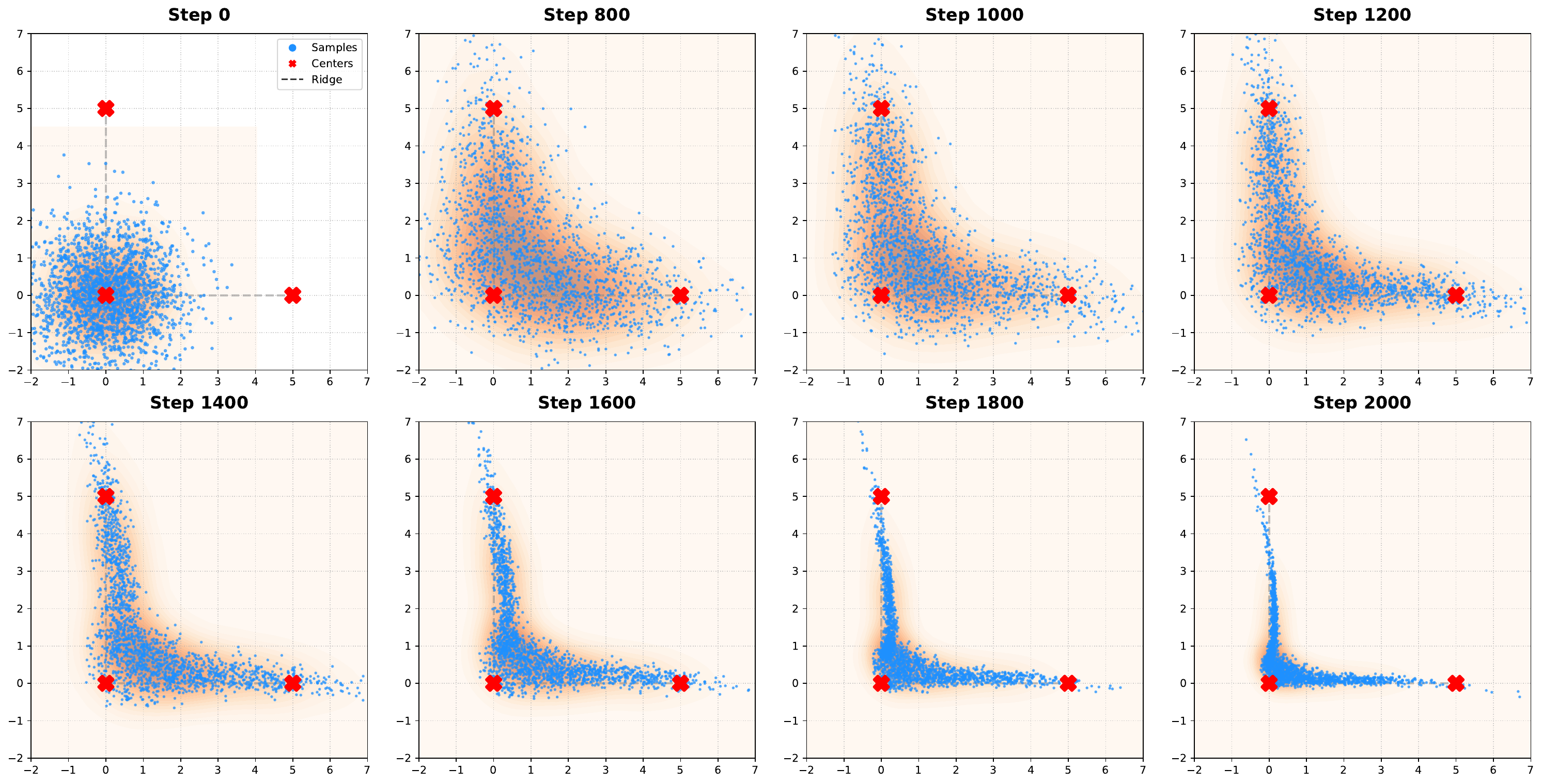} 
    \caption{\textbf{Evolution of generated samples under the proposed sampling dynamics with weight schedule $w(t)=h_t^2$.} The visualization displays snapshots of $N=2000$ particles in a 2D plane during the sampling process. The background contours depict the kernel density estimation (KDE) of the particle distribution. Samples reach the ridge neighborhood at \textbf{Step 1200}. \textbf{Steps 1200--1600} illustrates Normal Alignment. \textbf{Steps 1600--2000} demonstrate Tangent Sliding.\vspace{-.5cm}}
    \label{fig:3w_1_ht2} 
\end{figure}
\subsubsection*{Four-Point Cases}
These examples probe more complicated ridge structures generated by four-point configurations. Figure~\ref{fig:w_1_4_regular} corresponds to the symmetric configuration $(\pm2,\pm2)$, where the ridge geometry is comparatively regular and the resulting sample evolution is correspondingly structured. Figure~\ref{fig:w_1_4_singular} shows an unsymmetric four-point configuration, where the ridge becomes more distorted and locally less regular. Figure~\ref{fig:w_1_4_singular2} considers a qualitatively different arrangement in which three points form a triangle and the fourth lies inside that triangle, producing a more intricate interior geometry. Taken together, these examples show that the proposed ridge remains a useful reference object even when the local data structure becomes substantially more complicated than in the two- and three-point cases, and that the qualitative reach--align--slide behavior is still visible across these different configurations.
\begin{figure}[h]
    \centering
    \captionsetup{font=footnotesize}
    \includegraphics[width=0.6\linewidth]{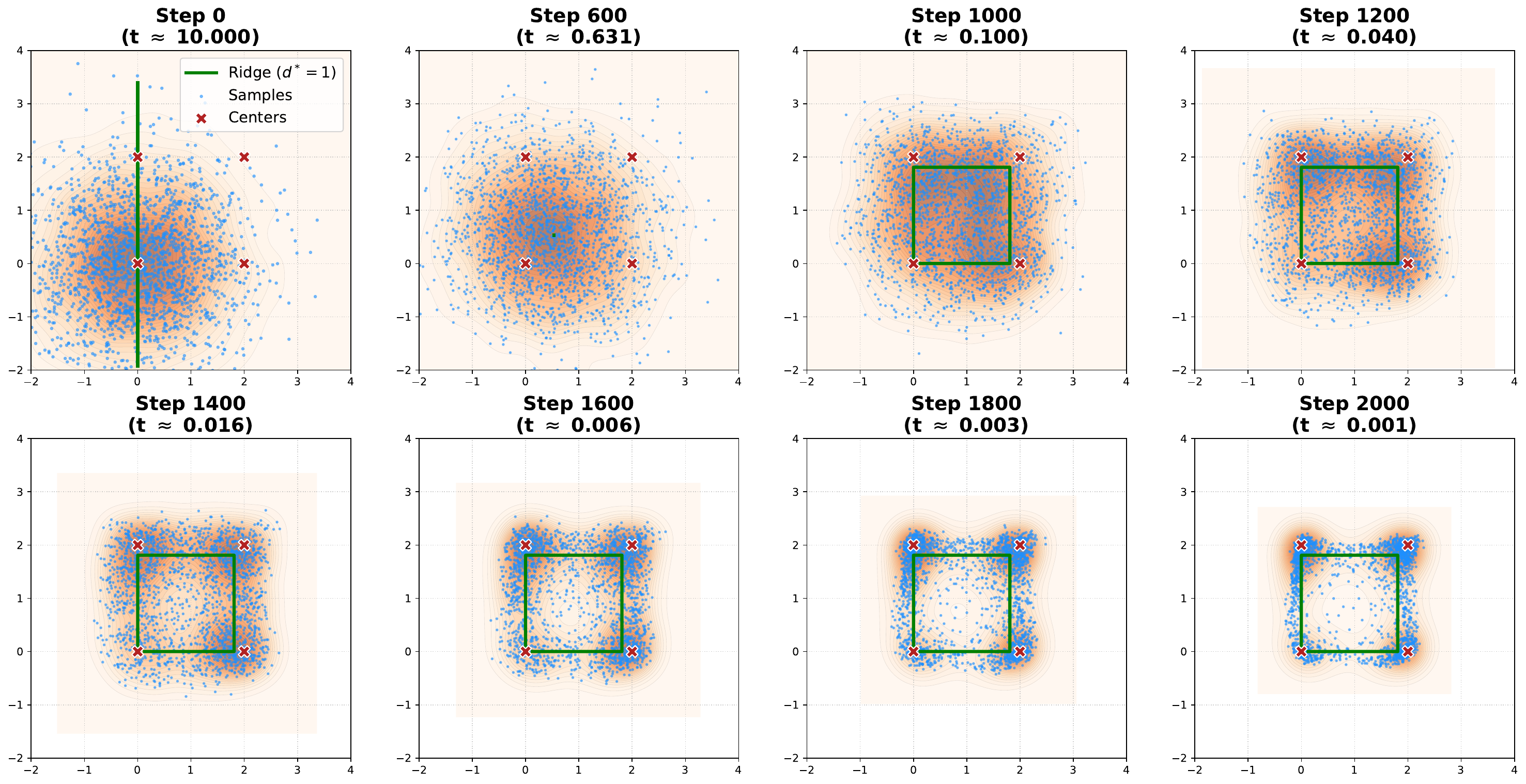} 
    \caption{\textbf{Evolution of generated samples under the proposed sampling dynamics with weight schedule $w(t)=h_t$.} The visualization displays snapshots of $N=2000$ particles in a 2D plane during the sampling process. The background contours depict the kernel density estimation (KDE) of the particle distribution. Samples reach the ridge neighborhood at \textbf{Step 1200}. \textbf{Steps 1200--1600} illustrates Normal Alignment. \textbf{Steps 1600--2000} demonstrate Tangent Sliding.}
    \label{fig:w_1_4_regular} 
\end{figure}
\begin{figure}[h]
    \centering
    \captionsetup{font=footnotesize}
    \includegraphics[width=0.6\linewidth]{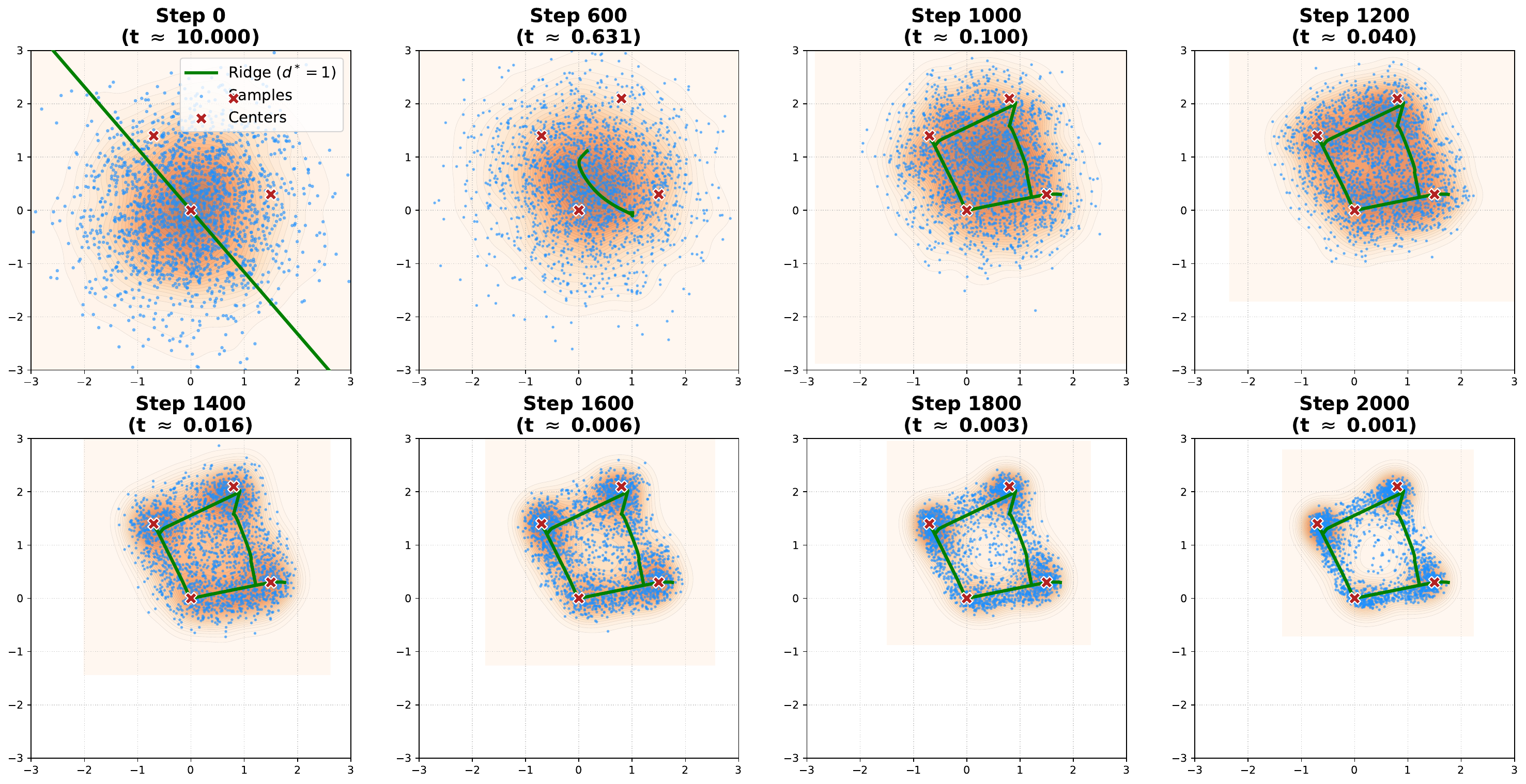} 
    \caption{\textbf{Evolution of generated samples under the proposed sampling dynamics with weight schedule $w(t)=h_t$.} The visualization displays snapshots of $N=2000$ particles in a 2D plane during the sampling process. The background contours depict the kernel density estimation (KDE) of the particle distribution. Samples reach the ridge neighborhood at \textbf{Step 1200}. \textbf{Steps 1200--1600} illustrates Normal Alignment. \textbf{Steps 1600--2000} demonstrate Tangent Sliding.}
    \label{fig:w_1_4_singular} 
\end{figure}

\clearpage
\begin{figure}[h]
    \centering
    \captionsetup{font=footnotesize}
    \includegraphics[width=0.6\linewidth]{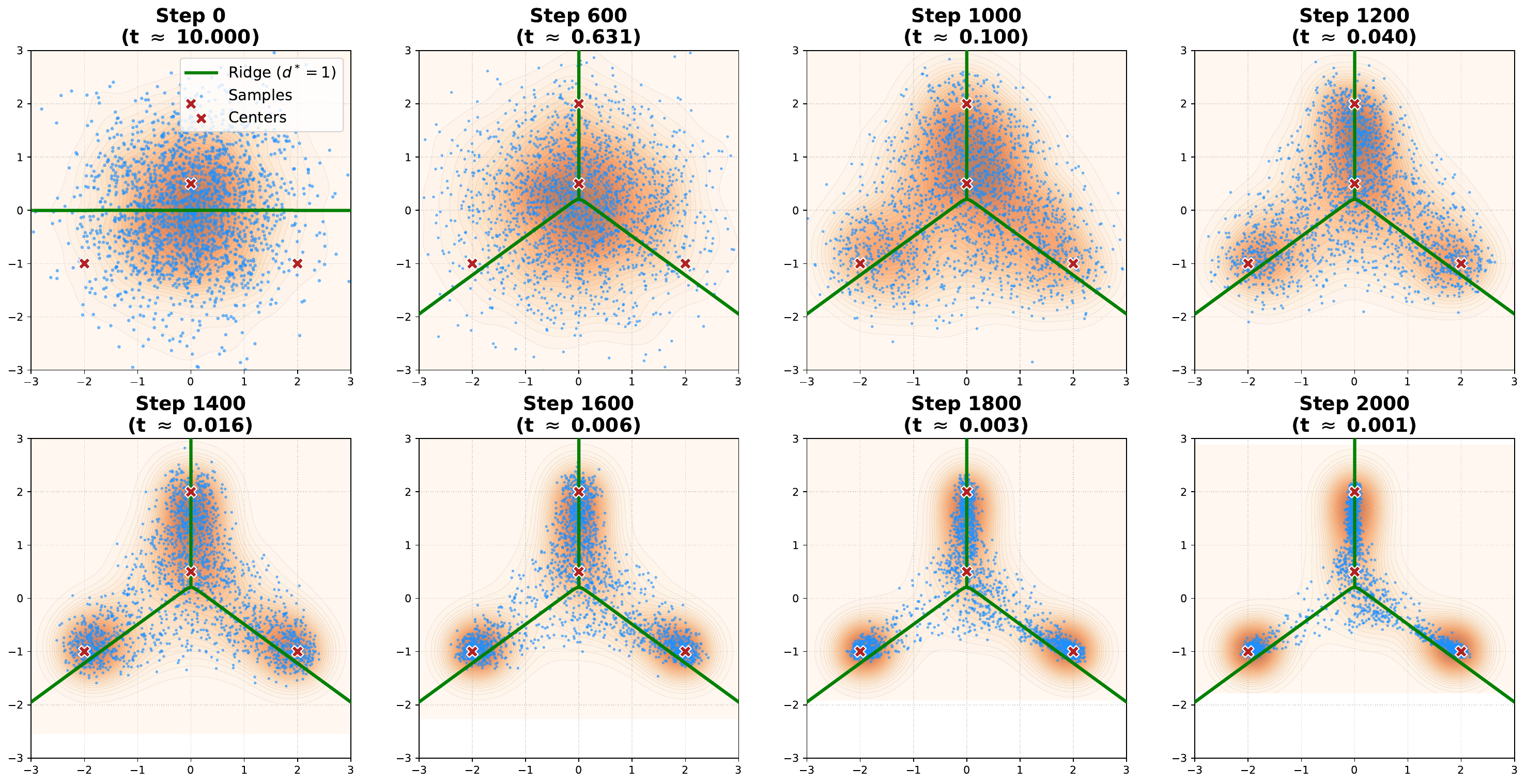} 
    \caption{\textbf{Evolution of generated samples under the proposed sampling dynamics with weight schedule $w(t)=h_t$.} The visualization displays snapshots of $N=2000$ particles in a 2D plane during the sampling process. The background contours depict the kernel density estimation (KDE) of the particle distribution. Samples reach the ridge neighborhood at \textbf{Step 1200}. \textbf{Steps 1200--1600} illustrates Normal Alignment. \textbf{Steps 1600--2000} demonstrate Tangent Sliding.}
    \label{fig:w_1_4_singular2} 
\end{figure}
\subsection{MNIST Trajectories}\label{sec:MNISTtrajectory}
The figures in this subsection complement the quantitative MNIST results in the main text by showing the full visual evolution of generated samples. Figure~\ref{fig:traj_evolution_full} displays the overall sampling process, while Figure~\ref{fig:traj_evolution_zoom} focuses on the final 200 steps. The main point is that the large-scale semantic structure emerges relatively early, whereas the final stage of sampling produces only minor refinements. This is consistent with the main-text observation that normal alignment occupies most of inference, while late-stage tangential motion becomes very limited.
\begin{figure}[h]
    \centering
    \captionsetup{font=footnotesize}
    \includegraphics[width=0.6\linewidth]{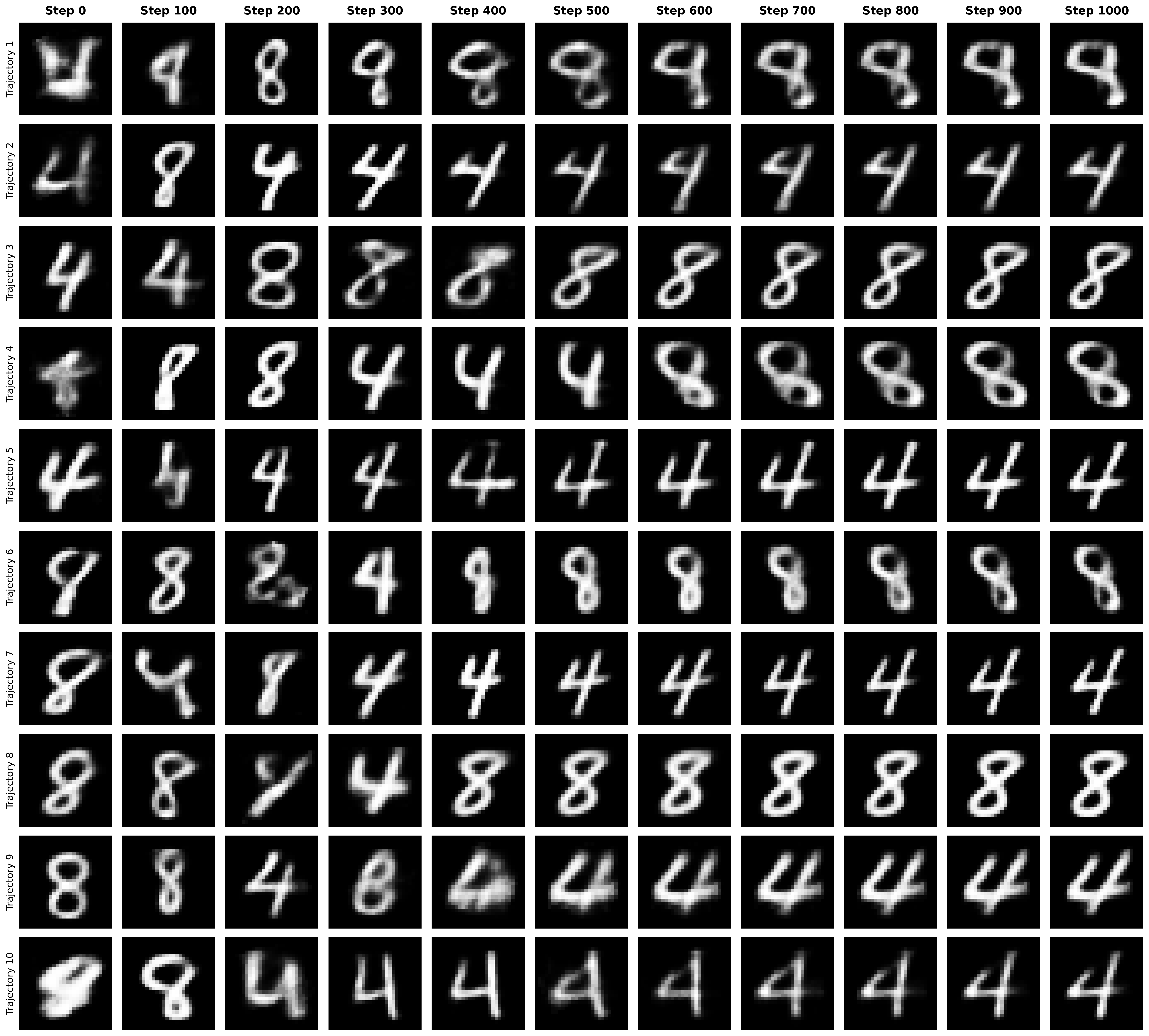} 
    \caption{\textbf{Visualization of the full sampling evolution for 10 independent trajectories.} The process initiates from standard Gaussian noise at Step 0. Distinct geometric structures begin to emerge around Step 200 as the samples are pulled towards the ridge manifold. By Step 800, the digits (4 or 8) are clearly formed, showing that the semantic content has stabilized.}
    \label{fig:traj_evolution_full}
\end{figure}
\begin{figure}[h]
    \centering
    \captionsetup{font=footnotesize}
    \includegraphics[width=0.6\linewidth]{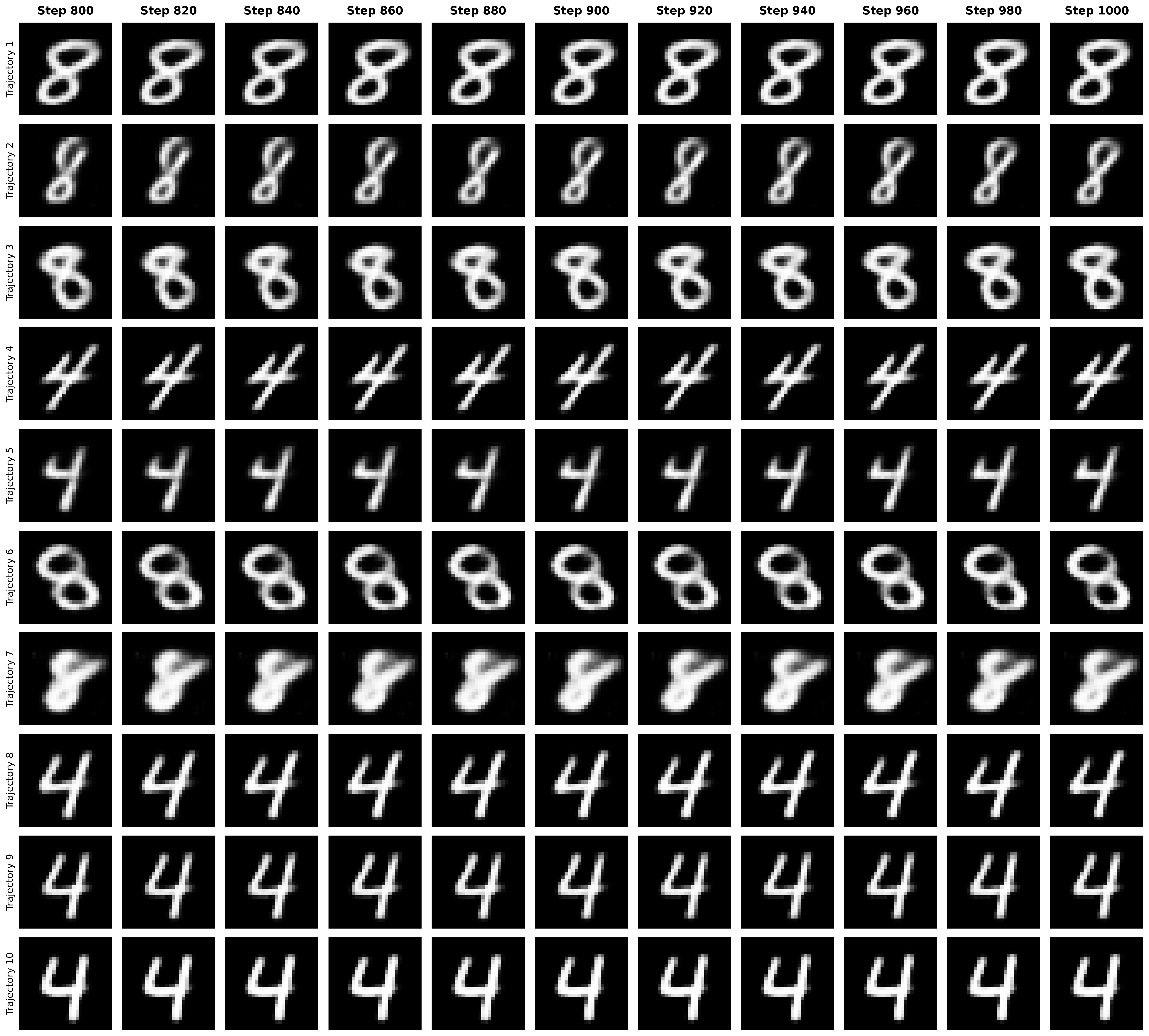} 
    \caption{\textbf{Evolution during the final sampling phase (Steps 800–1000).} In this regime, the visual changes are negligible, restricted to minor high-frequency refinements. This visual stability corroborates our quantitative finding that the tangent direction error plateaus, indicating that the generated samples remain stationary on the manifold rather than drifting towards specific training data points.}
    \label{fig:traj_evolution_zoom}
\end{figure}
\end{document}